\documentclass[twoside]{article}
\usepackage[accepted]{aistats2026}

\usepackage[round]{natbib}

\usepackage{microtype}
\usepackage{graphicx}
\usepackage{booktabs} 

\usepackage{soul}           
\usepackage[utf8]{inputenc} 
\usepackage[T1]{fontenc}    
\usepackage{hyperref}       
\usepackage{url}            
\usepackage{booktabs}       
\usepackage{amsfonts}       
\usepackage{nicefrac}       
\usepackage{microtype}      
\usepackage{xcolor}         
\usepackage{amsmath,amssymb,amsthm}
\usepackage{subcaption}
\usepackage{multirow}
\usepackage{rotating}
\usepackage{float} 
\usepackage{bbold}
\usepackage{comment}
\usepackage{algorithm}
\usepackage{algorithmic}

\newcommand{\E}[2]{\mathbb{E}_{#1}\left[ #2 \right]}

\DeclareMathOperator{\Tr}{Tr}

\DeclareMathOperator*{\argmin}{arg\,min}

\newcommand{\rank}{\mathrm{rank}}

\newcommand{\diag}{\mathrm{diag}}

\newcommand{\map}[1]{\hat{#1}}

\newcommand{\RR}{\mathbb{R}}

\newtheorem{theorem}{Theorem}
\newtheorem*{theorem*}{Theorem}
\newtheorem{lemma}{Lemma}

\newcommand{\KL}{D_{\mathrm{KL}}}

\newcommand{\Cat}{\mathrm{Cat}}

\newcommand{\norm}[1]{\lVert #1 \rVert}

\newcommand{\eqnn}[1]{
	\begin{equation*}
	    \begin{split}
		#1
	    \end{split}
	\end{equation*}
}
\newcommand{\eqn}[2]{
	\begin{equation}
		\label{#2}
	    \begin{split}
	    	#1
	    \end{split}
	\end{equation}
}

\begin{document}

\twocolumn[
\aistatstitle{Low Rank Based Subspace Inference for the Laplace Approximation of Bayesian Neural Networks}
\aistatsauthor{Josua Faller\textsuperscript{*} \and Jörg Martin\textsuperscript{*}}
\aistatsaddress{Physikalisch-Technische Bundesanstalt. Abbestr. 2-12, 10587 Berlin, Germany} ]

\runningauthor{Josua Faller, Jörg Martin}

\begingroup
\renewcommand\thefootnote{*}
\footnotetext{Both authors contributed equally.}
\endgroup

\begin{abstract}
Subspace inference for neural networks assumes that a subspace of their parameter space suffices to produce a reliable uncertainty quantification. In this work, we underpin the validity of this assumption by using low rank techniques. We derive an expression for a subspace model to a Bayesian inference scenario based on the Laplace approximation that is, in a certain sense, optimal given a specific dataset. We empirically show that a Laplace approximation constructed with a dimensionally reduced covariance matrix closely matches the full Laplace approximation obtained using the exact covariance matrix. Where feasible, this subspace model can serve as a baseline for benchmarking the performance of subspace models. 
In addition, we provide a scalable approximation of this subspace construction that is usable in practice and compare it to existing subspace models from the literature. In general, our approximation scheme outperforms previous work. Furthermore, we present a metric to qualitatively compare the approximation quality of different subspace models even if the exact Laplace approximation is unknown. 
\end{abstract}

\section{INTRODUCTION}

Bayesian modelling is an elegant and flexible method to quantify uncertainties of parametric models. Treating the parameters of the model as random variables allows to incorporate model uncertainty. Bayesian neural networks \citep{Gal2016,blundell2015weight,kendall2017uncertainties,hernandez2015probabilistic,maddox2019simple} implement this idea for neural networks (NNs). In practice, however, full posterior inference over Bayesian NNs is intractable due to the large number of parameters that define NNs. Thus, to quantify the uncertainty of a certain model, practitioners have to approximate the exact posterior distribution by a simpler one.
Several methods were developed to make this approximation feasible:
The posterior distribution can be approximated, e.g., by variational inference \citep{blundell2015weight,kingma2015variational,Gal2016,kendall2017uncertainties,Jordan1999,Wainwright2008}. A different idea, that goes in fact back to the 90s, is to use the Laplace approximation (LA) \citep{MacKay1992}, which has found increasing popularity in recent years due to scalable approximations \citep{lecun1989optimal,Ritter2018} and its flexible usability \citep{LaplaceRedux2021}. Moreover, in contrast to variational-inference-based approaches, it can be applied to off-the-shelf networks without any retraining: Given a maximum a posteriori (MAP) solution, that often coincides with the minimum of canonical loss functions, the LA replaces the exact posterior by a Gaussian distribution with the MAP as the mean and the inverse of the negative Hessian of the log posterior at the MAP as covariance matrix.

However, this approximation is still infeasible for NNs since the Hessian scales quadratically in the number of parameters such that often it cannot be computed or even stored, let alone be inverted. In addition, training NNs is a high dimensional non-convex optimization problem. In practice fully trained NNs are not located in a minimum of the loss function but rather on a saddle point \citep{Dauphin2014}. Hence, the so-computed Hessian is in general not positive semi-definite \citep{Sagun2016, Papyan2018}.
A partial solution to these issues is provided by approximating the Hessian by the generalized Gauss-Newton (GGN) matrix, which is identical to the Fisher Information matrix for common likelihoods \citep{Schraudolph2002,Pascanu2013,Martens14}. The GGN matrix is positive semi-definite and is constructed from objects that are feasible to compute, cf. Section \ref{sec:terminlogy_and_background} for details. 

However, it's sheer size makes the GGN matrix still unstorable, even for medium sized networks. 
Thus, to make the LA feasible for NNs, additional steps are necessary to reduce the size of the Hessian and to allow for an easier computation of its inverse. Common approaches include approximations via a diagonal \citep{lecun1989optimal,Salimans2016,kirkpatrick2017overcoming}, last layer \citep{kristiadi2020}, a Kronecker-factored \citep{Ritter2018} structure or dimensional reductions from the original $p$-dimensional space to the $s$-dimensionial subspace.

The last method gained more attention through a recent series of works which argue that it might suffice to consider partially stochastic NNs \citep{kristiadi2020,Snoek2015,Izmailov2019,Daxberger2021,Sharma2023} that is NNs where the Bayesian inference is performed in a lower dimensional subspace. NNs are heavily overparametrized and the idea is that a subset or well-selected linear combinations of parameters are sufficient to obtain reliable uncertainty estimates. 
We refer to this idea in this work as \emph{subspace inference}. \citet{Daxberger2021} apply this idea to make the LA for Bayesian NNs feasible by storing only a submatrix of the full GGN matrix. The submatrix is constructed using a subset of parameters that can be found via a diagonal approximation of the Hessian \citep{Daxberger2021}, via the magnitude of the parameters \citep{Cheng2017} or via an application of SWAG \citep{maddox2019simple}. 

Methods to reduce the size of the Hessian offer several computational advantages in comparison to the full LA. On the one hand, the memory complexity of the covariance matrix is reduced from $\mathcal{O}\left(p^2\right)$ to $\mathcal{O}\left(s^2\right)$, and on the other hand, the time complexity to draw Monte Carlo samples is of the order of $\mathcal{O}\left(s^3\right)$ instead of $\mathcal{O}\left(p^3\right)$.

The aim of our work is to give a systematic, generic and statistically sound approach to study the usability of subspace inference for the LA of Bayesian NNs.
Similar as \citet{Daxberger2021} we use the widespread combination of the LA with a linearization of our NN $f_\theta$ around the MAP value $\hat{\theta}$ of the parameters $\theta$ \citep{Foong2019,Immer2021,Deng2022,Ortega2023}:
\begin{align}
   \label{eq:intro_lin_model}
   f_{\mathrm{Lin},\theta}(X) = f_{\hat{\theta}}(X) + J_X (\theta-\hat{\theta})\,,
\end{align}
where $J_X = \nabla_\theta f_\theta(X)\vert_{\theta=\map{\theta}}$. This method is known as the linearized LA. Our method differs from existing work by making the \emph{predictive covariance} of the linearized Laplace approximation the centerpiece of our analysis. We do so for the following reasons: 
\begin{enumerate}
    \item The posterior predictive distribution is the actual objective: This was pointed out already in previous work, e.g. in \citep{Izmailov2021BNNReallyLike}: as prediction happens in function space the posterior predictive and not the posterior is the object we actually care about. 
    \item The predictive covariance determines the posterior predictive distribution completely under LA: As LA always starts with the MAP solution, the only object that differs from approach to approach is the variance estimate that is entirely determined by the predictive covariance. Thus, its approximation is all that matters.
\end{enumerate}
This viewpoint allows us to give some precise statements of approximation quality and optimality.
The contributions of our article are as follows:

\begin{enumerate}
    \item We specify an optimality criterion for subspace LAs based on closeness to the full GGN Laplace approximation. In contrast to prior work using parameter-based heuristics, we target the \emph{predictive covariance} directly, which determines predictive uncertainty in the linearized LA framework.
    \item We prove existence and uniqueness of the optimal subspace (Theorem \ref{thm:OptSubmodel}) and provide an explicit formula. Empirically, this baseline can achieve faithful approximation with <1\% of parameters.
    \item As the subspace from 2 can typically not be computed in practice,  we provide a feasible approximation and observe that it performs in many cases superior to the subset selection of \citep{Daxberger2021,LaplaceRedux2021}.
    \item  We propose a trace-based criterion for comparing subspace models when the full LA is unknown, and validate empirically that it correlates with KL-divergence to the full LA.  
\end{enumerate}

This article is organized as follows: In Section \ref{sec:recent_work} we recall recent work on the subject of the article and then evoke some background on the LA for Bayesian NNs in Section \ref{sec:terminlogy_and_background}. In Section \ref{sec:LA_for_subspace_models} we provide the main theoretical contributions of this work. In Section \ref{sec:Experiment} several experiments to empirically verify our theoretical analysis are carried out. Additional information is provided in the Appendix.

\section{RECENT WORK}
\label{sec:recent_work}

\textbf{Laplace Approximation.} The first application of the LA using the Hessian for NNs was introduced by \citet{MacKay1992}. \citet{MacKay1992Class} also proposed an approximation similar to the generalized Gauss-Newton (GGN) method. The combination of scalable factorizations or diagonal Hessian approximations with the GGN approximation \citep{Schraudolph2002,Martens14} made the LA applicable for larger networks. In particular, the GGN approximation  gained more attention due to 
the introduction of the Kronecker-Factored Approximate Curvature (KFAC) \citep{Ritter2018,Botev2017,Martens2015} which is scalable and outperforms the diagonal Hessian approximation.
Due to underfitting issues of the LA \citep{Foong2019}, the linearized
LA based on \eqref{eq:intro_lin_model} was developed \citep{Immer2021}. In this work, we use the same setting.

\textbf{Partially Stochastic Neural Networks.} 
The study of partially stochastic NNs gained some attention because of their computational efficiency. But even from a statistics viewpoint, partially stochastic NNs are attractive because they can capture the uncertainty of the full model by using only a fraction of the parameters \citep{Sharma2023,Andrade2024,2024Zhao}. \citet{Sharma2023} developed the concept of Universal Conditional Distribution Approximators and proved that certain partially stochastic NNs can form samplers of any continuous target conditional distribution arbitrary well. \citet{calvoordonez2024} extended this idea to infinitely deep Bayesian NNs.  

\citet{2020Lee} and \citet{yang2024bayesian} apply low rank approximations for dimensional reduction in the context of the KFAC approximation. \citet{Izmailov2019} developed a low-dimensional affine subspace inference scheme. They select a linear combination of parameter vectors which span a vector space around the MAP. Since they consider low-dimensional subspaces ($s\leq 10$) different methods are available to approximately sample from the posterior distribution. However, they observed that the posterior obtained in these low-dimensional subspaces is too concentrated, so that only a tempered posterior leads to reasonable uncertainties.
\citet{dold2024bayesian} apply this idea to semi-structured models.
\citet{Daxberger2021} choose a subset of parameters to construct a subspace model. This subset is selected by the parameters that have the highest posterior variance. 
However, this work often requires the selection of quite a large number of parameters ($s$ up to $4\cdot 10^4$). 
Our framework is closest to this work. In contrast to previous work, we study the predictive instead of the posterior distribution to obtain a feasible parameter subspace. In addition, we observe that neither an ad hoc tempering of the posterior distribution nor thousands of parameters are needed to construct faithful lower dimensional approximations.

\section{TERMINOLOGY AND BACKGROUND}
\label{sec:terminlogy_and_background}
\textbf{Setup and Notational Remarks.}
We consider the supervised learning framework. We model the relation between the independent observable $x$ and the target $y$ by a parametric distribution $p(y | x, \theta)$ with parameters $\theta \in \RR^p$. Different observations are, as usual, assumed to be independent and identically distributed.
We denote the training set of observations as $\mathcal{D} = \{(x_i,y_i) | 1 \leq i \leq N\}$\,, where $N$ denotes the number of observations.
We study regression and classification tasks. $C$ represents the number of outputs $f_\theta(x)=(f_\theta^1(x),\ldots, f_\theta^C(x))^\intercal\in \RR ^C$ of the NN $f_\theta$, for both, regression and classification problems. For regression we make a Gaussian model assumption $p(y | x, \theta) = \mathcal{N}(y | f_{\theta}(x), \sigma^2\mathbb{1}_C)$, where only the mean is modelled by the NN. Classification tasks with $C$ classes are modelled by a categorical distribution $p(y|x,\theta)=\mathrm{Cat}\left(y | \phi(f_{\theta}(x))\right)$ with probability vector $\phi(f_{\theta}(x))$, where $\phi$ denotes the softmax function. 

We will often consider not a single input sample to $f_\theta$ but a whole set such as $X=(x_1,\ldots,x_n)$. In this case $f_\theta(X)=(f_\theta(x_1)^\intercal,\ldots,f_\theta(x_n)^\intercal)^\intercal\in \RR^{nC}$ should be read as the concatenation of the outputs. We will frequently use the Jacobian of $f_\theta$ w.r.t. its parameter $\theta \in \RR^p$ evaluated at the MAP $\map{\theta}$ defined in \eqref{eq:MAP} below. Given a set $X$ we concatenate the single input Jacobians along the output dimension and use the symbol
\begin{equation}
    \label{eq:J_X}
    \begin{aligned}
        J_X := ( \nabla_\theta f_\theta(x_1)^\intercal,\ldots,  \nabla_\theta f_\theta(x_n)^\intercal)^\intercal\vert_{\theta=\map{\theta}}\in \RR^{nC\times p}\,.
    \end{aligned}
\end{equation}

\textbf{Bayesian Neural Networks.}
When taking a Bayesian view on NNs the parameter $\theta$ is considered as a random variable equipped with a prior distribution $p(\theta)$. Given the training data $\mathcal{D}=\{(x_i,y_i) | 1 \leq i \leq N\}$, the posterior distribution of $\theta$ is given by $p(\theta | \mathcal{D}) \propto p(\theta) p(D|\theta)= p(\theta)\prod_{i=1}^N p(y_i | x_i, \theta)$ (with $p(y_i|x_i, \theta)$ as above). A point estimate for $\theta$ is then given by the value that is most likely under $p(\theta|\mathcal{D})$, the so-called MAP (\emph{maximum a posteriori}) estimate, that is
\begin{align}
    \label{eq:MAP}
     \map{\theta} = \argmin_{\theta} \mathcal{L}_{\theta}(\mathcal{D})  \,,
\end{align} 
where we used the (unnormalized) negative log-posterior $\mathcal{L}_\theta(\mathcal{D}) = -\sum_{i=1}^N \ln p(y_i |x_i, \theta) - \ln p(\theta)$ \,.
In this work we will use the common choice $p(\theta) = \mathcal{N}(\theta|0,\lambda^{-1}\mathbb{1}_p)$ with precision $\lambda>0$ for which $\mathcal{L}_\theta$ just boils down to the MSE loss (for regression) or cross-entropy loss (for classification) combined with L2 regularization.

\textbf{Laplace Approximation.} With $\mathcal{L}_\theta(\mathcal{D})$ as above the posterior distribution $p(\theta|\mathcal{D})$ reads as
$p(\theta| \mathcal{D}) = \frac{1}{Z} p(\mathcal{D} | \theta) p(\theta) =: \frac{1}{Z} e^{-\mathcal{L}_{\theta}(\mathcal{D})}$
with the normalization constant $Z=\int d\theta\; p(\mathcal{D} | \theta) p(\theta)$. 
For complex models such as Bayesian NNs the exact posterior is typically infeasible to compute or sample from. Expanding $\mathcal{L}_\theta(\mathcal{D})$ to second order around the MAP $\map{\theta}$ from \eqref{eq:MAP}, we obtain 
the \emph{Laplace approximation} of the posterior 
\begin{align*}
 p(\theta| \mathcal{D}) \simeq \mathcal{N}(\theta | \map{\theta}, \Psi)
 \end{align*}
with mean $\map{\theta}$ and covariance $\Psi = \left(\nabla_{\theta}^2 \mathcal{L}_{\theta}(\mathcal{D}) \vert_{\theta=\map{\theta}}\right)^{-1} = \left( N H + \lambda\mathbb{1}_p  \right)^{-1} \in \mathbb{R}^{p\times p}$, where we denote by $H=-\frac{1}{N}\sum_{i=1}^N\nabla_\theta^2 \ln p(y_i|\theta,x_i)\vert_{\theta=\map{\theta}}$ the Hessian of the averaged negative log-likelihood.

\textbf{Generalized Gauss-Newton Matrix.}
The Hessian $H \in \mathbb{R}^{p \times p}$ from above is the second order derivative of $-\frac{1}{N}\ln p(\mathcal{D}|\theta)$ at the MAP $\hat{\theta}$. On the one hand, to compute $H$ is infeasible, and on the other hand, even if $H$ could be computed, it would be impractical, if not impossible, to store the $\frac{p \left(p+1\right)}{2}$ free components for typical values of $p$.
In addition, for trained NNs the Hessian does usually not have the nice property of positive semi-definiteness that is found, e.g., in the context of convex problems, because the learned MAP $\map{\theta}$ is, in general, not a local minimum but rather a saddle point. 
The difficulties of computational complexity and missing positive definiteness can be overcome by using the generalized Gauss-Newton (GGN) matrix \citep{Schraudolph2002} instead of $H$:
\begin{align}
    \label{eq:GGN}
    H_{\mathrm{GGN}} = \frac{1}{N} \sum_{i=1}^N J^\intercal_{f_i} H_{\text{-}\ln p(y_i|f_i)} J_{f_i}\,,
    \end{align}
where $J_{f_i} = \nabla_{\theta} f_{\theta}(x_i)\vert_{\theta=\hat{\theta}} \in \mathbb{R}^{C \times p}$ and $H_{\text{-}\ln p(y_i|f_i)} = - \nabla_{f}^2 \ln p(y_i|f_i)\vert_{f_i=f_{\hat{\theta}}(x_i)} \in  \mathbb{R}^{C \times C}$ is the Hessian of the negative log-likelihood w.r.t. the model output ${f_i} = f_{\theta}(x_i)$. $H_{\mathrm{GGN}}$ can be interpreted as the Hessian of the linearized model \citep{Martens14, Immer2021} and is positive semi-definite if all $H_{\text{-}\ln p(y|f_i)}$ are positive-semi definite \citep{Schraudolph2002}, which is the case in our work. More detailed  information on $H_{\mathrm{GGN}}$ and the relation between $H_{\mathrm{GGN}}$ and $H$ is provided in Appendix \ref{sec:FI_GGN_and_H}.
Combining \eqref{eq:GGN} with the term arising from the prior $p(\theta)$ we obtain as precision matrix of the LA
$
    \Psi_{\mathrm{GGN}}^{-1}=  \sum_{i=1}^N J^\intercal_{f_i} H_{\text{-}\ln p(y_i|f_i)} J_{f_i} + \lambda \mathbb{1}_p\,.
$
Where feasible, the LA $p(\theta|\mathcal{D})\simeq \mathcal{N}(\theta|\hat{\theta},\Psi_{\mathrm{GGN}})$ will serve as gold standard in this work. We refer to it as the \textit{full LA} and we identify, if not stated otherwise, $\Psi$ with its positive semi-definite estimate $\Psi_{\mathrm{GGN}}$.

\textbf{Approximations.}
While the GGN relation \eqref{eq:GGN} consists of objects, $J_{f_i}$ and $H_{\text{-}\ln p(y_i|f_i)}$, that are scalable in their computation we usually can't compute $H_{\mathrm{GGN}}$ or $\Psi_{\mathrm{GGN}}^{-1}$ as the resulting matrices have still too many dimensions for modern NNs. In particular, we can't invert $\Psi_{\mathrm{GGN}}^{-1}$ to obtain the posterior covariance
$\Psi_{\mathrm{GGN}}$.
As a consequence, various approximations have been developed that modify the structure in such a way that it takes less storage and is easier to invert. An easy solution is to only keep the diagonal of $\Psi_{\mathrm{GGN}}$. In the KFAC approximation the Hessian is reduced to a form where it is the Kronecker product of two smaller matrices.

\textbf{Predictive Distribution.}
For the posterior distribution $p(\theta|\mathcal{D})$ and a set of $n$ inputs $X$ the posterior predictive distribution is given by
\begin{align}
    \label{eq:pred_distribution}
    p(Y|X, \mathcal{D}) = \int d\theta\; p(Y| X, \theta) p(\theta | \mathcal{D}).
\end{align}
Under the LA and using the linearized model \eqref{eq:intro_lin_model} for $p(Y|X, \mathcal{D})$ we can give an explicit formula to this distribution for regression problems
\begin{align}
    \label{eq:PredDistrGauss}
    p(Y|X,\mathcal{D}) & \simeq \mathcal{N}(Y|f_{\map{\theta}}(X), \Sigma_X + \sigma^2 \mathbb{1}_{nC} )
    \\
    \label{eq:Sigma_X}
    \text{with} \quad
    \Sigma_X & = J_{X} \Psi J_{X}^\intercal  \in \RR^{nC\times nC}
\end{align}
denoting the model uncertainty part of the predictive covariance.
For classification tasks the predictive distribution can be approximated, for a single input $x_i$, by the probit approximation \citep{Bishop2006} 
\eqn{
    p(y_i|x_i,\mathcal{D}) \simeq \Cat\left(y_i | \phi\left(\frac{f_{ \map{\theta}}(x_i)}{\sqrt{1 + \frac{\pi}{8} \diag\Sigma_{x_i}}}\right)\right)
}{eq:PredDistrCat}
with $\Sigma_{x_i}\in \RR^{C\times C}$ being the corresponding block diagonal element of \eqref{eq:Sigma_X} and $\phi$ denoting the softmax function. Note that in both cases, regression and classification, the predictive distribution is essentially fixed by $\Sigma_X$ from \eqref{eq:Sigma_X}, which is why this object will be the linchpin of our analysis below. We will call $\Sigma_X$ the \emph{epistemic predictive covariance}.

\section{THE LAPLACE APPROXIMATION FOR SUBSPACE MODELS}
\label{sec:LA_for_subspace_models}
\textbf{Subspace Models.}
In this work we study, as in \citep{Izmailov2019}, models that are defined on an affine subspace of the parameter space $\mathbb{R}^p$ chosen to contain the MAP $\hat{\theta}$ from \eqref{eq:MAP}. That is, we consider a re-parametrization 
\begin{align}
    \label{eq:reparametrization}
    \theta = \hat{\theta} + P \mu \,,
\end{align}
where $P\in \RR^{p\times s}$ is a matrix that we call, somewhat loosely, the projection matrix (in general it's not related to a mathematical projection) and $\mu$ is a new parameter that runs through $\RR^s$ where $s\leq p$ is the subspace dimension.
The assumption in considering Bayesian inference of NNs in a subspace is that only a fraction of the parameter space is actually needed to represent the (epistemic) uncertainty faithfully. 
To emphasize that a Bayesian inference is done under the subspace model \eqref{eq:reparametrization} we will use in the following $p_P$ instead of $p$ as a symbol for our probability densities. 

Note that the selection of a subset of parameters on which to perform inference, as it is done by \cite{Daxberger2021} and \cite{Sharma2023}, is a special case of \eqref{eq:reparametrization}. 

\textbf{Bayesian Inference for $\mu$.} 
To perform Bayesian inference in the subspace model \eqref{eq:reparametrization}, we choose the following prior 
\begin{align}
    \label{eq:mu_prior}
    p_P(\mu) = \mathcal{N}(\mu|0, \left(\lambda P^\intercal P \right)^{-1}) \,,
\end{align}
where we recall that $\lambda$ is the precision of $p(\theta)$. 
Together with the following likelihood 
\begin{align}
    \label{eq:mu_likelihood}
    p_P(\mathcal{D}|\mu) = p(\mathcal{D} |\hat{\theta} + P\mu)\,,
\end{align}
that is induced by \eqref{eq:reparametrization}, the following lemma holds: 
\begin{lemma}
\label{lem:Bayes_model}
In the setting above, consider a full rank $P\in \RR^{p\times s}$.  For the posterior $p_P(\mu|\mathcal{D})\propto p_P(\mu) p_P(\mathcal{D}|\mu)$ with prior $p_P(\mu)$ as in \eqref{eq:mu_prior} we have the LA
\begin{align}
    \label{eq:LA_mu}
    p_P(\mu|\mathcal{D}) \simeq \mathcal{N}(0,(P^\intercal \Psi^{-1} P)^{-1}) \,.
\end{align}
\end{lemma}

Lemma \ref{lem:Bayes_model} follows from \eqref{eq:mu_likelihood} and \eqref{eq:mu_prior}, we can deduce $-\nabla_\mu^2 \ln\left(p_P(\mathcal{D}|\mu) p_P(\mu)\right) \vert_{\mu=0} = P^\intercal (NH+ \lambda \mathbb{1}_p)P = P^\intercal \Psi^{-1} P$.

We will find in Theorem \ref{thm:OptSubmodel} below that the family of posteriors \eqref{eq:LA_mu} is rich enough to approximate the full LA optimally in a certain sense when a suitable $P$ is chosen. Moreover, in the case where $P$ encodes a subset, \eqref{eq:LA_mu} coincides with the posterior considered by \citet{Daxberger2021}.

\textbf{Predictive Distributions of Subspace Models.}
Similar to \eqref{eq:intro_lin_model} we linearize $\tilde{f}_{\mu} =f_{\map{\theta} + P \mu}$ around $\mu=0$ to obtain for a set of $n$ inputs $X$ 
\begin{align}
    \label{eq:lin_model_mu}
    \tilde{f}_{\mathrm{Lin},\mu}(X) = f_{\hat{\theta}}(X) + J_X P (\mu-0) \,,
\end{align}
where we denoted, as in \eqref{eq:intro_lin_model}, by $J_X=\nabla_\theta f_{\theta}(X)\vert_{\theta=\map{\theta}} \in \RR^{nC\times p}$ the Jacobian of the full network at the MAP.

Combining \eqref{eq:LA_mu} with \eqref{eq:lin_model_mu} we obtain as above for the predictive distribution $p_P(Y|X,\mathcal{D})$
\begin{equation}
    \begin{gathered}
        \mathcal{N}(Y |f_{\map{\theta}}(X), \Sigma_{P,X}+ \sigma^2 \mathbb{1}_{nC} ) \mbox{ or} 
        \\
        \Cat\left(y_i | \phi\left(\frac{f_{ \map{\theta}}(x_i)}{\sqrt{1 + \frac{\pi}{8} \diag\Sigma_{P,x_i}}}\right)\right)\,, 
    \end{gathered}
    \label{eq:LaplaceSubmodel}
\end{equation}
for regression and classification respectively, with the notation 
\begin{align}
    \label{eq:Sigma_p}
    \Sigma_{P,X} = J_X P (P^\intercal \Psi^{-1} P)^{-1} P^\intercal J_X^\intercal \in \mathbb{R}^{nC\times nC}
\end{align}
for its epistemic predictive covariance. In \eqref{eq:LaplaceSubmodel} we use $\Sigma_{P,x_i}\in \RR^{C\times C}$ to denote the block diagonal elements of $\Sigma_{P,X} \in \RR^{nC \times nC}$ corresponding to the single inputs $x_i$.
 
\subsection{The Subspace LA Closest to the Full LA}
\label{subsec:the_optimal_subspace_model}

Consider a set of $n$ inputs $X=\left(x_1, \ldots, x_n \right)$. As already noted above, the predictive distribution $p(Y|X,\mathcal{D})$ of the full model and $p_P(Y|X,\mathcal{D})$ of the subspace model only differ by their epistemic predictive covariances $\Sigma_P$ and $\Sigma_{P,X}$. Given $X$ and a fixed subspace dimension $s\leq p$, we might therefore consider $P^*\in \RR^{p\times s}$ optimal if it solves the following minimization problem.
\begin{align}
    \label{eq:minimization_problem}
    P^* \in \argmin_{P\in \RR^{p\times s}, \,\rank \,P=s} \|\Sigma_{P,X} - \Sigma_X\|_F \,,
\end{align}
where $\|\ldots\|_F$ denotes the Frobenius norm. A solution to \eqref{eq:minimization_problem} is never unique. In fact, for any $P^*$ that solves \eqref{eq:minimization_problem} we can also consider $P^*Q$ for an arbitrary invertible $Q\in \RR^{s\times s}$  since we have $\Sigma_{P^*Q,X}=\Sigma_{P^*,X}$, cf. \eqref{eq:Sigma_p}. Note that such a change of $P^*$ corresponds to a reparameterization $\mu \rightarrow Q^{-1}\mu$. We will prove below, however, that up to such reparametrizations the solution to \eqref{eq:minimization_problem} is unique.

For the solution of the problem \eqref{eq:minimization_problem} we will need the eigenvalue decomposition $\Sigma_X = J_X \Psi J_X = U \Lambda U^\intercal$ where $U$ is an orthogonal matrix and $\Lambda \in \mathbb{R}^{nC \times nC}$ is a positive semi-definite diagonal matrix. We choose this eigendecomposition such that the diagonal entries of $\Lambda$ are decreasing. We will use the Eckart-Young-Mirsky-Theorem \citep{Schmidt1907,Eckart1936,Mirsky1960} which states that the following low rank problem has an explicit solution
\begin{align}
    \label{eq:BestLowRankSigma}
    U_s \Lambda_s U_s^\intercal \in \argmin_{A\in \mathbb{R}^{nC \times nC}: \,\rank\,A \leq s} \| A- \Sigma_X \|_F \,,
\end{align}
where $U_s \in \mathbb{R}^{nC \times s}$ contains the first $s$ eigenvectors, called dominant eigenvectors from now on, and $\Lambda_s \in \mathbb{R}^{s\times s}$ is the reduced diagonal matrix obtained by taking the upper $s \times s$ block containing the $s$ leading eigenvalues of $\Sigma_X$.
While solving the low rank problem \eqref{eq:BestLowRankSigma} is standard, solving \eqref{eq:optimal_P} for $P^{*}$ is not. In the following theorem we give an explicit expression to $P^{*}$ and show that its epistemic predictive covariance solves \eqref{eq:BestLowRankSigma}.

\begin{theorem}[Existence and uniqueness for \eqref{eq:minimization_problem}]
\label{thm:OptSubmodel}
Consider the problem \eqref{eq:minimization_problem} with $s\leq s_{\mathrm{max}}=\min(nC,p)$.
Suppose that $J_X\in \RR^{nC\times p}$ has full rank. For any invertible $Q\in \RR^{s\times s}$ the matrix
\begin{align}
    \label{eq:optimal_P}
    P^* = \Psi J_X^\intercal U_s Q
\end{align}
solves \eqref{eq:minimization_problem}.
For any such $P^*$ we have $\Sigma_{P^*,X} = U_s \Lambda_s U_s^\intercal$.
If the diagonal elements of $\Lambda$ satisfy $\sigma_s > \sigma_{s+1}$ then the solution of the minimization problem \eqref{eq:minimization_problem} is unique up to reparametrization: any minimizer of \eqref{eq:minimization_problem} is of the form \eqref{eq:optimal_P}.
\end{theorem}

The proof of Theorem \ref{thm:OptSubmodel} is provided in Appendix \ref{sec:AppendixProofSubmodel}.
The condition $\sigma_s > \sigma_{s+1}$ on the singular values of $\Sigma_X$ is typically true in practice since ``noisy real-world'' matrices almost surely have distinct singular values \citep{hartfiel1995dense}. The restriction to dimensions below $s_{\mathrm{max}} = \min(nC,p)$ and the assumption on the full rank of $J_X$ is needed to assure that $P^*$ has full rank which is required for $\Sigma_{P^*,X}$ in order to be well-defined. If $J_X$ doesn't have full rank, we restrict ourselves to $s$ below the rank of the Jacobian. This is the case for some regression problems in Section \ref{sec:Experiment}. In Appendix \ref{subsec:additional_remark_to_theorem} we discuss some additional consequences of Theorem \ref{thm:OptSubmodel}.
\begin{algorithm}[tb]
   \caption{Subspace Construction}
   \label{alg:construction}
\begin{algorithmic}[1]
   \REQUIRE Trained $f_{\hat{\theta}}$, dimension $s$, subset size $n$
   \ENSURE Projection matrix $P \in \mathbb{R}^{p \times s}$
   \STATE Compute $\Psi_{\text{approx}} \in \{\mathrm{KFAC},\, \mathrm{Diagonal}\}$
   \STATE Sample $X' \subset \mathcal{D}$ with $|X'| = n$
   \STATE Compute $J_{X'} = \nabla_{\theta} f_{\theta}(X')|_{\theta=\hat{\theta}}$
   \STATE Compute $M = J_{X'}\,\Psi_{\text{approx}}\,J_{X'}^{\top} \in \mathbb{R}^{nC\times nC}$
   \STATE SVD: $M = U\,\Lambda\,U^{\top}$, extract top $s$ eigenvectors $U_s$
   \STATE \textbf{Return} $P = \Psi_{\text{approx}}\,J_{X'}^{\top}\, U_s$
\end{algorithmic}
\end{algorithm}

\begin{figure*}[t!]
    \centering

    
    
    \begin{subfigure}{0.192\textwidth}
        \centering
        \includegraphics[width=\linewidth]{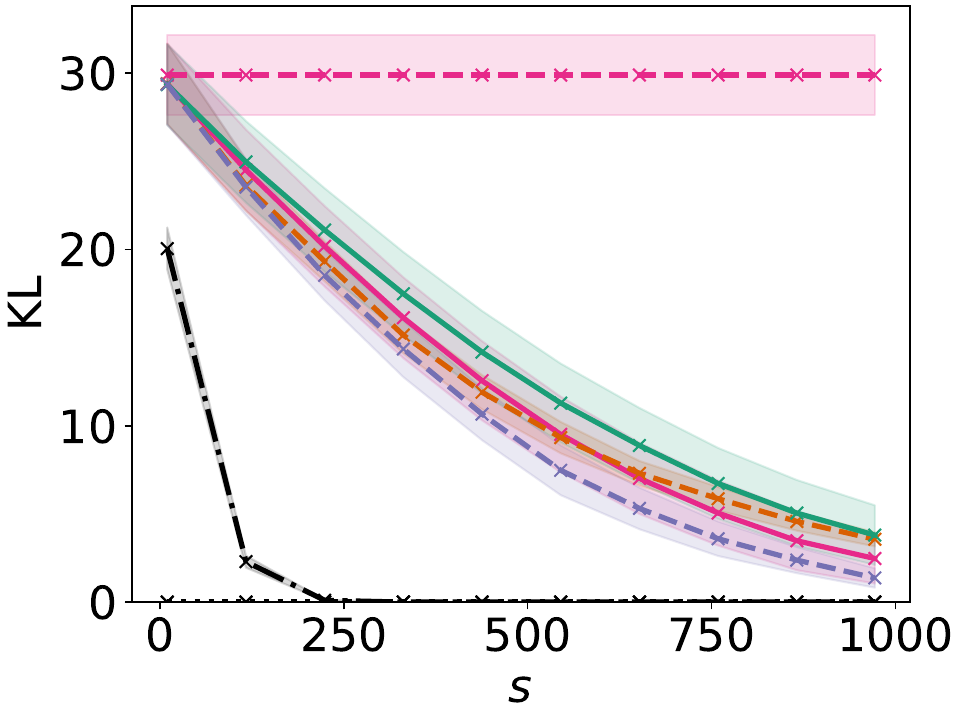}
    \end{subfigure}
    \hfill
    \begin{subfigure}{0.192\textwidth}
        \centering
        \includegraphics[width=\linewidth]{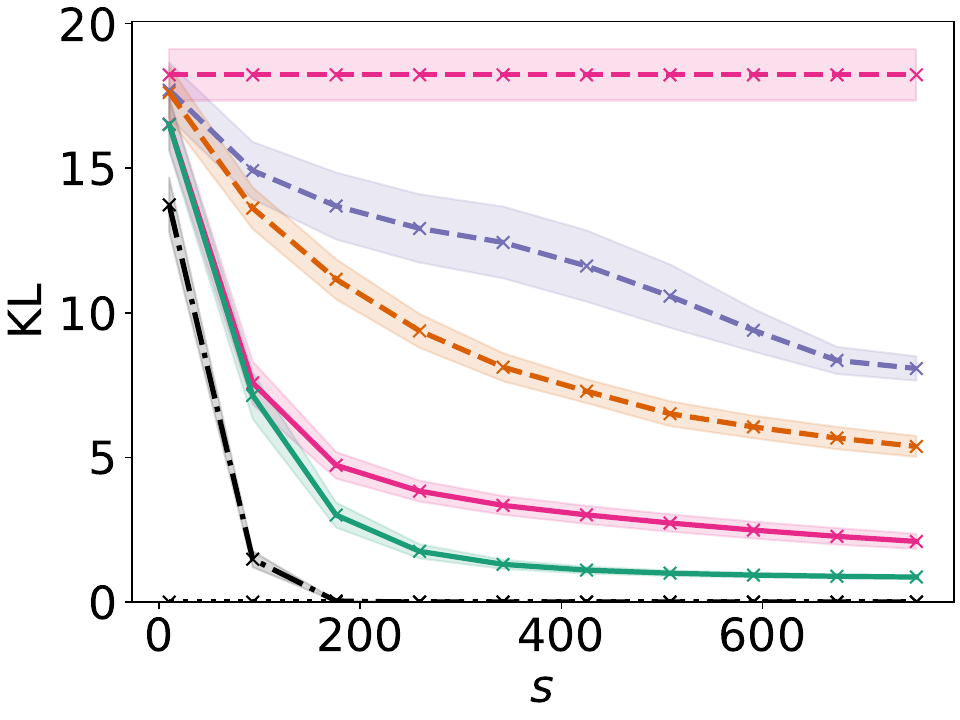}
    \end{subfigure}
    \hfill
    \begin{subfigure}{0.192\textwidth}
        \centering
        \includegraphics[width=\linewidth]{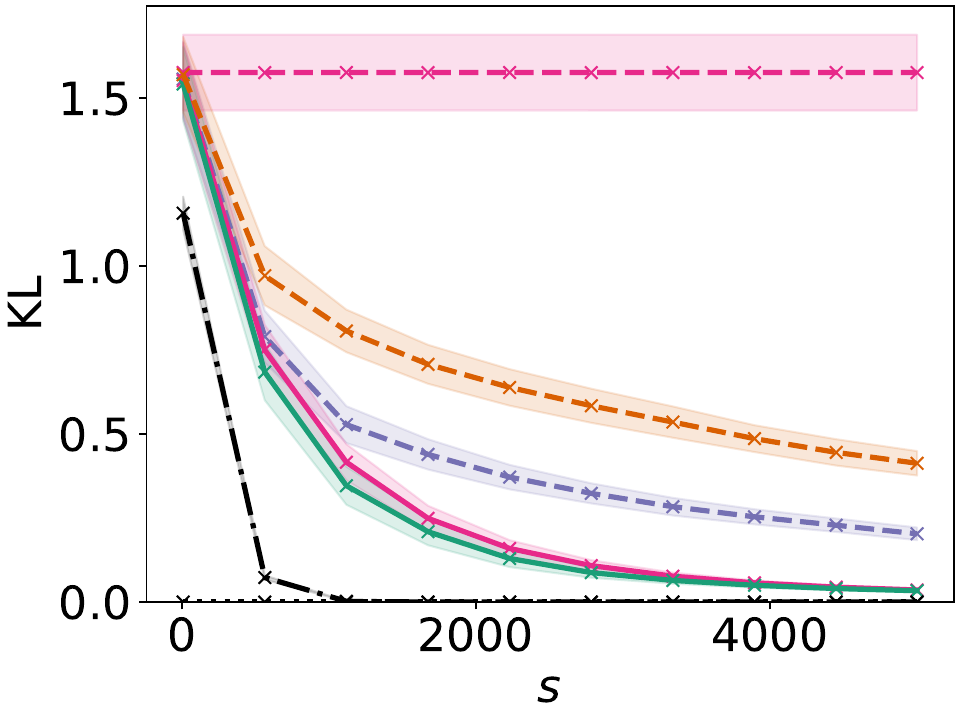}
    \end{subfigure}
    \hfill
    \begin{subfigure}{0.192\textwidth}
        \centering
        \includegraphics[width=\linewidth]{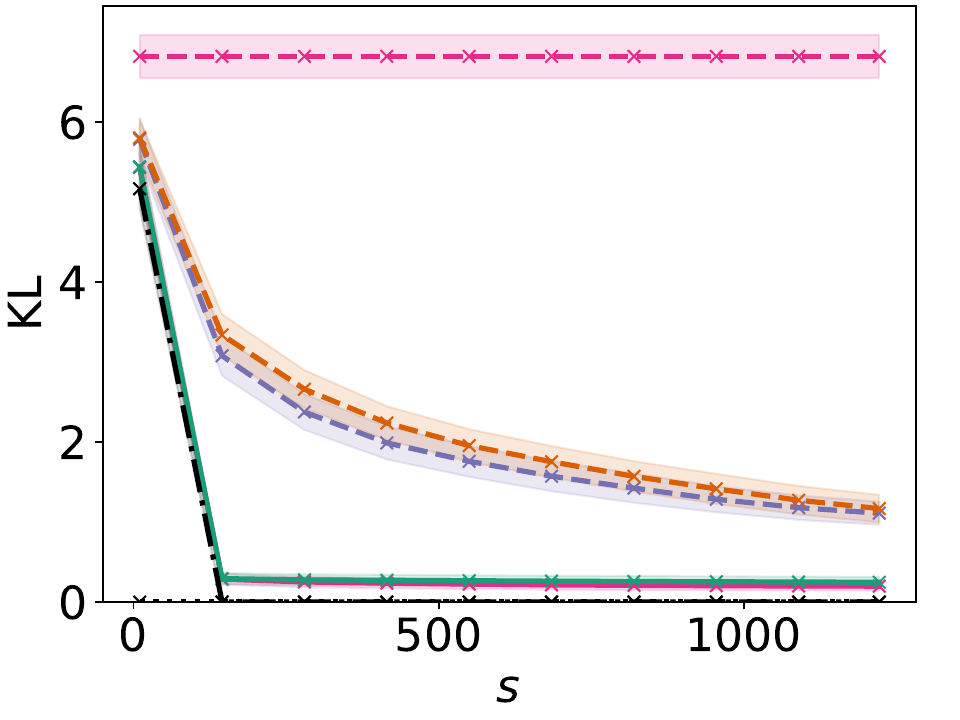}
    \end{subfigure}
    \hfill
    \begin{subfigure}{0.192\textwidth}
        \centering
        \includegraphics[width=\linewidth]{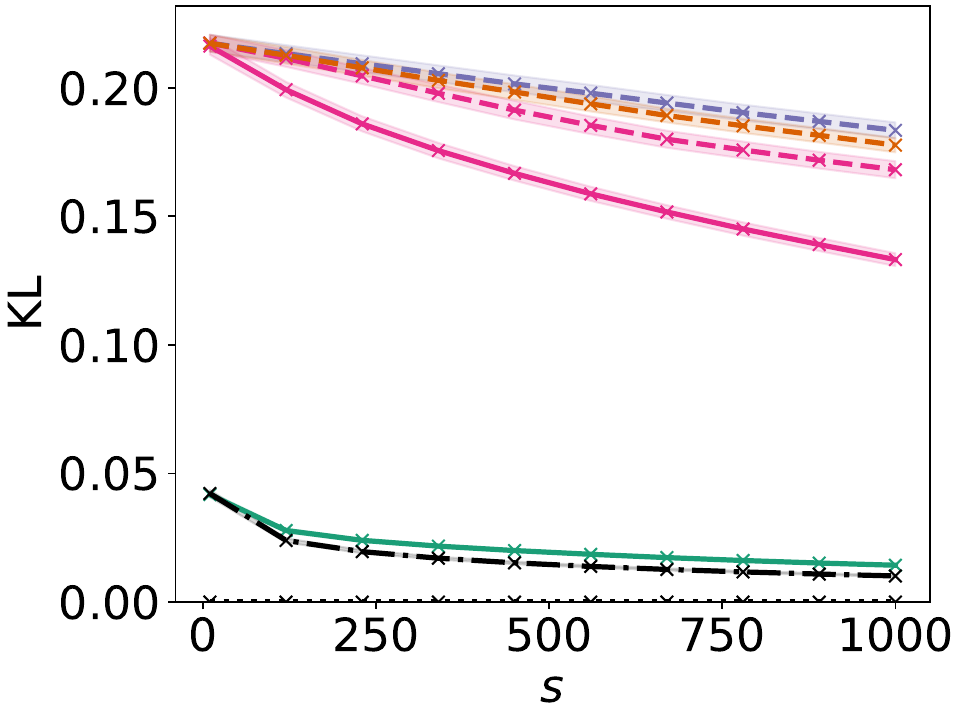}
    \end{subfigure}
    
    \vspace{0.1cm} 
    
    \begin{subfigure}{0.195\textwidth}
        \centering
        \includegraphics[width=\linewidth]{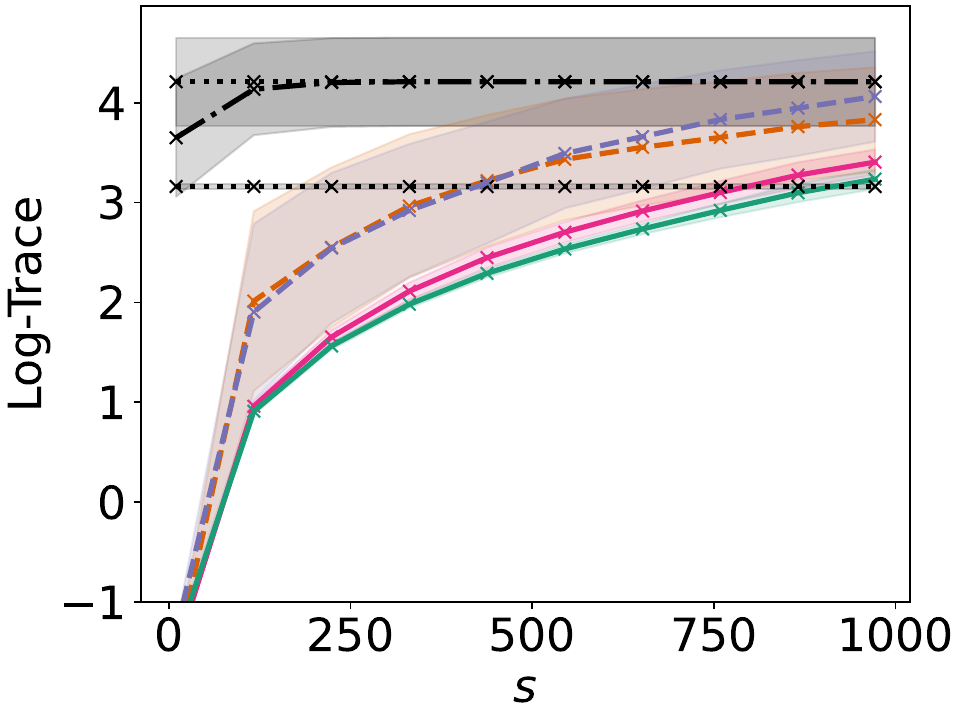}
        \caption{ENB}
    \end{subfigure}
    \hfill
    \begin{subfigure}{0.195\textwidth}
        \centering
        \includegraphics[width=\linewidth]{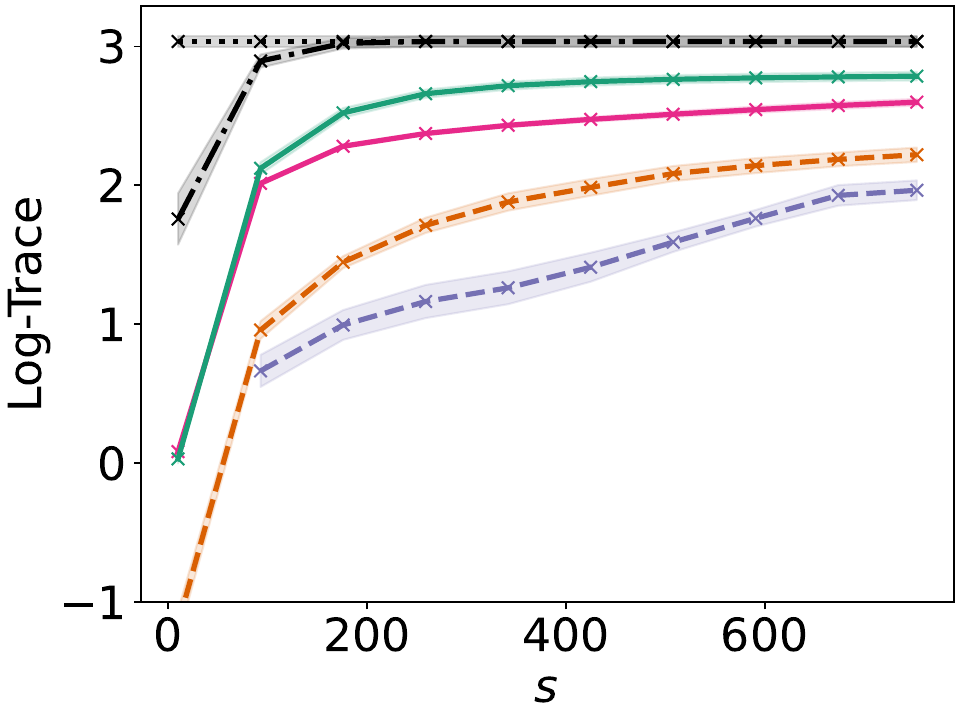}
        \caption{Red Wine}
    \end{subfigure}
    \hfill
    \begin{subfigure}{0.195\textwidth}
        \centering
        \includegraphics[width=\linewidth]{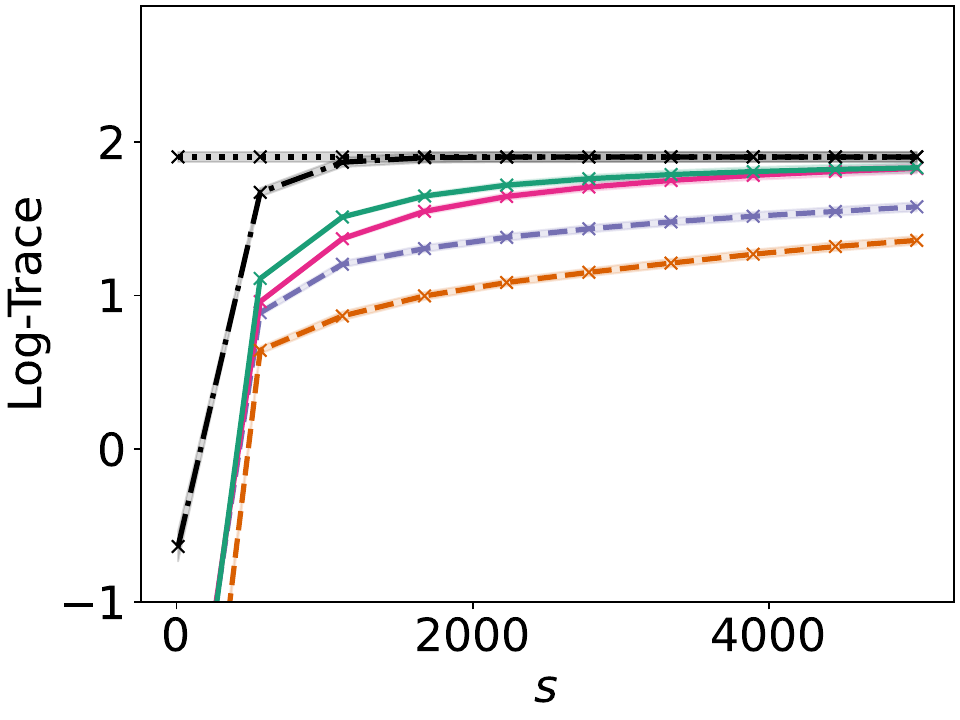}
        \caption{California}
    \end{subfigure}
    \hfill
    \begin{subfigure}{0.195\textwidth}
        \centering
        \includegraphics[width=\linewidth]{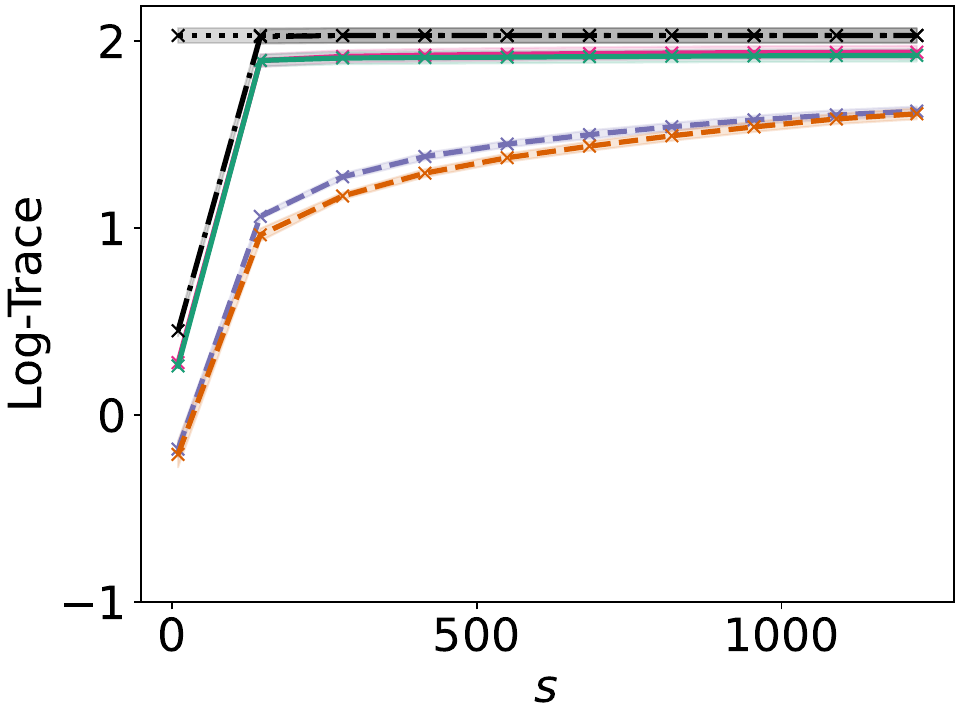}
        \caption{Naval Propulsion}
    \end{subfigure}
    \hfill
    \begin{subfigure}{0.195\textwidth}
        \centering
        \includegraphics[width=\linewidth]{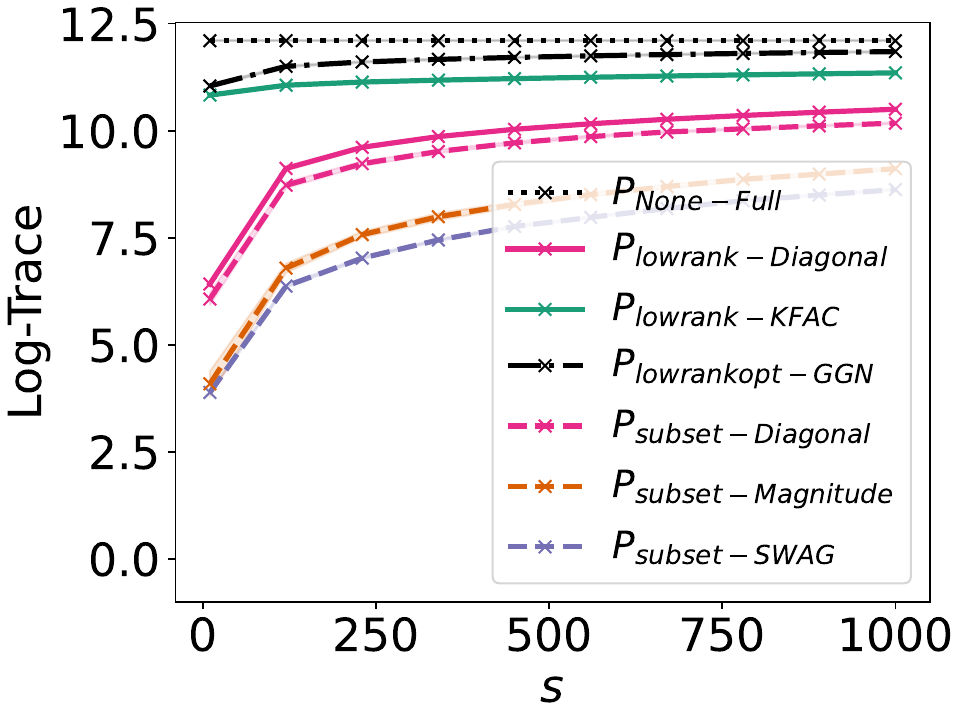}
        \caption{Fashion-MNIST}
    \end{subfigure}
    \caption{Comparison of subspace models for different datasets and dimensions $s$. Different choices of $P$ are marked by different colours and line types (cf. the legend in lower rightmost plot). The first row displays the KL-divergence \eqref{eq:KL} and the second the logarithm of the trace \eqref{eq:TraceOrderMain}. Missing values in the logarithm of trace plots have a trace of zero at these values of $s$ (e.g. $P_{\mathrm{subset-Diagonal}}$ for all regression datasets and all considered $s$ or $P_{\mathrm{subset-SWAG}}$ for the lowest $s$ in Red Wine).}
    \label{fig:plots_with_feasible_optimal_solution}
\end{figure*}

\subsection{A Feasible Approximation} 
\label{sec:thm_in_practice}

\begin{algorithm}[tb]
   \caption{Predictive Covariance Computation}
   \label{alg:prediction}
\begin{algorithmic}[1]
   \REQUIRE Projection $P$ (from Algorithm~\ref{alg:construction}), test inputs $X$
   \ENSURE Epistemic predictive covariance $\Sigma_{P,X}$
   \STATE Compute $J_X = \nabla_{\theta} f_{\theta}(X)|_{\theta=\hat{\theta}}$ (batchwise)
   \STATE Compute $V^\top P$ where $\Psi_{\mathrm{GGN}} = VV^\top$ (batchwise)
   \STATE Compute $(P^\top \Psi_{\mathrm{GGN}} P)^{-1} \in \mathbb{R}^{s \times s}$
   \STATE \textbf{Return} $\Sigma_{P,X} = J_X P (P^\top \Psi_{\mathrm{GGN}} P)^{-1} P^\top J_X^\top$
\end{algorithmic}
\end{algorithm}
Theorem \ref{thm:OptSubmodel} states that there is an optimal solution to problem \eqref{eq:minimization_problem} and it is, to the best knowledge of the authors, the first systematic solution to a subspace modelling for Bayesian NNs in the context of LA. 
However, the applicability of Theorem \ref{thm:OptSubmodel} is limited in practice, due to the following two reasons:

\textbf{Epistemic Limitation.} Training datasets $\mathcal{D}$ are often so large that computing the eigendecomposition of $\Sigma_X$ with $X=\mathcal{D}$, and thus of $U_s$, is infeasible. However, even if we can pick $X=\mathcal{D}$ we actually want the subspace model to work for unseen data points, that is data points that are not contained in $\mathcal{D}$.

\textbf{No Access to $\Psi$.} The posterior covariance $\Psi=\Psi_{\mathrm{GGN}}$ from the full LA is in general not feasible.

In practice, we will therefore use the following workflow:
\begin{enumerate}
    \item Fix an approximation $\Psi_{\mathrm{approx}}$ to $\Psi$ such as the KFAC or diagonal approximation.
    \item Use a subset $X'$ of size $n$ of the inputs in the training set to construct $J_{X'} \Psi_{\mathrm{approx}} J_{X'}^\intercal \in \RR^{nC\times nC}$ and determine its $s$ dominant eigenvectors $U_s\in \RR^{nC \times s}$.
    \item Construct $P$ via $P=\Psi_{\mathrm{approx}} J_{X'}^\intercal U_s$ (we will in this work fix $Q$ to be always the identity).
    \item For the $X$ of interest (usually not contained in the training set), compute the predictive covariance $J_X P (P^\intercal \Psi P)^{-1} P^\intercal J_X^\intercal$. Note that we can really use the GGN $\Psi$ here, since $\Psi_{\mathrm{GGN}}=VV^\intercal$ can be written as an outer product which allows for a batch-wise computation (see Appendix \ref{sec:FI_GGN_and_H}). In our experiments, we will use the full GGN for the dimensionally reduced posterior for all projectors (subset and lowrank), and we use $X$ of size $n$ that are randomly drawn from the test data.
\end{enumerate}
The first three steps outline the construction of our projection $P$. These  are summarized in Algorithm \ref{alg:construction}. And the last step, is presented in Algorithm \ref{alg:prediction}. Algorithm \ref{alg:prediction} is used for methods based on lowrank approximations as well as subset methods. 

\textbf{Limitations.} As the subspace construction in Section \ref{sec:thm_in_practice} deviates due to $X'\neq X$ and $\Psi_{\mathrm{approx}}\neq \Psi$ from the setting in Theorem \ref{thm:OptSubmodel} it should only be considered as an approximation.  Computational bottlenecks of this algorithm are discussed in Appendix 
\ref{subsec:size_of_training_data_subset}. The main bottlenecks arise from $s$ via the need to store $P\in \RR^{p\times s}$, which is unavoidable in an approach with general affine relations, and from $nC$ via the need of computing a $nC\times nC$ matrix and performing a SVD on it. Due to the last restriction only a subset of ImageNet was considered in Section \ref{sec:Experiment}. A transition to randomized algorithms or using suitable submatrices might lift such a bottleneck but was not studied in this work.

\section{EXPERIMENTS}
    \label{sec:Experiment}

We use various \texttt{OpenML} \citep{OpenML2013, OpenMLPython2019} regression datasets as well as common classification tasks such as FashionMNIST \citep{Xiao2017FashionMNIST}, CIFAR10 \citep{Krizhevsky2009cifar} and a subset of ImageNet \citep{ImageNet2009}, called ImageNet10, that contains the ten classes listed in Appendix \ref{sec:AppendixExp}, to validate different subspace methods. In addition, we study the performance of the methods on out-of-distribution (OOD) data: We chose a corrupted version of CIFAR10 \citep{hendrycks2019cifarc} for this analysis and provide some additional results on corrupted MNIST \citep{Mu2019mnistc} in Appendix \ref{subsec:corrupted_mnist}. Details about the used NNs can be found in Appendix \ref{sec:AppendixExp}. We compare the following LAs:
\begin{itemize}
    \item The projectors $P_{\mathrm{subset-Magnitude}}, P_{\mathrm{subset-Diagonal}}$ and $P_{\mathrm{subset-SWAG}}$ (dashed lines) select a subset of parameters according to the magnitude of parameters, the diagonal GGN approximation or variances produced via SWAG. We use the term \textit{subset methods} for these approximations from \cite{Daxberger2021} and \cite{LaplaceRedux2021} because they select certain parameters to construct $P$. 
    \item $P_{\mathrm{lowrank-KFAC}}$ and $P_{\mathrm{lowrank-Diagonal}}$ (solid lines) are constructed as in Section \ref{sec:thm_in_practice} and use a KFAC or a diagonal GGN approximation to estimate $\Psi$ in Step 2 of Section \ref{sec:thm_in_practice}. We use the term \textit{low rank methods} for these since Theorem \ref{thm:OptSubmodel} bases its argument on a low rank approximation. A subset of the training data was used for the construction of these subspace models, cf. Appendix \ref{subsec:size_of_training_data_subset}.
    \item Moreover, where feasible, we show results for a $P_{\mathrm{lowrankopt-GGN}}$ (dashed-dotted line) that is exactly constructed as in Theorem \ref{thm:OptSubmodel} by using the \emph{test} data and $\Psi_{\mathrm{GGN}}$ for the construction of the subspace model. 
    This is the subspace model that minimizes \eqref{eq:minimization_problem}. $P_{\mathrm{None-Full}}=\mathbb{1}_p$ (dotted line) is the full LA without any dimensional reduction.
\end{itemize}
All experiments are done with five different seeds and the average of the results is plotted with markers. To enhance the visualization, the markers are in some plots linearly interpolated by lines whose type indicates the methods used to approximate the LA. We further show the sample standard error via shaded areas or error bars. All plots use the same colour coding.

\subsection{Faithfulness of Subspace Approximation}
    \label{sec:ExpSubspaceApprQuality}

First, we study whether the subspace methods provide a faithful approximation of the full LA. As the posterior predictive distribution \eqref{eq:pred_distribution} is the object of genuine interest for predictions via Bayesian NNs, we assess whether the distributions in \eqref{eq:LaplaceSubmodel} are similar to the distributions \eqref{eq:PredDistrGauss}, \eqref{eq:PredDistrCat} obtained by the full LA. We consider two cases: One in which the actual distributions \eqref{eq:PredDistrGauss}, \eqref{eq:PredDistrCat} are known and the other in which they are unknown.

\textbf{Evaluation of $P$ When Full Model Is Known.} We would like to find a subspace model \eqref{eq:reparametrization} whose LA closely aligns  with the full LA.
To assess this, we use the Kullback-Leibler divergence between \eqref{eq:PredDistrGauss} (for regression) or \eqref{eq:PredDistrCat} (for classification) and the according distribution from \eqref{eq:LaplaceSubmodel}:
\begin{align}
    \label{eq:KL}
    \KL(p(Y|X,\mathcal{D}) \| p_P(Y|X,\mathcal{D})) \,.
\end{align}
Since the probit approximation \eqref{eq:PredDistrCat} only holds for single inputs, we use for classification the ``block diagonal approximations'' $p(Y|X,\mathcal{D})\simeq\prod_{i=1}^N p(y_i|x_i, \Sigma_{x_i})$ and $p_P(Y|X,\mathcal{D})\simeq\prod_{i=1}^N p_P(y_i|x_i, \Sigma_{P,x_i})$.

\begin{table*}[t]
\begin{small}
\centering
\begin{tabular}{lllcccccc}
\toprule
metric & dataset & $s$ & lowrank-D. & lowrank-KFAC & subset-D. & subset-M. & subset-SWAG & $\hat{\theta}$ \\
\midrule
\multirow{6}{*}{NLL} & \multirow{2}{*}{FashionMNIST} & 120 & 0.289 & \textbf{0.248} & 0.292 & 0.292 & 0.292 & 0.293 \\
 &  & 450 & 0.282 & \textbf{0.249} & 0.290 & 0.288 & 0.288 & 0.293 \\
 & \multirow{2}{*}{CIFAR10} & 120 & 0.287 & \textbf{0.263} & 0.289 & 0.287 & 0.287 & 0.289 \\
 &  & 450 & 0.281 & \textbf{0.259} & 0.289 & 0.282 & 0.283 & 0.289 \\
 & \multirow{2}{*}{ImageNet10} & 50 & 0.274 & \textbf{0.270} & 0.277 & 0.276 & 0.277 & 0.277 \\
 &  & 100 & \textbf{0.271} & 0.277 & 0.277 & 0.276 & 0.276 & 0.277 \\
\midrule
\multirow{6}{*}{ECE} & \multirow{2}{*}{FashionMNIST} & 120 & 0.039 & \textbf{0.010} & 0.040 & 0.040 & 0.040 & 0.041 \\
 &  & 450 & 0.037 & \textbf{0.013} & 0.040 & 0.039 & 0.039 & 0.041 \\
 & \multirow{2}{*}{CIFAR10} & 120 & 0.037 & \textbf{0.029} & 0.038 & 0.037 & 0.037 & 0.038 \\
 &  & 450 & 0.035 & \textbf{0.027} & 0.038 & 0.036 & 0.036 & 0.038 \\
 & \multirow{2}{*}{ImageNet10} & 50 & \textbf{0.039} & 0.041 & 0.040 & 0.040 & 0.040 & 0.040 \\
 &  & 100 & \textbf{0.038} & 0.048 & 0.040 & 0.039 & 0.040 & 0.040 \\
\midrule
\multirow{6}{*}{Brier} & \multirow{2}{*}{FashionMNIST} & 120 & 0.135 & \textbf{0.128} & 0.135 & 0.135 & 0.135 & 0.135 \\
 &  & 450 & 0.134 & \textbf{0.129} & 0.135 & 0.135 & 0.135 & 0.135 \\
 & \multirow{2}{*}{CIFAR10} & 120 & 0.125 & \textbf{0.123} & 0.126 & 0.125 & 0.125 & 0.126 \\
 &  & 450 & 0.125 & \textbf{0.122} & 0.126 & 0.125 & 0.125 & 0.126 \\
 & \multirow{2}{*}{ImageNet10} & 50 & 0.129 & \textbf{0.127} & 0.129 & 0.129 & 0.129 & 0.129 \\
 &  & 100 & 0.128 & \textbf{0.127} & 0.129 & 0.129 & 0.129 & 0.129 \\
\bottomrule
\end{tabular}
\end{small}
\caption{Evaluation of the uncertainty quantification produced of various methods for subspace inference using the NLL, ECE metric and the Brier score for the subspace inference methods studied in his article. The last column shows the performance of the model with parameter $\hat{\theta}$ from \eqref{eq:MAP}.} 
\label{tab:eval_UQ}
\end{table*}

\textbf{Evaluation of $P$ When Full Model Is Unknown.}  The KL-divergence \eqref{eq:KL} quantifies the deviation of the subspace model to the full LA. As the latter is in general not feasible, we also consider the logarithm of the trace of the epistemic predictive covariances as a proxy. The predictive distributions for full and subspace models are essentially fixed by their covariances with identity between both distributions for the case $\Sigma_{P,X}=\Sigma_X$.
Heuristically, $\Sigma_{P,X}$ approximates better $\Sigma_X$ if it contains the dominant eigenspace, because in the directions of these eigenvectors the covariance has its largest contributions. Hence, we propose to use the following \emph{trace criterion}: If 
\begin{equation}
    \label{eq:TraceOrderMain}
    0 \leq \Tr \Sigma_{P_1,X} < \Tr \Sigma_{P_2,X} \leq \Tr \Sigma_{X}
\end{equation}
holds, $P_2$ is a better projector than $P_1$. A larger trace value indicates that the more dominant eigenspace is captured for a given $P$. The proof of $\Tr \Sigma_{P,X} \leq \Tr \Sigma_X$ and an extended explanation are given in Appendix \ref{sec:AppendixTraceCriterion}. In Appendix \ref{subsec:additional_remark_to_theorem} we show that the $P^\ast$ from Theorem \ref{thm:OptSubmodel} is also optimal under the trace criterion.
Empirically, we validate the trace criterion on cases in which we can compute \eqref{eq:PredDistrGauss}, \eqref{eq:PredDistrCat} exactly.

\textbf{Datasets With Feasible Baseline.} Figure \ref{fig:plots_with_feasible_optimal_solution} shows the KL-divergence \ref{eq:KL} and the logarithm of the trace criterion \eqref{eq:TraceOrderMain} for several datasets and the subspace models listed above for different $s$. All datasets in Figure \ref{fig:plots_with_feasible_optimal_solution} are such that the full LA and the baseline $P_{\mathrm{lowrankopt-GGN}}$ are feasible.
For ENB, Red Wine and Naval Propulsion the Jacobian is rank-deficient, so that only $s$ up to the rank of the Jacobian on the training data are considered. First, we observe for all datasets that the baseline $P_{\mathrm{lowrankopt-GGN}}$ (black dashed-dotted line) needs only a fraction of the total number of model parameters, listed in Table \ref{tab:NumParamsModel}, to reach a small KL-divergence with the full model. Hence, subspace models can indeed be suitable for quantifying the uncertainty provided by a LA. However, $P_{\mathrm{lowrankopt-GGN}}$ is usually unknown such that the ideal approximation isn't available. 
Comparing the feasible approximations in Figure \ref{fig:plots_with_feasible_optimal_solution} we find that low rank approximations demonstrate superior approximations compared to subset methods in general. In particular, the performance of $P_{\mathrm{lowrank-Diagonal}}$ is stricly superior to $P_{\mathrm{subset-Diagonal}}$. Only for ENB some subset methods obtain a slightly better performance. We speculate that the different performance on this dataset is related to the number of `dead parameters' (with gradient almost zero), whose complement provides a natural subset to be selected. Indeed, ENB has the most number of dead parameters with 93\%. More details on this investigation are given in Appendix \ref{sec:AppendixDeadParam}.

\begin{table*}
\centering
\begin{small}
\begin{tabular}{llcccccc}
\toprule
metric & corruption & lowrank-D. & lowrank-KFAC & subset-D. & subset-M. & subset-SWAG & $\hat{\theta}$ \\
\midrule
\multirow{4}{*}{NLL} & brightness & 0.340 & \textbf{0.313} & 0.344 & 0.340 & 0.340 & 0.344 \\
 & elastic transform & 0.522 & \textbf{0.490} & 0.527 & 0.522 & 0.522 & 0.527 \\
 & gaussian blur & 0.353 & \textbf{0.341} & 0.354 & 0.352 & 0.353 & 0.354 \\
 & impulse noise & 1.822 & \textbf{1.660} & 1.861 & 1.836 & 1.831 & 1.861 \\
\midrule
\multirow{4}{*}{ECE} & brightness & 0.044 & \textbf{0.035} & 0.046 & 0.044 & 0.045 & 0.046 \\
 & elastic transform & 0.063 & \textbf{0.050} & 0.065 & 0.064 & 0.064 & 0.065 \\
 & gaussian blur & 0.031 & \textbf{0.022} & 0.032 & 0.031 & 0.031 & 0.032 \\
 & impulse noise & 0.227 & \textbf{0.203} & 0.233 & 0.229 & 0.229 & 0.233 \\
\midrule
\multirow{4}{*}{Brier} & brightness & 0.149 & \textbf{0.145} & 0.149 & 0.149 & 0.149 & 0.149 \\
 & elastic transform & 0.230 & \textbf{0.225} & 0.231 & 0.230 & 0.230 & 0.231 \\
 & gaussian blur & 0.165 & \textbf{0.164} & 0.166 & 0.165 & 0.165 & 0.166 \\
 & impulse noise & 0.612 & \textbf{0.595} & 0.616 & 0.614 & 0.613 & 0.616 \\
\bottomrule
\end{tabular}
\end{small}
\caption{Performance of various subspace inference methods on corrupted versions of CIFAR10 for $s=200$. The method ``subset-Diagonal'' was omitted here as it yielded results indistinguishable from the model at the MAP $\hat{\theta}$ due to a negligible uncertainty. }
\label{tab:OOD}
\end{table*}

\begin{figure}[h]
    \centering
    \begin{subfigure}[b]{0.195\textwidth}
        \includegraphics[width=\textwidth]{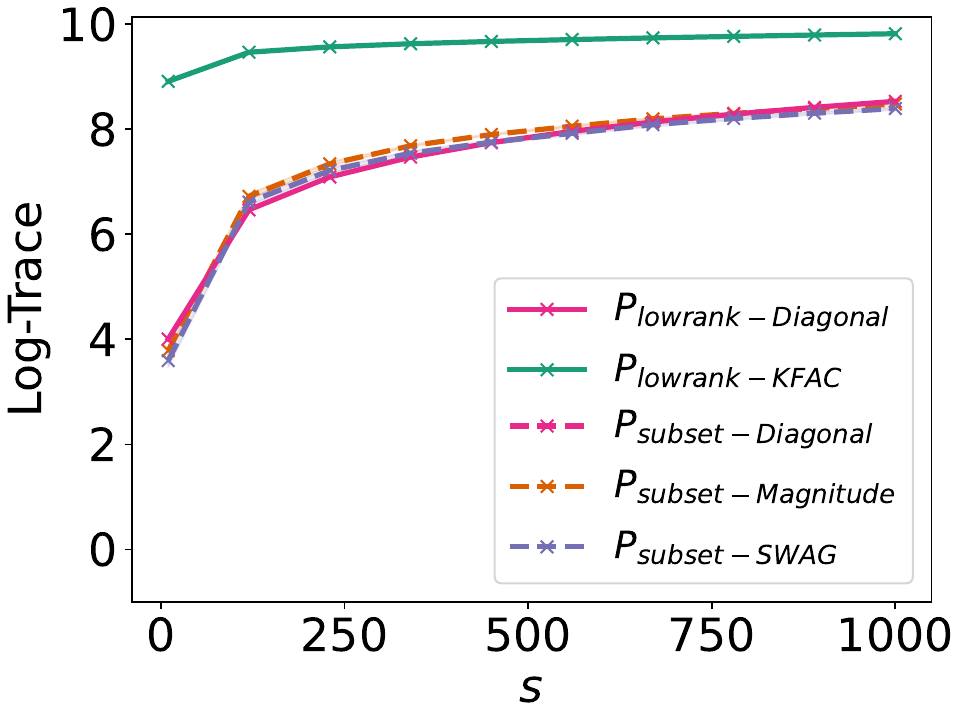}
        \caption{CIFAR10}
        \label{subfig:logtrace_cifar10_resnet9}
    \end{subfigure}
    \hfill
    \begin{subfigure}[b]{0.195\textwidth}
        \includegraphics[width=\textwidth]{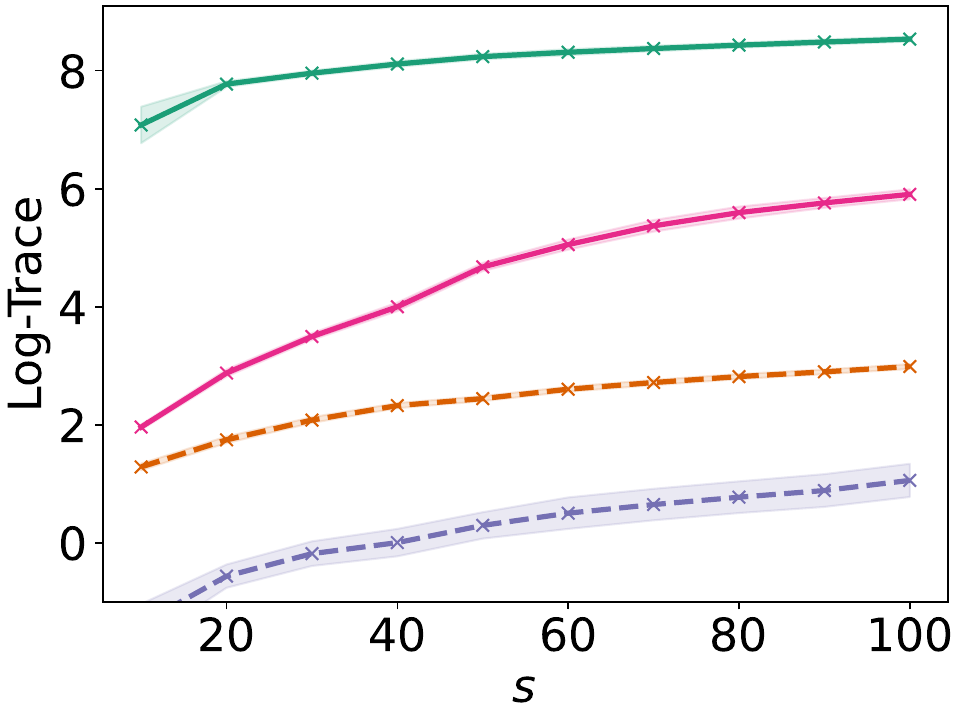}
        \caption{ImageNet10}
        \label{subfig:logtrace_imagenet10_resnet18}
    \end{subfigure}
    \caption{Evaluation the trace criterion \eqref{eq:TraceOrderMain} for CIFAR10 and ImageNet10 and different choices of $P$. The colour coding and linestyles are chosen as in Figure \ref{fig:plots_with_feasible_optimal_solution}.}
    \label{fig:trace_cifarImage_dimensional}
\end{figure}
\textbf{Datasets Without Feasible Baseline.} A comparison between the first and second row of Figure \ref{fig:plots_with_feasible_optimal_solution} demonstrates that the log-trace retains the ordering of the KL-divergence. Differences are rare, and if they occur, they are small and usually contained in the sample standard deviation. Thus, we use the log-trace to assess the approximation quality of the full LA for ResNet9 and ResNet18 applied on CIFAR10 and ImageNet10, respectively. For both networks the number of parameters is so large that the computation of the full LA is infeasible. Both experiments indicate again a better performance of the low-rank based methods, especially of $P_{\mathrm{lowrank-KFAC}}$: As revealed by the trace criterion in Figure \ref{fig:trace_cifarImage_dimensional} the eigenspace of $\Sigma_{P,X}$ spanned by the selected eigenvectors in parameter space is orders of magnitude higher for $P=P_{\mathrm{lowrank-KFAC}}$ than for all other considered methods.

\subsection{Uncertainty Quantification}
    \label{sec:ExpUncertainty}

To assess the uncertainty quantification of the studied methods, we study the common ``negative log-likelihood'' (NLL) criterion, which is the average of the values $- \ln p_{P}(y^*|x^*,\mathcal{D})$ for all pairs $(x^*,y^*)$ from the test data. For classification, we also include the expected calibration error (ECE), which measures the discrepancy between predicted confidence and empirical accuracy across probability bins, and the Brier score with the convention outlined in Appendix \ref{sec:addtional_results_on_UQ}.

For the regression problems studied in this article the NLL values are essentially identical within their standard errors across methods, cf. Appendix \ref{sec:addtional_results_on_UQ}. Results for the classification datasets considered in this work are contained in Table \ref{tab:eval_UQ} and also compared to the performance of the model at the MAP \eqref{eq:MAP}.  All three metrics, NLL, ECE and Brier score, indicate a better performance of the lowrank methods. 

\subsection{Performance on OOD Data}
    \label{sec:ExpOOD}
OOD data analysis reveals how well different BNNs approximation methods maintain meaningful uncertainty quantification when encountering inputs that differ from the training distribution. We trained a NN on the original CIFAR10 \citep{Krizhevsky2009cifar} dataset and applied it to corrupted versions from \cite{hendrycks2019cifarc}. The results are shown in Table \ref{tab:OOD} and show a better performance of lowrank methods, especially of $P_{\mathrm{lowrank-KFAC}}$. In Appendix \ref{subsec:corrupted_mnist} we also provide results for corrupted versions of MNIST from \cite{Mu2019mnistc}.

\section{CONCLUSION}

In this work we propose to look at subspace Laplace approximations of Bayesian neural networks through the lens of their predictive covariances. Using low rank techniques this approach allows us to derive an explicit relation for a subspace model that would be optimal given a set of datapoints. This subspace model can be used, where feasible, as a baseline for subspace inference. 

We propose a construction of a subspace model that is scalable but still conceptually based on our theoretical insights. Starting with a preliminary estimate of the predictive covariance we can use the derived relation to get an estimate of relevant directions in parameter space. In a second step these directions can be combined with the generalized Gauss-Newton (GGN) Laplace approximation in such a way that they do not require the computation of the, typically infeasible, GGN matrix.

Studying the performance of approximating the full Laplace approximation and of the resulting uncertainty quantification,  we find that this construction outperforms existing methods with respect to various metrics and is in some cases even close to the baseline constructed via the test data. Moreover, we observe that a well chosen method for subspace construction can often have more impact on the performance than an increase in the subspace dimension.

\newpage

\newpage
\subsubsection*{Acknowledgements}
The research leading to this work was funded by the “Metrology for Artificial Intelligence in Medicine” (M4AIM) program in the frame of the QI-Digital initiative.

The authors acknowledge with gratitude the contributions of their colleague, Clemens Elster, who took part in discussions leading to this article but passed away before its completion.
We will deeply miss his scientific contributions and the discussions with him as well as his personal support and advice.

\bibliographystyle{apalike} 
\bibliography{OptProj}

@book{Bishop2006,
  author       = {Christopher M. Bishop},
  title        = {Pattern recognition and machine learning, 5th Edition},
  series       = {Information science and statistics},
  publisher    = {Springer},
  year         = {2007},
  url          = {https://www.worldcat.org/oclc/71008143},
  isbn         = {9780387310732}
}

@article{Gal2016,
  title={Uncertainty in Deep Learning},
  author={Yarin Gal},
  year={2016},
  url={https://www.cs.ox.ac.uk/people/yarin.gal/website/thesis/thesis.pdf}
}

@article{Snoek2015,
      title={Scalable {B}ayesian Optimization Using Deep Neural Networks}, 
      author={Jasper Snoek and Oren Rippel and Kevin Swersky and Ryan Kiros and Nadathur Satish and Narayanan Sundaram and Md. Mostofa Ali Patwary and Prabhat and Ryan P. Adams},
      year={2015},
      journal={arXiv preprint: 1502.05700},
      url={https://arxiv.org/abs/1502.05700}, 
}

@article{Wainwright2008,
  title={Graphical Models, Exponential Families, and Variational Inference},
  author={Martin J. Wainwright and Michael I. Jordan},
  journal={Found. Trends Mach. Learn.},
  year={2008},
  volume={1},
  pages={1-305}
}

@article{Jordan1999,
  title={An Introduction to Variational Methods for Graphical Models},
  author={Michael I. Jordan and Zoubin Ghahramani and T. Jaakkola and Lawrence K. Saul},
  journal={Machine Learning},
  year={1999},
  volume={37},
  pages={183-233},
 }

@article{MacKay1992,
  title={A Practical {B}ayesian Framework for Backpropagation Networks},
  author={David John Cameron MacKay},
  journal={Neural Computation},
  year={1992},
  volume={4},
  pages={448-472}
}

@article{MacKay1992Class,
  title={The Evidence Framework Applied to Classification Networks},
  author={David John Cameron MacKay},
  journal={Neural Computation},
  year={1992},
  volume={4},
  pages={720-736}
}

@inproceedings{Ritter2018,
  title={A Scalable {Laplace} Approximation for Neural Networks},
  author={Hippolyt Ritter and Aleksandar Botev and David Barber},
  booktitle={International Conference on Learning Representations},
  year={2018}
}

@article{Salimans2016,
  title={Weight Normalization: A Simple Reparameterization to Accelerate Training of Deep Neural Networks},
  author={Tim Salimans and Diederik P. Kingma},
  journal={arXiv preprint: 1602.07868},
  url={https://arxiv.org/abs/1602.07868},
  year={2016}
}

@article{Foong2019,
      title={'In-Between' Uncertainty in {B}ayesian Neural Networks}, 
      author={Andrew Y. K. Foong and Yingzhen Li and José Miguel Hernández-Lobato and Richard E. Turner},
      year={2019},
      journal={arXiv preprint: 1906.11537},
      url={https://arxiv.org/abs/1906.11537}
}

@article{kristiadi2020,
      title={Being {B}ayesian, Even Just a Bit, Fixes Overconfidence in ReLU Networks}, 
      author={Agustinus Kristiadi and Matthias Hein and Philipp Hennig},
      year={2020},
      journal={arXiv preprint: 2002.10118},
      url={https://arxiv.org/abs/2002.10118}, 
}

@article{Immer2021,
      title={Improving predictions of {B}ayesian neural nets via local linearization}, 
      author={Alexander Immer and Maciej Korzepa and Matthias Bauer},
      year={2021},
      journal={arXiv preprint: 2008.08400},
      url={https://arxiv.org/abs/2008.08400}
}

@article{Deng2022,
  title={Accelerated Linearized {Laplace} Approximation for {B}ayesian Deep Learning},
  author={Zhijie Deng and Feng Zhou and Jun Zhu},
  journal={ArXiv},
  year={2022},
  volume={abs/2210.12642}
}

@article{Ortega2023,
  title={Variational Linearized {Laplace} Approximation for {B}ayesian Deep Learning},
  author={Luis A. Ortega and Sim{\'o}n Rodr{\'i}guez Santana and Daniel Hern'andez-Lobato},
  journal={ArXiv},
  year={2023},
  volume={abs/2302.12565},
}

@article{Papyan2018,
  title={The Full Spectrum of Deepnet {H}essians at Scale: Dynamics with SGD Training and Sample Size.},
  author={Vardan Papyan},
  journal={arXiv preprint: 1811.07062},
  url={https://arxiv.org/abs/1811.07062},
  year={2018}
}

@article{Sagun2016,
  title={Eigenvalues of the {H}essian in Deep Learning: Singularity and Beyond},
  author={Levent Sagun and L{\'e}on Bottou and Yann LeCun},
  journal={arXiv preprint: 1611.07476},
  url={https://arxiv.org/abs/1611.07476},
  year={2016}
}

@article{Dauphin2014,
  author       = {Yann N. Dauphin and
                  Razvan Pascanu and
                  {\c{C}}aglar G{\"{u}}l{\c{c}}ehre and
                  Kyunghyun Cho and
                  Surya Ganguli and
                  Yoshua Bengio},
  title        = {Identifying and attacking the saddle point problem in high-dimensional
                  non-convex optimization},
  journal      = {arXiv preprint: 1406.2572},
  year         = {2014},
  url          = {http://arxiv.org/abs/1406.2572},
}

@inproceedings{Martens2015,
  title={Optimizing Neural Networks with {Kronecker}-factored Approximate Curvature},
  author={James Martens and Roger Baker Grosse},
  booktitle={International Conference on Machine Learning},
  year={2015}}

@inproceedings{Botev2017,
  title={Practical {Gauss}-{Newton} Optimisation for Deep Learning},
  author={Aleksandar Botev and Hippolyt Ritter and David Barber},
  booktitle={International Conference on Machine Learning},
  year={2017}
}

@article{Martens14,
  author       = {James Martens},
  title        = {New insights and perspectives on the natural gradient method},
  journal      = {arxiv preprint: 1412.1193},
  year         = {2014},
  url          = {http://arxiv.org/abs/1412.1193}
}

@article{Schraudolph2002,
    author = {Schraudolph, Nicol N.},
    title = "{Fast Curvature Matrix-Vector Products for Second-Order Gradient Descent}",
    journal = {Neural Computation},
    volume = {14},
    number = {7},
    pages = {1723-1738},
    year = {2002},
    month = {07},
    issn = {0899-7667},
    doi = {10.1162/08997660260028683}
}

@article{Heskes2000,
  title={On Natural Learning and Pruning in Multilayered Perceptrons},
  author={Tom M. Heskes},
  journal={Neural Computation},
  year={2000},
  volume={12},
  pages={881-901}
}

@article{Pascanu2013,
  title={Revisiting Natural Gradient for Deep Networks},
  author={Razvan Pascanu and Yoshua Bengio},
  journal={arxiv preprint: 1301.3584},
  year={2013}
}

@article{Izmailov2019,
  author       = {Pavel Izmailov and
                  Wesley J. Maddox and
                  Polina Kirichenko and
                  Timur Garipov and
                  Dmitry P. Vetrov and
                  Andrew Gordon Wilson},
  title        = {Subspace Inference for {B}ayesian Deep Learning},
  journal      = {arxiv preprint: 1907.07504},
  year         = {2019},
}

@article{2020Lee,
  title={Estimating Model Uncertainty of Neural Networks in Sparse Information Form},
  author={Jongseo Lee and Matthias Humt and Jianxiang Feng and Rudolph Triebel},
  journal={ArXiv},
  year={2020},
  volume={abs/2006.11631},
 }

@inproceedings{LaplaceRedux2021,
  title={{Laplace} Redux - Effortless {B}ayesian Deep Learning},
  author={Erik Daxberger and Agustinus Kristiadi and Alexander Immer and Runa Eschenhagen and M. Bauer and Philipp Hennig},
  booktitle={Neural Information Processing Systems},
  year={2021},
}

@inproceedings{Daxberger2021,
  title={Bayesian deep learning via subnetwork inference},
  author={Daxberger, Erik and Nalisnick, Eric and Allingham, James U and Antor{\'a}n, Javier and Hern{\'a}ndez-Lobato, Jos{\'e} Miguel},
  booktitle={International Conference on Machine Learning},
  pages={2510--2521},
  year={2021},
  organization={PMLR}
}

@article{calvoordonez2024,
      title={Partially Stochastic Infinitely Deep {B}ayesian Neural Networks}, 
      author={Sergio Calvo-Ordonez and Matthieu Meunier and Francesco Piatti and Yuantao Shi},
      year={2024},
      journal={arxiv preprint: 2402.03495},
      url={https://arxiv.org/abs/2402.03495}
}

@article{Sharma2023,
      title={Do {B}ayesian Neural Networks Need To Be Fully Stochastic?}, 
      author={Mrinank Sharma and Sebastian Farquhar and Eric Nalisnick and Tom Rainforth},
      year={2023},
      journal={arxiv preprint: 2211.06291},
      url={https://arxiv.org/abs/2211.06291}, 
}

@article{Andrade2024,
  title    = "On the effectiveness of partially deterministic Bayesian neural
              networks",
  author   = "Andrade, Daniel and Sato, Koki",
  journal  = "Computational Statistics",
  month    =  oct,
  year     =  2024
}

@inproceedings{2024Zhao,
  title={On Feynman-Kac training of partial Bayesian neural networks},
  author={Zheng Zhao and Sebastian Mair and Thomas B. Sch{\"o}n and Jens Sj{\"o}lund},
  booktitle={International Conference on Artificial Intelligence and Statistics},
  year={2023},
}

@techreport{Krizhevsky2009cifar,
  title={Learning Multiple Layers of Features from Tiny Images},
  author={Alex Krizhevsky},
  year={2009},
}

@misc{hendrycks2019cifarc,
  author       = {Dan Hendrycks and
                  Thomas G. Dietterich},
  title        = {Benchmarking Neural Network Robustness to Common Corruptions and Perturbations},
  booktitle    = {7th International Conference on Learning Representations, {ICLR} 2019,
                  New Orleans, LA, USA, May 6-9, 2019},
  publisher    = {OpenReview.net},
  year         = {2019},
}

@article{LeCun1998mnist,
  title={Gradient-based learning applied to document recognition},
  author={Yann LeCun and L{\'e}on Bottou and Yoshua Bengio and Patrick Haffner},
  journal={Proc. IEEE},
  year={1998},
  volume={86},
  pages={2278-2324}
}

@article{Xiao2017FashionMNIST,
  title={Fashion-MNIST: a Novel Image Dataset for Benchmarking Machine Learning Algorithms},
  author={Han Xiao and Kashif Rasul and Roland Vollgraf},
  journal={ArXiv},
  year={2017},
  volume={abs/1708.07747}
}

@misc{Mu2019mnistc,
      title={MNIST-C: A Robustness Benchmark for Computer Vision}, 
      author={Norman Mu and Justin Gilmer},
      year={2019},
      eprint={1906.02337},
      archivePrefix={arXiv},
      primaryClass={cs.CV},
      url={https://arxiv.org/abs/1906.02337}, 
}

@INPROCEEDINGS{ImageNet2009,
  author={Deng, Jia and Dong, Wei and Socher, Richard and Li, Li-Jia and Kai Li and Li Fei-Fei},
  booktitle={2009 IEEE Conference on Computer Vision and Pattern Recognition}, 
  title={ImageNet: A large-scale hierarchical image database}, 
  year={2009},
  volume={},
  number={},
  pages={248-255},
  doi={10.1109/CVPR.2009.5206848}}

@article{OpenML2013,
author = {Vanschoren, Joaquin and van Rijn, Jan N. and Bischl, Bernd and Torgo, Luis},
title = {OpenML: Networked Science in Machine Learning},
journal = {SIGKDD Explorations},
volume = {15},
number = {2},
year = {2013},
pages = {49--60},
url = {http://doi.acm.org/10.1145/2641190.2641198},
doi = {10.1145/2641190.2641198},
publisher = {ACM},
address = {New York, NY, USA},
}

@article{OpenMLPython2019,
author = {Matthias Feurer and Jan N. van Rijn and Arlind Kadra and Pieter Gijsbers and Neeratyoy Mallik and Sahithya Ravi and Andreas Mueller and Joaquin Vanschoren and Frank Hutter},
title = {{OpenML-Python}: an extensible {Python API for OpenML}},
year=2019,
journal = {arXiv},
volume = {1911.02490},
url = {https://arxiv.org/pdf/1911.02490.pdf},
}

@article{Cheng2017,
  title={A Survey of Model Compression and Acceleration for Deep Neural Networks},
  author={Yu Cheng and Duo Wang and Pan Zhou and Zhang Tao},
  journal={ArXiv},
  year={2017},
  volume={abs/1710.09282}
}

@Article{Schmidt1907,
author={Schmidt, Erhard},
title={Zur {T}heorie der linearen und nichtlinearen {I}ntegralgleichungen},
journal={Mathematische Annalen},
year={1907},
month={Dec},
day={01},
volume={63},
number={4},
pages={433-476},
issn={1432-1807},
doi={10.1007/BF01449770},
url={https://doi.org/10.1007/BF01449770}
}

@Article{Eckart1936,
author={Eckart, Carl and Young, Gale},
title={The approximation of one matrix by another of lower rank},
journal={Psychometrika},
year={1936},
month={Sep},
day={01},
volume={1},
number={3},
pages={211-218},
issn={1860-0980},
doi={10.1007/BF02288367},
url={https://doi.org/10.1007/BF02288367}
}

@article{Mirsky1960,
    author = {Mirsky, L.},
    title = {Symmetric Gauge Functions and Unitarily Invariant Norms},
    journal = {The Quarterly Journal of Mathematics},
    volume = {11},
    number = {1},
    pages = {50-59},
    year = {1960},
    month = {01},
    issn = {0033-5606},
    doi = {10.1093/qmath/11.1.50},
    url = {https://doi.org/10.1093/qmath/11.1.50},
    eprint = {https://academic.oup.com/qjmath/article-pdf/11/1/50/7295335/11-1-50.pdf},
}

@article{kingma2015variational,
  title={Variational dropout and the local reparameterization trick},
  author={Kingma, Diederik P and Salimans, Tim and Welling, Max},
  journal={arXiv preprint arXiv:1506.02557},
  year={2015}
}

@inproceedings{blundell2015weight,
	title={Weight uncertainty in neural network},
	author={Blundell, Charles and Cornebise, Julien and Kavukcuoglu, Koray and Wierstra, Daan},
	booktitle={International Conference on Machine Learning},
	pages={1613--1622},
	year={2015},
	organization={PMLR}
}

@inproceedings{hernandez2015probabilistic,
  title={Probabilistic backpropagation for scalable learning of {B}ayesian neural networks},
  author={Hern{\'a}ndez-Lobato, Jos{\'e} Miguel and Adams, Ryan},
  booktitle={International conference on machine learning},
  pages={1861--1869},
  year={2015},
  organization={PMLR}
}

@article{kendall2017uncertainties,
  title={What uncertainties do we need in {B}ayesian deep learning for computer vision?},
  author={Kendall, Alex and Gal, Yarin},
  journal={arXiv preprint arXiv:1703.04977},
  year={2017}
}

@article{He2015resnet,
  author       = {Kaiming He and
                  Xiangyu Zhang and
                  Shaoqing Ren and
                  Jian Sun},
  title        = {Deep Residual Learning for Image Recognition},
  journal      = {arXiv preprint arXiv:1512.03385},
  year         = {2015},
  url          = {http://arxiv.org/abs/1512.03385},
  eprinttype    = {arXiv}
}

@misc{PyTorch2019,
title = {PyTorch: An Imperative Style, High-Performance Deep Learning Library},
author = {Paszke, Adam and Gross, Sam and Massa, Francisco and Lerer, Adam and Bradbury, James and Chanan, Gregory and Killeen, Trevor and Lin, Zeming and Gimelshein, Natalia and Antiga, Luca and Desmaison, Alban and Kopf, Andreas and Yang, Edward and DeVito, Zachary and Raison, Martin and Tejani, Alykhan and Chilamkurthy, Sasank and Steiner, Benoit and Fang, Lu and Bai, Junjie and Chintala, Soumith},
booktitle = {Advances in Neural Information Processing Systems 32},
pages = {8024--8035},
year = {2019},
publisher = {Curran Associates, Inc.}
}

@article{maddox2019simple,
  title={A simple baseline for {B}ayesian uncertainty in deep learning},
  author={Maddox, Wesley J and Izmailov, Pavel and Garipov, Timur and Vetrov, Dmitry P and Wilson, Andrew Gordon},
  journal={Advances in neural information processing systems},
  volume={32},
  year={2019}
}

@InProceedings{Izmailov2021BNNReallyLike,
  title = 	 {What Are Bayesian Neural Network Posteriors Really Like?},
  author =       {Izmailov, Pavel and Vikram, Sharad and Hoffman, Matthew D and Wilson, Andrew Gordon Gordon},
  booktitle = 	 {Proceedings of the 38th International Conference on Machine Learning},
  pages = 	 {4629--4640},
  year = 	 {2021},
  editor = 	 {Meila, Marina and Zhang, Tong},
  volume = 	 {139},
  series = 	 {Proceedings of Machine Learning Research},
  month = 	 {18--24 Jul},
  publisher =    {PMLR},
  url = 	 {https://proceedings.mlr.press/v139/izmailov21a.html},
}

@article{lecun1989optimal,
  title={Optimal brain damage},
  author={LeCun, Yann and Denker, John and Solla, Sara},
  journal={Advances in neural information processing systems},
  volume={2},
  year={1989}
}

@article{kirkpatrick2017overcoming,
  title={Overcoming catastrophic forgetting in neural networks},
  author={Kirkpatrick, James and Pascanu, Razvan and Rabinowitz, Neil and Veness, Joel and Desjardins, Guillaume and Rusu, Andrei A and Milan, Kieran and Quan, John and Ramalho, Tiago and Grabska-Barwinska, Agnieszka and others},
  journal={Proceedings of the national academy of sciences},
  volume={114},
  number={13},
  pages={3521--3526},
  year={2017},
  publisher={National Acad Sciences}
}

@article{hartfiel1995dense,
  title={Dense sets of diagonalizable matrices},
  author={Hartfiel, Darald J},
  journal={Proceedings of the American Mathematical Society},
  volume={123},
  number={6},
  pages={1669--1672},
  year={1995}
}

@inproceedings{dold2024bayesian,
  title={Bayesian semi-structured subspace inference},
  author={Dold, Daniel and R{\"u}gamer, David and Sick, Beate and D{\"u}rr, Oliver},
  booktitle={International Conference on Artificial Intelligence and Statistics},
  pages={1819--1827},
  year={2024},
  organization={PMLR}
}

@inproceedings{
yang2024bayesian,
title={Bayesian Low-rank Adaptation for Large Language Models},
author={Adam X. Yang and Maxime Robeyns and Xi Wang and Laurence Aitchison},
booktitle={The Twelfth International Conference on Learning Representations},
year={2024},
url={https://openreview.net/forum?id=FJiUyzOF1m}
}

@article{marshall1979inequalities,
  title={Inequalities: theory of majorization and its applications},
  author={Marshall, Albert W and Olkin, Ingram and Arnold, Barry C},
  year={1979},
  publisher={Springer}
}

@article{szekely2003extremal,
  title={Extremal probabilities for Gaussian quadratic forms},
  author={Sz{\'e}kely, G{\'a}bor J and Bakirov, Nail K},
  journal={Probability theory and related fields},
  volume={126},
  number={2},
  pages={184--202},
  year={2003},
  publisher={Springer}
}

@book{horn2012matrix,
  title={Matrix analysis},
  author={Horn, Roger A and Johnson, Charles R},
  year={2012},
  publisher={Cambridge university press}
}

@book{hardy1988inequalities,
  title={Inequalities},
  author={Hardy, G.H. and Littlewood, J.E. and P{\'o}lya, G.},
  isbn={9781107647398},
  series={Cambridge Mathematical Library},
  url={https://books.google.de/books?id=EfvZAQAAQBAJ},
  year={1988},
  publisher={Cambridge University Press}
}

\section*{Checklist}
\begin{enumerate}
  \item For all models and algorithms presented, check if you include:
  \begin{enumerate}
    \item A clear description of the mathematical setting, assumptions, algorithm, and model. 
    \textit{Yes, cf. Section \ref{sec:terminlogy_and_background}, \ref{sec:LA_for_subspace_models} and Appendix \ref{sec:AppendixExp}.}
    \item An analysis of the properties and complexity (time, space, sample size) of any algorithm. 
    \textit{Yes, we included a section on computational costs, cf. Section \ref{Sec:Computational_Costs} and other computational aspects are discussed in Section \ref{sec:thm_in_practice} and \ref{subsec:size_of_training_data_subset}.}
    \item (Optional) Anonymized source code, with specification of all dependencies, including external libraries. 
    \textit{Yes, the source code is included in the supplementary material.}
  \end{enumerate}

  \item For any theoretical claim, check if you include:
  \begin{enumerate}
    \item Statements of the full set of assumptions of all theoretical results. \textit{Yes, cf. Section \ref{sec:terminlogy_and_background} - \ref{sec:LA_for_subspace_models}.}
    \item Complete proofs of all theoretical results. \textit{Yes, cf. Appendix \ref{sec:AppendixProofSubmodel}.}
    \item Clear explanations of any assumptions. \textit{Yes, cf. Section \ref{sec:terminlogy_and_background} - \ref{sec:LA_for_subspace_models}.}
  \end{enumerate}

  \item For all figures and tables that present empirical results, check if you include:
  \begin{enumerate}
    \item The code, data, and instructions needed to reproduce the main experimental results (either in the supplemental material or as a URL). \textit{Yes, source code for this will be part of the supplementary material.}
    \item All the training details (e.g., data splits, hyperparameters, how they were chosen). \textit{Yes, cf. Section \ref{sec:AppendixExp} and the source code in the supplementary material.}
    \item A clear definition of the specific measure or statistics and error bars (e.g., with respect to the random seed after running experiments multiple times). \textit{Yes, cf. Section \ref{sec:Experiment}.}
    \item A description of the computing infrastructure used. (e.g., type of GPUs, internal cluster, or cloud provider). \textit{Yes, cf. Section \ref{Sec:Computational_Costs}.}
  \end{enumerate}

  \item If you are using existing assets (e.g., code, data, models) or curating/releasing new assets, check if you include:
  \begin{enumerate}
    \item Citations of the creator If your work uses existing assets. \textit{Yes.}
    \item The license information of the assets, if applicable. \textit{Not Applicable.}
    \item New assets either in the supplemental material or as a URL, if applicable. \textit{Not Applicable.}
    \item Information about consent from data providers/curators. \textit{Not Applicable.}
    \item Discussion of sensible content if applicable, e.g., personally identifiable information or offensive content. \textit{Not Applicable.}
  \end{enumerate}

  \item If you used crowdsourcing or conducted research with human subjects, check if you include:
  \begin{enumerate}
    \item The full text of instructions given to participants and screenshots. \textit{Not Applicable.}
    \item Descriptions of potential participant risks, with links to Institutional Review Board (IRB) approvals if applicable. \textit{Not Applicable.}
    \item The estimated hourly wage paid to participants and the total amount spent on participant compensation. \textit{Not Applicable.}
  \end{enumerate}

\end{enumerate}

\newpage
\appendix
\thispagestyle{empty}

\onecolumn
\aistatstitle{
Low Rank Based Subspace Inference for the Laplace Approximation  of Bayesian Neural Networks: 
Supplementary Materials}

\section{EXPERIMENTS}
    \label{sec:AppendixExp}

\subsection{Architectures and Training}
\label{subsec:architectures_and_training}

The code of all experiments was developed in \texttt{PyTorch} \cite{PyTorch2019}. The subset methods were implemented using the \texttt{laplace} library \cite{LaplaceRedux2021}. The regression datasets were obtained from \texttt{OpenML} \cite{OpenML2013, OpenMLPython2019}. 

For the regression problems the expected mean $E_{y\sim p(y|x,\theta)}[y]$ is estimated by multi-layer perceptrons (MLPs) with ReLU activation functions and two hidden layers that include 128 units in each layer. The bias term is used as well. A full batch training is performed in each epoch, i.e. the batch size equals the size of the training set. The input data is normalized with respect to its mean value and its standard deviation. For Naval Propulsion the output data is also normalized in the same manner. Further details of the architecture and training procedure are given in Table \ref{tab:MLPTrainingArchitecture} and its caption.

\begin{table}[H]
\centering
 \caption{The architecture of all networks trained on regression datasets is an MLP with two hidden layers with 128 neurons each. After each hidden layer a ReLU is used. The models are trained on $n_{\mathrm{epoch}}$ epochs with learning rate $\alpha = \alpha_{\text{init}} \frac{b}{256}$ ($b$ is the batch size which here equals the number of training data points). For the fraction of epochs `warm up' the learning rate is linearly increased to $\alpha$ and starting from the fraction of epochs `decay' the learning rate is linearly decreased.}
 \label{tab:MLPTrainingArchitecture}
 \begin{tabular}{l c c c c c c c} 
 \toprule
 dataset &  $\alpha_{\text{init}}$ & $n_{\mathrm{epoch}}$ & warm up/ decay \\ 
 \midrule
  ENB & $0.004$ & 1500 &  (0.1/0.5)  \\
  Red Wine &  $0.0004$ & 300 &  (0.3/0.3) \\
  California & $0.0004$ & 100 &  (0.3/0.5)  \\
  Naval Propulsion & $0.0004$ & 100 & (0.3/0.5)  \\
 \bottomrule
 \end{tabular}
\end{table}

The architecture used for corrupted MNIST \cite{Mu2019mnistc}, trained on MNIST \cite{LeCun1998mnist}, and for FashionMNIST \cite{Xiao2017FashionMNIST} is a small hand-designed convolutional network (CNN) with 2d-convolutions, max-pooling, batch normalization and ReLU activation function. Before the softmax function a linear layer is applied. The exact architecture can be found in the code linked to this article. To train the CNNs, the input data is mapped to the interval $\left[ 0,1\right]$ and then normalized with ``mean'' and ``standard deviation'' $0.5$. Additional details are given in Table \ref{tab:CNNTrainingArchitecture} and its caption.   
\begin{table}[H]
\centering
 \caption{The models are trained on $n_{\mathrm{epoch}}$ epochs with learning rate $\alpha$ and batch size $b=256$. For the fraction of epochs `warm up' the learning rate is linearely increased to $\alpha$ and starting from the fraction of epochs `decay' the learning rate is linearly decreased.}
 \label{tab:CNNTrainingArchitecture}
 \begin{tabular}{l c c c c c c c} 
 \toprule
 dataset & $\alpha$ & $n_{\mathrm{epoch}}$ & warm up/ decay \\
 \midrule
  MNIST &  $ 0.004 $ & 20 & (0.1/0.3) \\
 FashionMNIST & $0.002$ & 40 & (0.1/0.5)  \\
 \bottomrule
 \end{tabular}
\end{table}

CIFAR10 \cite{Krizhevsky2009cifar} and ImageNet10 \cite{ImageNet2009} are classified by ResNet architectures \cite{He2015resnet}. CIFAR10 is trained from scratch with ResNet9, but for ImageNet10 the pretrained ResNet18 from \texttt{Pytorch} with weights \texttt{IMAGENET1K\_V1} is used as initialization, where the last layer is replaced by a linear layer with 10 classes. During training the images are normalized with respect to their channelwise pixel mean and pixel standard deviation. In addition, random flips are applied on both datasets. For CIFAR10 greyscale and random crops are used, too. More information is provided in Table \ref{tab:ResNetTrainingArchitecture} and its caption.
\begin{table}[H]
\centering
 \caption{The models are trained on $n_{\mathrm{epoch}}$ epochs with learning rate $\alpha$ and batch size $b=256$. For the fraction of epochs `warm up' the learning rate is linearly increased to $\alpha$ and starting from the fraction of epochs `decay' the learning rate is linearly decreased.}
 \label{tab:ResNetTrainingArchitecture}
 \begin{tabular}{l c c c c c c c} 
  \toprule
 dataset & $\alpha$ & $n_{\mathrm{epoch}}$ &  warm up/ decay \\ 
 \midrule
 CIFAR10 &  $\ 0.004$ & 100 & (0.1/0.7) \\
 ImageNet10 & 0.0004 & 10 & (0.5/0.5)  \\
 \bottomrule
 \end{tabular}
\end{table}

To evaluate the quality of the dimensional reduction, the size of the different models that are used for predictions are required. Table \ref{tab:NumParamsModel} lists the number of model parameters. The number of model parameters of the MLP and CNN has been chosen large enough such that the prediction performance is satisfying, but is also limited to be able to compute $P_{\mathrm{lowrankopt-GGN}}$ and the full LA.

\begin{table}[H]
\centering
\caption{Number of trainable parameters $p$ for each model trained on the corresponding dataset, and relative values for $s=100$ and $s=500$.}
\label{tab:NumParamsModel}
\begin{tabular}{l c c c c} 
\toprule
dataset & model & $p$ & $s=100$ & $s=500$ \\
\midrule
ENB & MLP & 17,922 & 0.56\% & 2.79\% \\
Red Wine & MLP & 18,177 & 0.55\% & 2.75\% \\
California & MLP & 17,793 & 0.56\% & 2.81\% \\
Naval & MLP & 18,690 & 0.53\% & 2.68\% \\
MNIST & CNN & 12,458 & 0.80\% & 4.01\% \\
FashionMNIST & CNN & 12,458 & 0.80\% & 4.01\% \\
CIFAR10 & ResNet9 & 668,234 & 0.01\% & 0.07\% \\
ImageNet10 & ResNet18 & 11,181,642 & $<0.01\%$ & -- \\
\bottomrule
\end{tabular}
\end{table}

\subsection{ImageNet10 Classes}

ImageNet10 is a proper subset of ImageNet \cite{ImageNet2009}. The selection of classes used for ImageNet10 is given in Table \ref{tab:ImagNet10}.

\begin{table}[H]
\centering
 \caption{These ten labels are selected from ImageNet to construct ImageNet10.}
 \label{tab:ImagNet10}
 \begin{tabular}{l c c c c c c c} 
 \toprule
 label & motifs \\ [0.5ex] 
 \midrule
 \multirow{2}{*}{n01968897} & pearly nautilus, nautilus, \\
           & chambered nautilus \\
 \multirow{2}{*}{n01770081} & harvestman, daddy longlegs, \\
           & Phalangium opilio \\
 \multirow{2}{*}{n01496331} & crampfish, numbfish,\\
           &  torpedo, electric ray \\
 \multirow{2}{*}{n01537544} & indigo bunting, indigo finch, \\
           & indigo bird, Passerina cyanea \\ 
 n01818515 & macaw \\
 n02011460 & bittern \\
 n01847000 & drake \\ 
 n01687978 & agama \\
 n01740131 & night snake, Hypsiglena torquata \\ 
 n01491361 & tiger shark, Galeocerdo cuvieri \\
 \bottomrule
 \end{tabular}
\end{table}

\subsection{Extended Discussion on the Computational Bottlenecks of our Method}
\label{subsec:size_of_training_data_subset}
For the low rank methods we construct $P$ as described in Section \ref{sec:thm_in_practice} as
\begin{align}
    \label{eq:P_from_training_data}
    P = \Psi_{\mathrm{approx}} J_{X'}^\intercal U_s\,.
\end{align}
All three objects in \eqref{eq:P_from_training_data}, $\Psi_{\mathrm{approx}}$, $J_{X'}$ and $U_s$, are constructed from the training data. While we can take the full training data for the construction of $\Psi_{\mathrm{approx}}$, both, $J_{X'}$ and $U_s$, are constructed from a subset $X'$ of size $n$ of the training data. Ideally, we would of course like to take $X'$ to be full training data. However, doing so presents us with two difficulties:
\begin{enumerate}
    \item The object $U_s$ needs to be computable. 
    \item The computation of the product $J_{X'}^\intercal U_s$ needs to be feasible.
\end{enumerate}
Obstacle 2 is rather straightforward to circumvent as we can compute the matrix product via mini-batches from the training data. It turns out that Obstacle 1 sets the actual limit on the subset of training data as we compute $U_s$ via an SVD of the object $J_{X'} \Psi_{\mathrm{approx}} J_{X'}^\intercal\in \RR^{nC\times nC}$. For Red Wine and Naval we pick $n=1000$. For ENB, the training set has only $514$ data points which is why the entire training dataset is considered. For California we can analyze the subspace models until $s=5000$ as the Jacobian of the model has full rank. To allow for this analysis we choose $n=5000$. For the classification problems, i.e. MNIST, FashionMNIST, CIFAR10 and ImageNet10, we pick $n=100$ so that we have $nC=1000$ for these datasets. This choice allows for a substantially faster computation of $P$. Our methods demand the explicit storage of $P$, which limits the the maximum value of $s$ from a computational perspective.

For our experiments, we choose to use the same size $n$ for $X'$ (from the training data) and $X$ (from the test data). This is done only for convenience and is by no means a necessary assumption for the workflow from Section \ref{sec:thm_in_practice}.

The explicit computation and storage of the projector is a limitation of methods based on linear affine transformation. For large neural networks, this is a bottleneck, because the memory complexity grows with at least $\mathcal{O}\left(s(p+s)\right)$. The subset methods of \cite{Daxberger2021}, by contrast, use heuristics to avoid the explicit construction of $P$. For example, they can directly construct a subspace by selecting the parameters with the highest weights. In this case, the memory complexity is only $\mathcal{O}\left(s^2\right)$. Thus, in such cases, subset methods can be applied, whereas lowrank methods cannot. The order of memory complexity $s(p+s)$ for all datasets used in this work together with typical order of $s$ used here is listed in Table \ref{tab:MemoryLoad}.

\begin{table}[H]
\centering
\caption{Order of the memory load $s(p+s)$ required for the application of lowrank based subspace methods for all datasets and typical orders of $s$ used in this article.}
\label{tab:MemoryLoad}
\begin{small}
\begin{tabular}{l l c c c c c c} 
\toprule
dataset & $p$ & \multicolumn{6}{c}{$s(s+p)$} \\
\cline{3-8} & &
 $s=100$ & $s=200$ & $s=300$ & $s=400$ & $s=500$ & $s=1000$ \\
 \midrule
 ENB & 17922 & $1.802 \cdot 10^{6}$ & $3.624 \cdot 10^{6}$ & $5.467 \cdot 10^{6}$ & $7.329 \cdot 10^{6}$ & $9.211 \cdot 10^{6}$ & $1.892 \cdot 10^{7}$ \\
 Red Wine & 18,177 & $1.828 \cdot 10^{6}$ & $3.675 \cdot 10^{6}$ & $5.543 \cdot 10^{6}$ & $7.431 \cdot 10^{6}$ & $9.338 \cdot 10^{6}$ & $1.918 \cdot 10^{7}$ \\
California & 17,793 & $1.789 \cdot 10^{6}$ & $3.599 \cdot 10^{6}$ & $5.428 \cdot 10^{6}$ & $7.277 \cdot 10^{6}$ & $9.146 \cdot 10^{6}$ & $1.879 \cdot 10^{7}$ \\
Naval & 18,690 & $1.879 \cdot 10^{6}$ & $3.778 \cdot 10^{6}$ & $5.697 \cdot 10^{6}$ & $7.636 \cdot 10^{6}$ & $9.595 \cdot 10^{6}$ & $1.969 \cdot 10^{7}$ \\
MNIST & 12,458 & $1.256 \cdot 10^{6}$ & $2.532 \cdot 10^{6}$ & $3.827 \cdot 10^{6}$ & $5.143 \cdot 10^{6}$ & $6.479 \cdot 10^{6}$ & $1.346 \cdot 10^{7}$ \\
FashionMNIST & 12,458 & $1.256 \cdot 10^{6}$ & $2.532 \cdot 10^{6}$ & $3.827 \cdot 10^{6}$ & $5.143 \cdot 10^{6}$ & $6.479 \cdot 10^{6}$ & $1.346 \cdot 10^{7}$ \\
CIFAR10 & 668,234 & $6.683 \cdot 10^{7}$ & $1.337 \cdot 10^{8}$ & $2.006 \cdot 10^{8}$ & $2.675 \cdot 10^{8}$ & $3.344 \cdot 10^{8}$ & $6.692 \cdot 10^{8}$ \\
ImageNet10 & 11,181,642  & $1.118 \cdot 10^{9}$ & $2.236 \cdot 10^{9}$ & $3.355 \cdot 10^{9}$ & $4.473 \cdot 10^{9}$ & $5.591 \cdot 10^{9}$ & $1.118 \cdot 10^{10}$ \\
\bottomrule
\end{tabular}
\end{small}
\end{table}


\subsection{Prior Distribution}

For all problems the prior distribution of the full parameter $\theta\in \RR^p$ is chosen to be a centred Gaussian prior $p(\theta)=\mathcal{N}(\theta|0,\lambda^{-1})$. 
For each dataset one fixed prior precision is used. The precisions were obtained by applying the marginal likelihood optimization for the KFAC LA from the \texttt{laplace} software library \cite{LaplaceRedux2021}. For all the datasets but ImageNet10 we averaged the prior precision over five seeds. Due to computational resources and the fact that the prior precision does not vary much across all sets, we estimated the prior precision for ImageNet10 for seed 1 only. The corresponding values are listed in Table \ref{tab:prior_precision}.

\begin{table}[H]
    \centering
    \caption{Used prior precision for the experiments in this article.}
    \label{tab:prior_precision}
    \begin{tabular}{l c}
        \toprule
         dataset & prior precision $\lambda$  \\
         \midrule
         ENB & 2.0 \\
         Red Wine & 6.5 \\
         California & 8.8 \\
         Naval Propulsion & 4.3 \\
          MNIST & 18.9 \\
         FashionMNIST & 32.7 \\
         CIFAR10 & 68.1 \\
         ImageNet10 & 13.0 \\
         \bottomrule
    \end{tabular}
\end{table}

\subsection{Computation Costs}
\label{Sec:Computational_Costs}
The main factor for the computational costs of all subspace methods depends on the approximation scheme of the LA. In general, to compute $P$ for the subset or low rank methods needs similar time. The only exception is $P_{\mathrm{subset-magnitude}}$ because it selects the weights depending on their magnitude which is easy to extract for trained NNs. The time to compute the projector using an Nvidia A100 graphics card is given in Table \ref{tab:ComputeTime} for ImageNet. For the other experiments the computational time is much less, e.g. FashionMNIST needs only seconds for any projector.

\begin{table}[H]
    \centering
     \caption{Computation costs of the projector for ImageNet10. The time to compute $P$ for the lowrank methods (lr) or the subset (sb) methods depends on the low rank approximation or selection methods.}
 \label{tab:ComputeTime}
    \begin{tabular}{ccccc}
    \toprule
    lowrank-KFAC & lowrank-Diagonal & subset-Diagonal & subset-Magnitude & subset-SWAG \\
    \midrule
    $17$min & $28$min & $28$min & $<1$min &  $20$min \\
    \bottomrule
    \end{tabular}
\end{table}

\section{EXISTENCE OF AN OPTIMAL SUBSPACE MODEL FOR THE
LAPLACE APPROXIMATION}
    \label{sec:AppendixProofSubmodel}

In the following we prove Theorem \ref{thm:OptSubmodel} which we restate here.

\begin{theorem*}
Consider the problem \eqref{eq:minimization_problem} with $s\leq s_{\mathrm{max}}=\min(nC,p)$.
Suppose that $J_X\in \RR^{nC\times p}$ has full rank. For any invertible $Q\in \RR^{s\times s}$ the matrix
\begin{align}
    \label{eq:Appendix_optimal_P}
    P^* = \Psi J_X^\intercal U_s Q
\end{align}
solves \eqref{eq:minimization_problem}.
 For any such $P^*$ we have
\begin{align}
    \label{eq:Appendix_low_rank_laplace_optimal_Sigma}
    \Sigma_{P^*,X} = U_s \Lambda_s U_s^\intercal.
\end{align}

If the diagonal elements of $\Lambda$ satisfy $\sigma_s > \sigma_{s+1}$ then the solution of the minimization problem \eqref{eq:minimization_problem} is unique up to reparametrization: any minimizer of \eqref{eq:minimization_problem} is of the form \eqref{eq:Appendix_optimal_P}.
\end{theorem*}

\begin{proof}

    For the sake of this proof it will be convenient to work with a more general family of $P\in \mathbb{R}^{p \times s}$ than \eqref{eq:Appendix_optimal_P}:
    \begin{align}
        \label{eq:P_plus_Kernel}
        P = \left( \Psi J_X^\intercal U_s + K \right) Q\,.
    \end{align}
    Here $Q\in \mathbb{R}^{s \times s}$ is an arbitrary non-singular matrix as in the statement of the theorem and $K\in \mathbb{R}^{p\times s}$ is an arbitrary matrix that satisfies $J_X K=0$. 
    The proof splits in three parts:
    \begin{enumerate}
        \item We show that for any $P$ of the form \eqref{eq:P_plus_Kernel} we have
        \begin{align}
            \label{eq:Sigma_P_with_kernel}
            \Sigma_{P,X} = U_s \Lambda_s U_s^\intercal - U_s K^\intercal A K  U_s^\intercal
        \end{align}
        where $A=(\Psi + K \Lambda_s^{-1} K^\intercal)^{-1}\in \mathbb{R}^{s \times s}$ is a positive definite matrix. For the special case $K=0$, which corresponds to \eqref{eq:Appendix_optimal_P}, we obtain in particular the identity \eqref{eq:Appendix_low_rank_laplace_optimal_Sigma},
        which minimizes indeed \eqref{eq:minimization_problem} due to the Eckart-Young-Mirsky theorem.
        \item Assuming $\sigma_s > \sigma_{s+1}$ we show that any $P\in \mathbb{R}^{p\times s}$ that solves \eqref{eq:minimization_problem} must be of the form \eqref{eq:P_plus_Kernel}. 
        \item Still assuming $\sigma_{s} > \sigma_{s+1}$ we show that for $P$ of form \eqref{eq:P_plus_Kernel} $\Sigma_{P,X}$ can only minimize \eqref{eq:minimization_problem} if $K=0$. Together with part 1-2 the claim follows.
    \end{enumerate}

    For the first part of the proof let us now assume that $P$ is as in \eqref{eq:P_plus_Kernel}. We then have
    \begin{align}
        \label{eq:optimal_Sigma_P_outer_part}
        J_X P = \overbrace{J_X \Psi J_X^\intercal}^{=\Sigma_X} U_s Q + \overbrace{J_X K }^{=0}Q = \Sigma_X U_s Q = U_s\Lambda_s Q
    \end{align}
    and further
    \begin{align*}
        P^\intercal \Psi^{-1} P & = Q^\intercal ( U_s^\intercal J_X \Psi+  K^\intercal) \Psi^{-1} (\Psi J_X^\intercal U_s  + K)Q  \\
        & = Q^\intercal (U_s^\intercal J_X \Psi J_X^\intercal U_s  + K^\intercal \Psi^{-1} K)Q
    \end{align*}
    where we used in the last equation that $J_X K = 0$ and $K^\intercal J_X^\intercal = 0$. Using $J_X \Psi J_X^\intercal = \Sigma_X$ and $U_s^\intercal \Sigma_X U_s = \Lambda_s$ this can be further simplified to 
    \begin{align}
        \label{eq:optimal_Sigma_P_inner_part}
        P^\intercal \Psi^{-1} P = Q^\intercal (\Lambda_s + K^\intercal \Psi^{-1} K) Q\,.
    \end{align}
   Applying the Woodbury matrix identity (Wb) we obtain with \eqref{eq:optimal_Sigma_P_outer_part} and \eqref{eq:optimal_Sigma_P_inner_part} 
    \begin{align*}
        \Sigma_{P,X} &= J_X P (P^\intercal \Psi^{-1} P)^{-1} P^\intercal J_X^\intercal \\
        &= U_s\Lambda_s Q \left(Q^\intercal (\Lambda_s + K^\intercal \Psi^{-1} K) Q\right)^{-1} Q^\intercal \Lambda_s U_s^\intercal
        \\
        &= U_s \Lambda_s (\Lambda_s + K^\intercal \Psi^{-1} K)^{-1} \Lambda_s U_s^\intercal \\
        &\overset{\mbox{\tiny Wb}}{=}U_s \Lambda_s (\Lambda_s^{-1} - \Lambda_{s}^{-1} K^\intercal (\Psi + K \Lambda_{s}^{-1} K^\intercal)^{-1} K \Lambda_{s}^{-1}) \Lambda_s U_s^\intercal \\
        & = U_s \Lambda_s U_s - U_s K^\intercal (\Psi + K \Lambda_s^{-1} K^\intercal)^{-1} K U_s^\intercal \\
        &=   U_s \Lambda_s U_s^\intercal - U_s K^\intercal A K U_s^\intercal \,,
    \end{align*}
    which concludes part 1 of the proof.
    
    For part 2-3 of the proof we use that, according to the Eckart-Young-Mirsky theorem, the minimizer $U_s \Lambda_s U_s^\intercal$ of \eqref{eq:minimization_problem} is \emph{unique} if $\sigma_s > \sigma_{s+1}$.

    Considering part 2 of the proof, this uniqueness implies that for any $P$ that solves \eqref{eq:minimization_problem} the projected covariance matrix has to be of the form $\Sigma_{P,X}=U_s \Lambda_s U_s^\intercal$ and in particular
    \begin{align*}
       (\Sigma_{P,X} - \Sigma_X) U_s = U_s \Lambda_s - U_s \Lambda_s = 0 \,.
    \end{align*}
    Plugging the definition of $\Sigma_{P,X}$ and $\Sigma_X$ in this identity we obtain
    \begin{align*}
        J_X \left(P(P^\intercal \Psi^{-1} P)^{-1} P^\intercal J_X^\intercal U_s - \Psi J_X^\intercal U_s\right) = 0\,,
    \end{align*}
    so that we know that the matrix
    \begin{align*}
        K =P(P^\intercal \Psi^{-1} P)^{-1} P^\intercal J_X^\intercal U_s - \Psi J_X^\intercal U_s =: P B - \Psi J_X^\intercal U_s
    \end{align*}
    satisfies $J_X K=0$. Note that $B$ is non-singular because $\det(U_s^\intercal J_X P) \det(B) = \det(U_s^\intercal J_X P B) = \det(U_s^\intercal \Sigma_{P,X} U_s) = \det(\Lambda_s) \neq 0$. The choice $Q:=B^{-1}$ therefore concludes part 2 of the proof.

    We are left with part 3 of the proof. Assume therefore that $P$ is of the form \eqref{eq:P_plus_Kernel} \emph{and} that $\Sigma_{P,X}$ minimizes \eqref{eq:minimization_problem}. We have by part 1 of the proof
    \begin{align*}
        \Sigma_{P,X} = U_s \Lambda_s U_s^\intercal - U_s K^\intercal A K U_s^\intercal\,,
    \end{align*}
    where $A$ is positive definite. Now, under the assumption that $\sigma_s > \sigma_{s+1}$ we can, once more, use the uniqueness part of the Eckart-Young-Mirsky theorem to conclude that $\Sigma_{P,X} =U_s \Lambda_s U_s^\intercal$ which implies that
    \begin{align}
         U_s K^\intercal A K U_s^\intercal = 0
    \end{align}
    Multiplying both sides with $U_s$ and using that $U_s^\intercal U_s = \mathbb{1}_{s}$ we obtain 
    \begin{align*}
         U_s^\intercal U_s K^\intercal A K U_s^\intercal U_s = K^\intercal A K = 0
    \end{align*}
    which implies that $K=0$ due to the positive definiteness of $A$.
\end{proof}

Note that all we used in the proof of this Theorem were the identities \eqref{eq:Sigma_X} and \eqref{eq:Sigma_p} so that the statement of the theorem does not really require $\Psi$ to be derived via a Laplace approximation. 
\subsection{Additional Remarks to Theorem \ref{thm:OptSubmodel}}
\label{subsec:additional_remark_to_theorem}
\paragraph{Optimality of $P^\ast$ According to the Trace Ordering.}
Note, that Lemma \ref{lem:TraceOrder} below implies that the ordered eigenvalues (in decreasing order) of any $\Sigma_{P,X}$ satisfy $\lambda_i(\Sigma_{P,X}) \leq \lambda_i(\Sigma_X)$ as follows from the min-max theorem \citep[Theorem 4.2.6]{horn2012matrix}. Since any $\Sigma_{P,X}$ has rank $s$ and as $\Sigma_{P^\ast,X}$ shares the $s$ largest eigenvalues with $\Sigma_X$ we obtain 
\begin{align}
\label{eq:eigenvalue_ordering}
\lambda_i(\Sigma_{P,X}) \leq \lambda_i(\Sigma_{P^\ast,X}) = \lambda_i(\Sigma_X) 1_{i\leq s}
\end{align}
so that $\Sigma_{P^\ast,X}$ is actually optimal w.r.t. the trace criterion \eqref{eq:TraceOrderMain}:
\begin{align*}
    \Tr\Sigma_{P,X}  \leq \Tr \Sigma_{P^\ast,X} \,.
\end{align*}

\paragraph{Implications on Coverage (for Regression Problems). } Theorem \ref{thm:OptSubmodel} also implies, in a regression setting, that among any ellipsoidal region of the form $E_{P^\ast}(c)=\{Y: \,(Y-f_{\hat{\theta}}(X))^\intercal (\Sigma_{P,X}+\sigma \mathbb{1}_{nC})^{-1} (Y-f_{\hat{\theta}}(X)) \leq c\}$ the coverage of samples of the full LA is maximal for the choice $P=P^\ast$ for $c>0$ large enough.

In the following, we show that this is true. Using $Z\sim \mathcal{N}(0,\mathbb{1}_{nC})$ and the invariance of the standard multivariate normal with respect to orthogonal transformations, we obtain via diagonalization of $\Sigma_P$ and $\Sigma_{P,X}$
\begin{align*}
    & P_{Y\sim \mathcal{N}(f_{\hat{\theta}}(X),\Sigma_X+\sigma^2 \mathbb{1}_{nC})} (Y \in E_P(c)) \\ 
    &= P(Z^\intercal (\Sigma_{X}+\sigma^2 \mathbb{1}_{nC})^{1/2} ( \Sigma_{P,X}+ \sigma^2 \mathbb{1}_{nC} )^{-1} ( \Sigma_{X} + \sigma^2 \mathbb{1}_{nC})^{1/2} Z \leq c) \\ 
    &= P(Z^\intercal ( \Lambda+ \sigma^2 \mathbb{1}_{nC})^{1/2} U ( \Lambda(\Sigma_{P,X})+ \sigma^2 \mathbb{1}_{nC})^{-1} U^\intercal (\Lambda +\sigma^2 \mathbb{1}_{nC}  )^{1/2} Z \leq c) \\ 
    &\leq P(Z^\intercal A_U Z \leq c)\,,
\end{align*} 
where $U$ is an orthogonal matrix (depending on the choice of $P$), 
$\Lambda(\Sigma_{P,X})$ and $\Lambda=\Lambda(\Sigma_X)$ denote the corresponding diagonal matrices with decreasing eigenvalues on the diagonal and where we wrote
$A_U = (\Lambda+\sigma^2 \mathbb{1}_{nC} )^{1/2} U (\Lambda(\Sigma_{P^\ast,X})+\sigma^2 \mathbb{1}_{nC})^{-1} U^\intercal (\Lambda+\sigma^2\mathbb{1}_{nC})^{1/2}$. The inequality above follows from \eqref{eq:eigenvalue_ordering}. For $P=P^\ast$ one easily gets
using $\Sigma_{P^\ast,X}= U_s \Lambda_s U_s^\intercal$ (Theorem \ref{thm:OptSubmodel})
\begin{align*}
P_{Y\sim \mathcal{N}(f_{\hat{\theta}}(X),\sigma^2\Sigma_X)} (Y\in E_{P^\ast}(c)) = P(Z^\intercal A_{\mathbb{1}_{nC}} Z \leq c) \,.
\end{align*}
Once we can show 
\begin{align}
\label{eq:A_ordering}
 P(Z^\intercal A_U Z \leq c) \leq P(Z^\intercal A_{\mathbb{1}_{nC}} Z \leq c)
\end{align}
the proof of the claim above is therefore finished. The ordering \eqref{eq:A_ordering} follows from \citep{szekely2003extremal} once we know that for any $1\leq k \leq n$
\begin{align*}
    \sum_{i=1}^k \lambda_i(A_{\mathbb{1}_{nC}}) \leq \sum_{i=1}^k \lambda_i(A_U) \,.
\end{align*}
Due to \citep[Theorem 20.A.2]{marshall1979inequalities} we have, writing $\lambda_i$ for the diagonal elements of $\Lambda =\Lambda(\Sigma_X)$,
\begin{align}
    \sum_{i=1}^k \lambda_i (A_U) & = \max_{V\in \RR^{k\times nC}: V V^\intercal =\mathbb{1}_k} \Tr (V A_U V^\intercal) \notag  \\ &\geq \Tr \mathbb{1}_{k\times nC}  A(U) \mathbb{1}_{k\times nC}^\intercal 
     = \sum_{i,j=1}^k ( \lambda_i+ \sigma^2 ) |U_{ij}|^2 (\lambda_j1_{j\leq s} + \sigma^2)  
    = \sum_{m} \theta_m  \sum_{i,j=1}^k( \lambda_i +\sigma^2  ) P_m (\lambda_j 1_{j\leq s} +\sigma^2)^{-1} 
    \label{eq:Birkhoff}
\end{align}
where $\mathbb{1}_{k\times nC}$ denotes the projection on the first $k$ components and where we used in the second line Birkhoff's theorem \citep[Theorem 8.7.2]{horn2012matrix} to replace the doubly stochastic matrix $(|U_{ij}|^2)_{ij}$ by $\sum_{m} \theta_m P_m$ with $\sum_{m}\theta_m=1,\theta_m\geq 0$ and permutation matrices $P_m$. By the rearrangement inequality \citep[Theorem 368]{hardy1988inequalities} the object \eqref{eq:Birkhoff} becomes minimal when all $P_m$ equal $\mathbb{1}_{nC}$ and in this case we just obtain $\sum_{i=1}^k \lambda_{i}(A_{\mathbb{1}_{nC}})$.

\section{TRACE CRITERION}
    \label{sec:AppendixTraceCriterion}
Ideally we would like to choose a $P$ such that the predictive distribution of the full \eqref{eq:PredDistrGauss}, \eqref{eq:PredDistrCat} and subspace model \eqref{eq:LaplaceSubmodel} are as close as possible. In Section \ref{sec:LA_for_subspace_models} we introduced the KL divergence between these distributions.
However, in practice the full LA is unknown, which makes this metric infeasible. 

As an alternative criterion we propose for our purposes to use the trace $\Tr \Sigma_{P,X}$ instead. This criterion is feasible to compute, aligns well with the KL divergence as we show empirically in Section \ref{sec:Experiment} and can be motivated by the following lemma: 

\begin{lemma}
    \label{lem:TraceOrder}
For any $P \in \mathbb{R}^{p \times s}$ we have
\begin{equation}
    \label{eq:SigmaLoewner}
    \Sigma_{P,X} \preccurlyeq \Sigma_X
\end{equation}
in the Loewner ordering, i.e. $\Sigma_X - \Sigma_{P,X}$ is positive semi-definite. In particular we have 
\begin{equation}
    \label{eq:TraceOrder}
    \Tr \Sigma_{P,X} \leq \Tr \Sigma_X \,.
\end{equation}
\end{lemma}
\begin{proof}
Due to the identities \eqref{eq:Sigma_X} and \eqref{eq:Sigma_p} it suffices to show that $ P \left(P^\intercal \Psi^{-1} P\right)^{-1}P^\intercal \preccurlyeq \Psi$, i.e., that the matrix 
\[
\Psi - P(P^\intercal\Psi^{-1} P)^{-1} P^\intercal= \Psi^{1/2} (\mathbb{1}-B) \Psi^{1/2}
\]
is positive semi-definite, where we introduced $B=
W \left(W^\intercal W\right)^{-1} W^\intercal$ with $W=\Psi^{-1/2}P$. It's easy to check that $B$ is a projection ($B^2=B$ and $B^\intercal=B$) which thus has only eigenvalues contained in $\{0,1\}$. From this it follows that $\mathbb{1}-B$ and $\Psi^{1/2} (\mathbb{1}-B) \Psi^{1/2}$ are positive semi-definite and thus \eqref{eq:SigmaLoewner}. 

From \eqref{eq:SigmaLoewner} we obtain $\Tr(\Sigma_X - \Sigma_{P,X})=\Tr \Sigma_X - \Tr \Sigma_{P,X} \geq 0$ from which \eqref{eq:TraceOrder} follows.
\end{proof}

The relation \eqref{eq:TraceOrder} shows that $\Tr (\Sigma_X  - \Sigma_{P,X})\geq 0$ is a non-negative quantity that quantifies the closeness between $\Sigma_X$ and $\Sigma_{P,X}$.
Since $\Sigma_X$ does not depend on $P$ we can judge whether for two $P_1,P_2$ we have $\Tr(\Sigma_X - \Sigma_{P_1,X}) \geq \Tr(\Sigma_X - \Sigma_{P_2,X})$ by simply comparing whether $\Tr\Sigma_{P_1,X} \geq \Tr \Sigma_{P_2,X}$. In other words, we can take $\Tr {\Sigma_{P,X}}$ to rank the quality of different $P$. Relation \eqref{eq:TraceOrder} ensures that there is an upper bound for this quantity. Recall that the trace is the sum of all eigenvalues of a matrix. If the trace of one approximation is greater than another one, it means that this affine subspace covers an eigenspace of greater eigenvalues. 

\section{ADDITIONAL RESULTS ON UNCERTAINTY QUANTIFICATION}
\label{sec:addtional_results_on_UQ}

Table \ref{tab:eval_UQ_extended} extends Table \ref{tab:eval_UQ} by including a broader range of subspace dimensions $s$ and reporting the standard error.
For the Brier score, we use in this work the convention $\frac{1}{N} \sum_{i=1}^N \sum_{c=1}^C (1_{y_i=c} - \hat{p}_{ic})^2$ that averages over datapoints but sums over classes.
Figure \ref{fig:all_nll_plots} summarizes all results of the NLL metric for all methods and datasets studied in this article with the same colour coding as in Figure \ref{fig:plots_with_feasible_optimal_solution}.

\begin{table}[h]
\begin{scriptsize}
\begin{tabular}{lllcccccc}
\toprule
metric & dataset & $s$ & lowrank-D. & lowrank-KFAC & subset-D. & subset-M. & subset-SWAG & $\hat{\theta}$ \\
\midrule
\multirow{15}{*}{NLL} & \multirow{5}{*}{FashionMNIST} & 120 & 0.289 (0.004) & \textbf{0.248} (0.002) & 0.292 (0.004) & 0.292 (0.004) & 0.292 (0.004) & 0.293 (0.004) \\
 &  & 230 & 0.286 (0.004) & \textbf{0.248} (0.002) & 0.291 (0.004) & 0.290 (0.004) & 0.290 (0.004) & 0.293 (0.004) \\
 &  & 450 & 0.282 (0.004) & \textbf{0.249} (0.002) & 0.290 (0.004) & 0.288 (0.004) & 0.288 (0.004) & 0.293 (0.004) \\
 &  & 670 & 0.278 (0.004) & \textbf{0.249} (0.002) & 0.288 (0.004) & 0.286 (0.004) & 0.286 (0.004) & 0.293 (0.004) \\
 &  & 1000 & 0.273 (0.003) & \textbf{0.250} (0.002) & 0.286 (0.004) & 0.283 (0.004) & 0.283 (0.004) & 0.293 (0.004) \\
 & \multirow{5}{*}{CIFAR10} & 120 & 0.287 (0.001) & \textbf{0.263} (0.001) & 0.289 (0.001) & 0.287 (0.001) & 0.287 (0.001) & 0.289 (0.001) \\
 &  & 230 & 0.285 (0.001) & \textbf{0.261} (0.001) & 0.289 (0.001) & 0.285 (0.001) & 0.285 (0.001) & 0.289 (0.001) \\
 &  & 450 & 0.281 (0.001) & \textbf{0.259} (0.001) & 0.289 (0.001) & 0.282 (0.001) & 0.283 (0.001) & 0.289 (0.001) \\
 &  & 670 & 0.277 (0.001) & \textbf{0.258} (0.001) & 0.289 (0.001) & 0.280 (0.001) & 0.280 (0.001) & 0.289 (0.001) \\
 &  & 1000 & 0.273 (0.001) & \textbf{0.256} (0.001) & 0.289 (0.001) & 0.277 (0.001) & 0.277 (0.001) & 0.289 (0.001) \\
 & \multirow{5}{*}{ImageNet10} & 10 & 0.276 (0.016) & \textbf{0.268} (0.012) & 0.277 (0.016) & 0.276 (0.016) & 0.277 (0.016) & 0.277 (0.016) \\
 &  & 20 & 0.275 (0.016) & \textbf{0.266} (0.011) & 0.277 (0.016) & 0.276 (0.016) & 0.277 (0.016) & 0.277 (0.016) \\
 &  & 50 & 0.274 (0.015) & \textbf{0.270} (0.010) & 0.277 (0.016) & 0.276 (0.016) & 0.277 (0.016) & 0.277 (0.016) \\
 &  & 70 & \textbf{0.272} (0.015) & 0.273 (0.010) & 0.277 (0.016) & 0.276 (0.016) & 0.276 (0.016) & 0.277 (0.016) \\
 &  & 100 & \textbf{0.271} (0.015) & 0.277 (0.009) & 0.277 (0.016) & 0.276 (0.016) & 0.276 (0.016) & 0.277 (0.016) \\
\midrule
\multirow{15}{*}{ECE} & \multirow{5}{*}{FashionMNIST} & 120 & 0.039 (0.001) & \textbf{0.010} (0.001) & 0.040 (0.001) & 0.040 (0.001) & 0.040 (0.001) & 0.041 (0.001) \\
 &  & 230 & 0.039 (0.001) & \textbf{0.012} (0.001) & 0.040 (0.001) & 0.040 (0.001) & 0.040 (0.001) & 0.041 (0.001) \\
 &  & 450 & 0.037 (0.001) & \textbf{0.013} (0.001) & 0.040 (0.001) & 0.039 (0.001) & 0.039 (0.001) & 0.041 (0.001) \\
 &  & 670 & 0.035 (0.001) & \textbf{0.015} (0.001) & 0.039 (0.001) & 0.039 (0.001) & 0.039 (0.001) & 0.041 (0.001) \\
 &  & 1000 & 0.033 (0.001) & \textbf{0.016} (0.001) & 0.038 (0.001) & 0.038 (0.001) & 0.038 (0.001) & 0.041 (0.001) \\
 & \multirow{5}{*}{CIFAR10} & 120 & 0.037 (0.001) & \textbf{0.029} (0.000) & 0.038 (0.001) & 0.037 (0.001) & 0.037 (0.001) & 0.038 (0.001) \\
 &  & 230 & 0.037 (0.001) & \textbf{0.029} (0.000) & 0.038 (0.001) & 0.037 (0.001) & 0.037 (0.001) & 0.038 (0.001) \\
 &  & 450 & 0.035 (0.001) & \textbf{0.027} (0.000) & 0.038 (0.001) & 0.036 (0.001) & 0.036 (0.001) & 0.038 (0.001) \\
 &  & 670 & 0.034 (0.001) & \textbf{0.026} (0.000) & 0.038 (0.001) & 0.035 (0.001) & 0.035 (0.001) & 0.038 (0.001) \\
 &  & 1000 & 0.033 (0.001) & \textbf{0.025} (0.001) & 0.038 (0.001) & 0.035 (0.001) & 0.035 (0.001) & 0.038 (0.001) \\
 & \multirow{5}{*}{ImageNet10} & 10 & 0.039 (0.002) & \textbf{0.036} (0.003) & 0.040 (0.003) & 0.040 (0.003) & 0.040 (0.003) & 0.040 (0.003) \\
 &  & 20 & 0.040 (0.003) & \textbf{0.036} (0.002) & 0.040 (0.003) & 0.041 (0.002) & 0.039 (0.003) & 0.040 (0.003) \\
 &  & 50 & \textbf{0.039} (0.003) & 0.041 (0.004) & 0.040 (0.003) & 0.040 (0.002) & 0.040 (0.003) & 0.040 (0.003) \\
 &  & 70 & \textbf{0.038} (0.002) & 0.044 (0.004) & 0.040 (0.003) & 0.039 (0.002) & 0.040 (0.003) & 0.040 (0.003) \\
 &  & 100 & \textbf{0.038} (0.001) & 0.048 (0.006) & 0.040 (0.003) & 0.039 (0.002) & 0.040 (0.003) & 0.040 (0.003) \\
\midrule
\multirow{15}{*}{Brier} & \multirow{5}{*}{FashionMNIST} & 120 & 0.135 (0.001) & \textbf{0.128} (0.001) & 0.135 (0.001) & 0.135 (0.001) & 0.135 (0.001) & 0.135 (0.001) \\
 &  & 230 & 0.134 (0.001) & \textbf{0.129} (0.001) & 0.135 (0.001) & 0.135 (0.001) & 0.135 (0.001) & 0.135 (0.001) \\
 &  & 450 & 0.134 (0.001) & \textbf{0.129} (0.001) & 0.135 (0.001) & 0.135 (0.001) & 0.135 (0.001) & 0.135 (0.001) \\
 &  & 670 & 0.134 (0.001) & \textbf{0.129} (0.001) & 0.135 (0.001) & 0.135 (0.001) & 0.135 (0.001) & 0.135 (0.001) \\
 &  & 1000 & 0.133 (0.001) & \textbf{0.129} (0.001) & 0.135 (0.001) & 0.134 (0.001) & 0.134 (0.001) & 0.135 (0.001) \\
 & \multirow{5}{*}{CIFAR10} & 120 & 0.125 (0.001) & \textbf{0.123} (0.001) & 0.126 (0.001) & 0.125 (0.001) & 0.125 (0.001) & 0.126 (0.001) \\
 &  & 230 & 0.125 (0.001) & \textbf{0.122} (0.001) & 0.126 (0.001) & 0.125 (0.001) & 0.125 (0.001) & 0.126 (0.001) \\
 &  & 450 & 0.125 (0.001) & \textbf{0.122} (0.001) & 0.126 (0.001) & 0.125 (0.001) & 0.125 (0.001) & 0.126 (0.001) \\
 &  & 670 & 0.124 (0.001) & \textbf{0.122} (0.001) & 0.126 (0.001) & 0.125 (0.001) & 0.125 (0.001) & 0.126 (0.001) \\
 &  & 1000 & 0.124 (0.001) & \textbf{0.122} (0.001) & 0.126 (0.001) & 0.124 (0.001) & 0.124 (0.001) & 0.126 (0.001) \\
 & \multirow{5}{*}{ImageNet10} & 10 & 0.129 (0.008) & \textbf{0.127} (0.007) & 0.129 (0.008) & 0.129 (0.008) & 0.129 (0.008) & 0.129 (0.008) \\
 &  & 20 & 0.129 (0.008) & \textbf{0.127} (0.006) & 0.129 (0.008) & 0.129 (0.008) & 0.129 (0.008) & 0.129 (0.008) \\
 &  & 50 & 0.129 (0.008) & \textbf{0.127} (0.006) & 0.129 (0.008) & 0.129 (0.008) & 0.129 (0.008) & 0.129 (0.008) \\
 &  & 70 & 0.128 (0.007) & \textbf{0.127} (0.006) & 0.129 (0.008) & 0.129 (0.008) & 0.129 (0.008) & 0.129 (0.008) \\
 &  & 100 & 0.128 (0.007) & \textbf{0.127} (0.006) & 0.129 (0.008) & 0.129 (0.008) & 0.129 (0.008) & 0.129 (0.008) \\
\bottomrule
\end{tabular}
\end{scriptsize}
\caption{Extension of Table \ref{tab:eval_UQ} with additional values of $s$ and standard errors over 5 seeds (in parentheses).}
\label{tab:eval_UQ_extended}
\end{table}

\newpage

\begin{figure*}[t!]
    \centering
    \begin{subfigure}{0.195\textwidth}
        \centering
        \includegraphics[width=\linewidth]{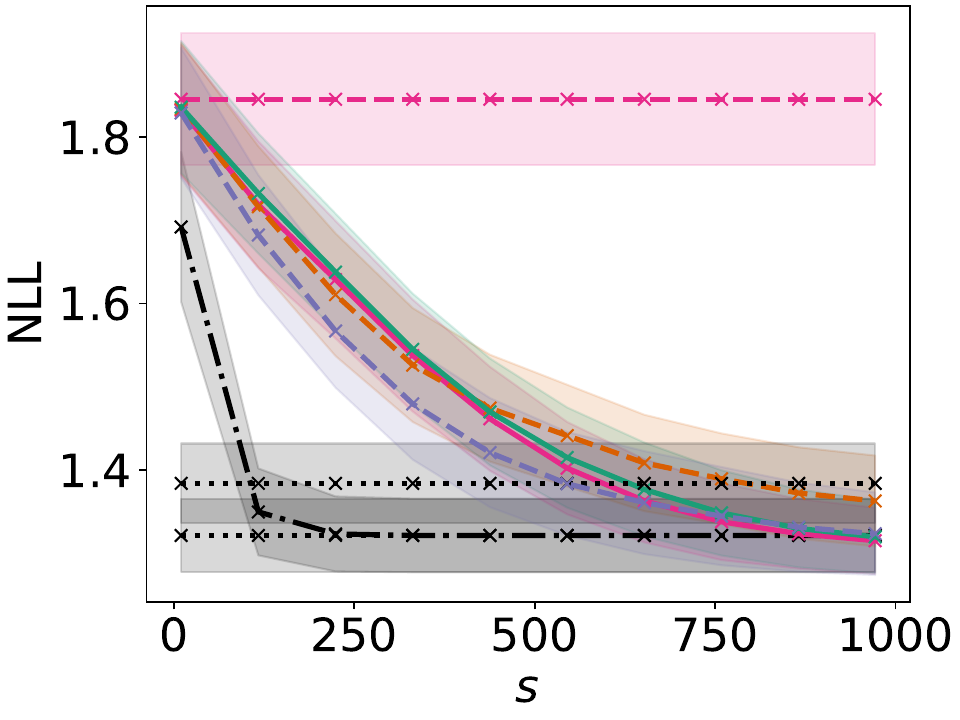}
        \caption{ENB}
    \end{subfigure}
    \begin{subfigure}{0.195\textwidth}
        \centering
        \includegraphics[width=\linewidth]{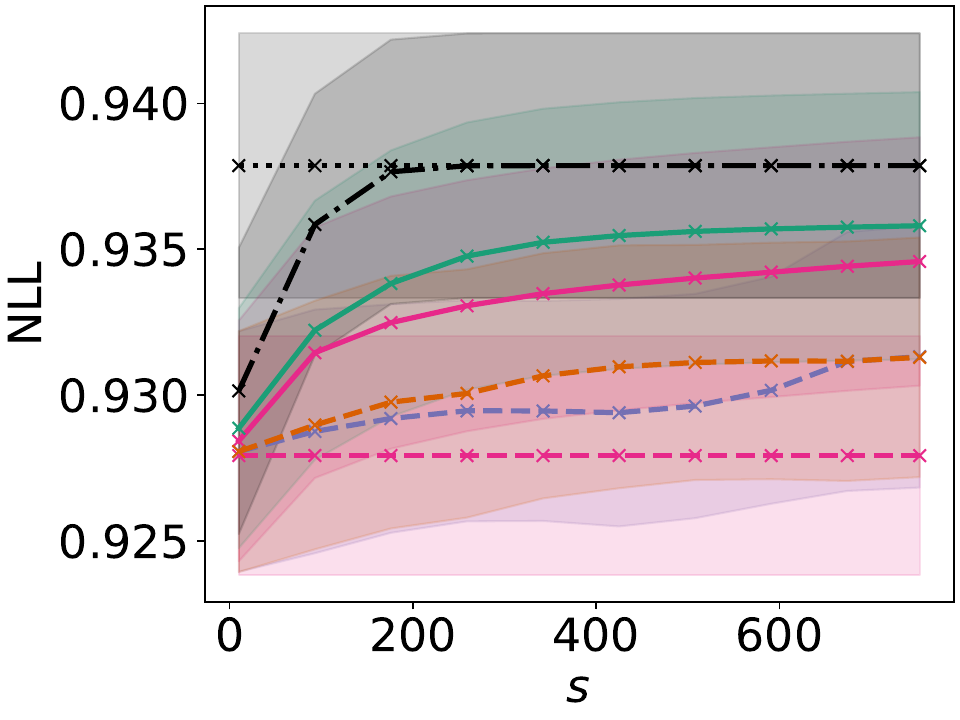}
        \caption{Red Wine}
    \end{subfigure}
    \begin{subfigure}{0.195\textwidth}
        \centering
        \includegraphics[width=\linewidth]{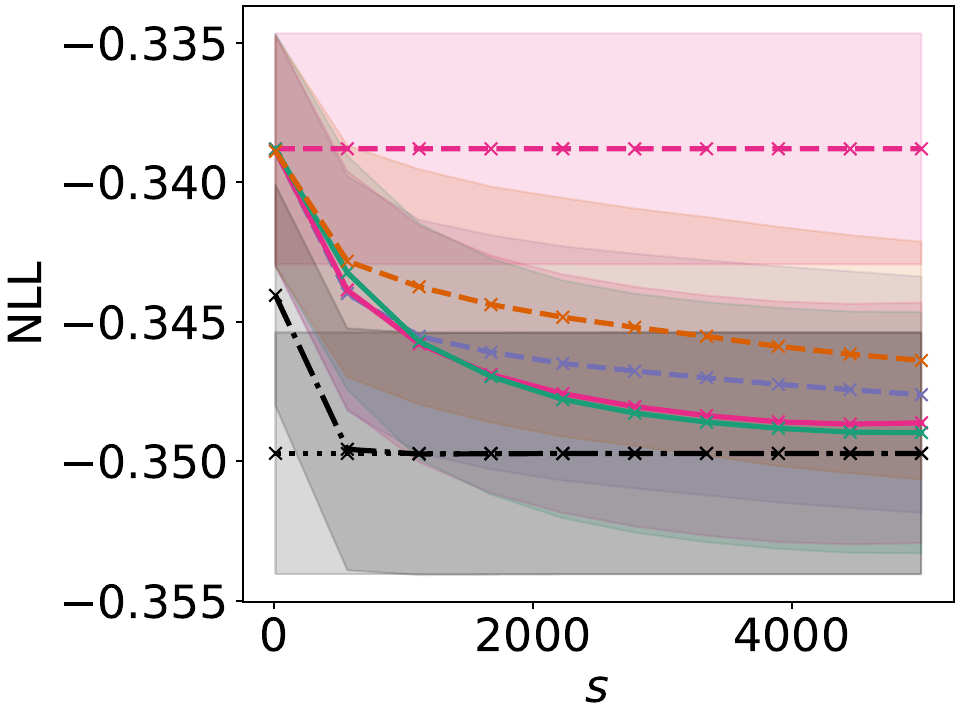}
        \caption{California}
    \end{subfigure}
    \begin{subfigure}{0.195\textwidth}
        \centering
        \includegraphics[width=\linewidth]{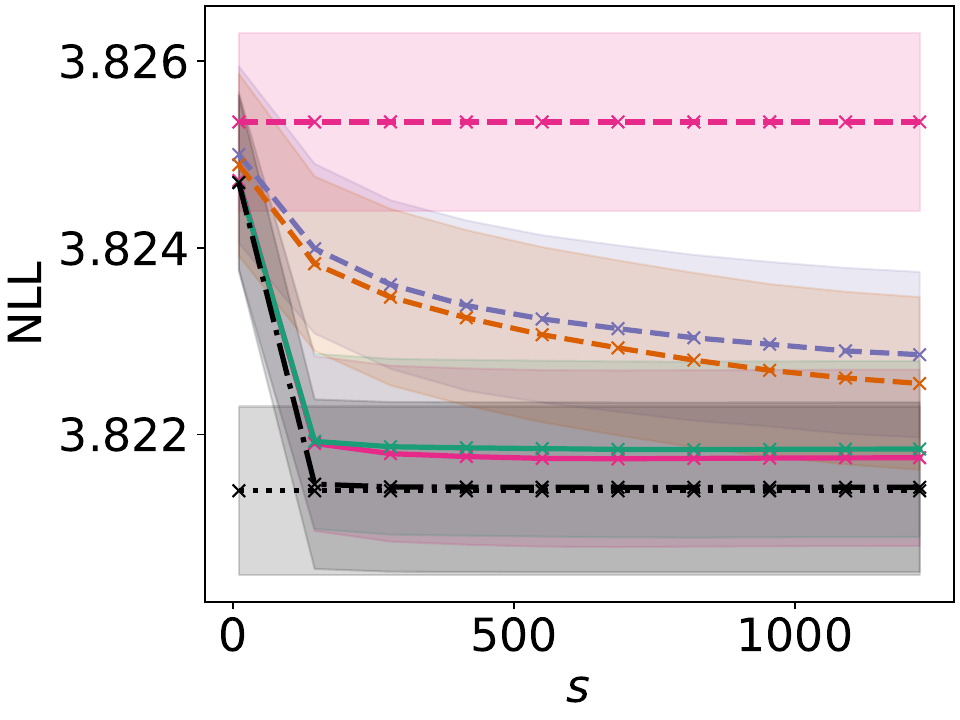}
        \caption{Naval Propulsion}
    \end{subfigure}
    \begin{subfigure}{0.195\textwidth}
        \centering
        \includegraphics[width=\linewidth]{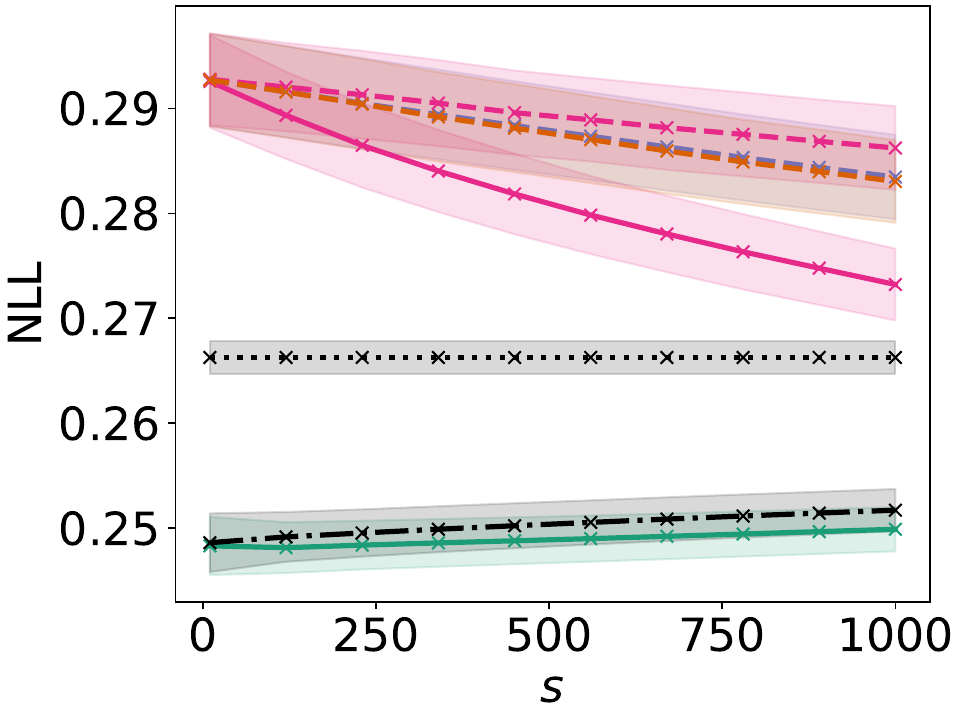}
        \caption{Fashion-MNIST}
    \end{subfigure}
    \begin{subfigure}{0.195\textwidth}
        \centering
        \includegraphics[width=\linewidth]{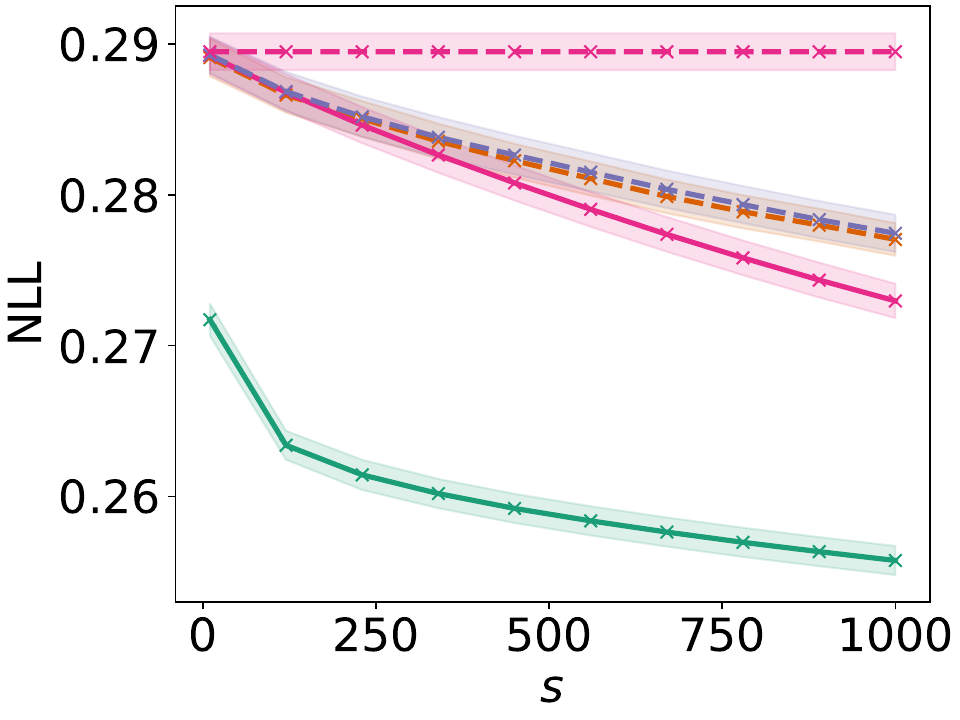}
        \caption{CIFAR10}
    \end{subfigure}
    \begin{subfigure}{0.195\textwidth}
        \centering
        \includegraphics[width=\linewidth]{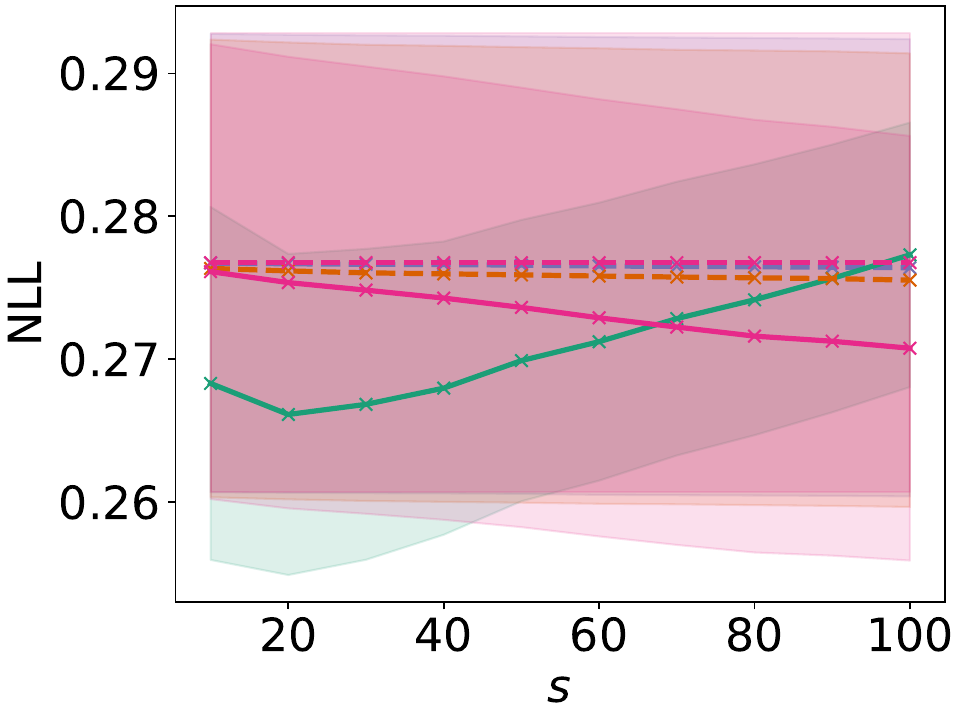}
        \caption{ImageNet10}
    \end{subfigure}
    \vspace{0.1cm} 
    \caption{Comparison of the NLL for all subspace models and datasets studied in this article and various subspace dimension $s$. Different choices of $P$ are marked by different colours and line types with the same colour coding as in Figure \ref{fig:plots_with_feasible_optimal_solution}.}
    \label{fig:all_nll_plots}
\end{figure*}

\section{ADDITIONAL MATERIAL ON CORRUPTED DATA}

\subsection{Corrupted CIFAR10}

We study 7 corruptions from \cite{hendrycks2019cifarc}, namely
\begin{itemize}
    \item 3 types of noise: impulse, shot and speckle noise,
    \item 2 types of blurring: Gaussian and motion blur,
    \item a distortion via an elastic transform,
    \item an increased luminance (brightness).
\end{itemize}
We use the data available at \href{https://doi.org/10.5281/zenodo.2535967}{https://doi.org/10.5281/zenodo.2535967} that contains the corruptions with different levels of severity. We do not distinguish between different levels of severity in our analysis. The entire set of results for all studied corruptions and subspace dimensions $s$ can be found in Table \ref{tab:OOD_extended} including standard errors obtained via 5 different seeds.

\begin{table}[p]
\begin{scriptsize}
\begin{tabular}{lllcccccc}
\toprule
metric & corruption & $s$ & lowrank-D. & lowrank-KFAC & subset-D. & subset-M. & subset-SWAG & $\hat{\theta}$ \\
\midrule
\multirow{21}{*}{NLL} & brightness & 100 & 0.342 (0.002) & \textbf{0.316} (0.002) & 0.344 (0.002) & 0.341 (0.002) & 0.342 (0.002) & 0.344 (0.002) \\
 & brightness & 200 & 0.340 (0.002) & \textbf{0.313} (0.002) & 0.344 (0.002) & 0.340 (0.002) & 0.340 (0.002) & 0.344 (0.002) \\
 & brightness & 300 & 0.338 (0.002) & \textbf{0.312} (0.002) & 0.344 (0.002) & 0.338 (0.002) & 0.338 (0.002) & 0.344 (0.002) \\
 & elastic transform & 100 & 0.524 (0.004) & \textbf{0.492} (0.004) & 0.527 (0.004) & 0.524 (0.004) & 0.524 (0.004) & 0.527 (0.004) \\
 & elastic transform & 200 & 0.522 (0.004) & \textbf{0.490} (0.004) & 0.527 (0.004) & 0.522 (0.004) & 0.522 (0.004) & 0.527 (0.004) \\
 & elastic transform & 300 & 0.520 (0.004) & \textbf{0.488} (0.004) & 0.527 (0.004) & 0.520 (0.004) & 0.521 (0.004) & 0.527 (0.004) \\
 & gaussian blur & 100 & 0.353 (0.002) & \textbf{0.342} (0.002) & 0.354 (0.002) & 0.353 (0.002) & 0.353 (0.002) & 0.354 (0.002) \\
 & gaussian blur & 200 & 0.353 (0.002) & \textbf{0.341} (0.002) & 0.354 (0.002) & 0.352 (0.002) & 0.353 (0.002) & 0.354 (0.002) \\
 & gaussian blur & 300 & 0.352 (0.002) & \textbf{0.341} (0.002) & 0.354 (0.002) & 0.352 (0.002) & 0.352 (0.002) & 0.354 (0.002) \\
 & impulse noise & 100 & 1.841 (0.057) & \textbf{1.679} (0.048) & 1.861 (0.059) & 1.848 (0.057) & 1.846 (0.056) & 1.861 (0.059) \\
 & impulse noise & 200 & 1.822 (0.055) & \textbf{1.660} (0.047) & 1.861 (0.059) & 1.836 (0.056) & 1.831 (0.055) & 1.861 (0.059) \\
 & impulse noise & 300 & 1.804 (0.054) & \textbf{1.647} (0.045) & 1.861 (0.059) & 1.823 (0.056) & 1.820 (0.055) & 1.861 (0.059) \\
 & motion blur & 100 & 0.589 (0.003) & \textbf{0.560} (0.003) & 0.591 (0.003) & 0.589 (0.003) & 0.589 (0.003) & 0.591 (0.003) \\
 & motion blur & 200 & 0.587 (0.003) & \textbf{0.558} (0.002) & 0.591 (0.003) & 0.587 (0.003) & 0.587 (0.003) & 0.591 (0.003) \\
 & motion blur & 300 & 0.585 (0.003) & \textbf{0.556} (0.002) & 0.591 (0.003) & 0.585 (0.003) & 0.586 (0.003) & 0.591 (0.003) \\
 & shot noise & 100 & 1.531 (0.037) & \textbf{1.400} (0.033) & 1.548 (0.038) & 1.536 (0.037) & 1.535 (0.037) & 1.548 (0.038) \\
 & shot noise & 200 & 1.515 (0.036) & \textbf{1.383} (0.033) & 1.548 (0.038) & 1.526 (0.036) & 1.524 (0.037) & 1.548 (0.038) \\
 & shot noise & 300 & 1.501 (0.036) & \textbf{1.371} (0.032) & 1.548 (0.038) & 1.516 (0.036) & 1.516 (0.036) & 1.548 (0.038) \\
 & speckle noise & 100 & 1.481 (0.035) & \textbf{1.352} (0.032) & 1.497 (0.036) & 1.485 (0.035) & 1.484 (0.035) & 1.497 (0.036) \\
 & speckle noise & 200 & 1.466 (0.034) & \textbf{1.335} (0.031) & 1.497 (0.036) & 1.476 (0.034) & 1.474 (0.034) & 1.497 (0.036) \\
 & speckle noise & 300 & 1.452 (0.033) & \textbf{1.324} (0.030) & 1.497 (0.036) & 1.466 (0.034) & 1.465 (0.034) & 1.497 (0.036) \\
\midrule
\multirow{21}{*}{ECE} & brightness & 100 & 0.045 (0.000) & \textbf{0.036} (0.000) & 0.046 (0.000) & 0.045 (0.000) & 0.045 (0.000) & 0.046 (0.000) \\
 & brightness & 200 & 0.044 (0.000) & \textbf{0.035} (0.000) & 0.046 (0.000) & 0.044 (0.000) & 0.045 (0.000) & 0.046 (0.000) \\
 & brightness & 300 & 0.044 (0.000) & \textbf{0.034} (0.000) & 0.046 (0.000) & 0.044 (0.000) & 0.044 (0.000) & 0.046 (0.000) \\
 & elastic transform & 100 & 0.064 (0.001) & \textbf{0.051} (0.001) & 0.065 (0.001) & 0.064 (0.001) & 0.064 (0.001) & 0.065 (0.001) \\
 & elastic transform & 200 & 0.063 (0.001) & \textbf{0.050} (0.001) & 0.065 (0.001) & 0.064 (0.001) & 0.064 (0.001) & 0.065 (0.001) \\
 & elastic transform & 300 & 0.063 (0.001) & \textbf{0.049} (0.001) & 0.065 (0.001) & 0.063 (0.001) & 0.063 (0.001) & 0.065 (0.001) \\
 & gaussian blur & 100 & 0.032 (0.000) & \textbf{0.023} (0.000) & 0.032 (0.000) & 0.031 (0.000) & 0.032 (0.000) & 0.032 (0.000) \\
 & gaussian blur & 200 & 0.031 (0.000) & \textbf{0.022} (0.000) & 0.032 (0.000) & 0.031 (0.000) & 0.031 (0.000) & 0.032 (0.000) \\
 & gaussian blur & 300 & 0.030 (0.000) & \textbf{0.021} (0.000) & 0.032 (0.000) & 0.030 (0.000) & 0.031 (0.000) & 0.032 (0.000) \\
 & impulse noise & 100 & 0.230 (0.004) & \textbf{0.206} (0.004) & 0.233 (0.004) & 0.231 (0.004) & 0.231 (0.004) & 0.233 (0.004) \\
 & impulse noise & 200 & 0.227 (0.004) & \textbf{0.203} (0.004) & 0.233 (0.004) & 0.229 (0.004) & 0.229 (0.004) & 0.233 (0.004) \\
 & impulse noise & 300 & 0.225 (0.004) & \textbf{0.200} (0.004) & 0.233 (0.004) & 0.228 (0.004) & 0.227 (0.004) & 0.233 (0.004) \\
 & motion blur & 100 & 0.071 (0.001) & \textbf{0.060} (0.001) & 0.072 (0.001) & 0.071 (0.001) & 0.071 (0.001) & 0.072 (0.001) \\
 & motion blur & 200 & 0.071 (0.001) & \textbf{0.059} (0.001) & 0.072 (0.001) & 0.071 (0.001) & 0.071 (0.001) & 0.072 (0.001) \\
 & motion blur & 300 & 0.070 (0.001) & \textbf{0.057} (0.001) & 0.072 (0.001) & 0.070 (0.001) & 0.070 (0.001) & 0.072 (0.001) \\
 & shot noise & 100 & 0.198 (0.006) & \textbf{0.181} (0.006) & 0.200 (0.006) & 0.198 (0.005) & 0.198 (0.006) & 0.200 (0.006) \\
 & shot noise & 200 & 0.196 (0.006) & \textbf{0.178} (0.006) & 0.200 (0.006) & 0.197 (0.005) & 0.197 (0.006) & 0.200 (0.006) \\
 & shot noise & 300 & 0.194 (0.006) & \textbf{0.176} (0.006) & 0.200 (0.006) & 0.196 (0.005) & 0.196 (0.006) & 0.200 (0.006) \\
 & speckle noise & 100 & 0.189 (0.004) & \textbf{0.172} (0.005) & 0.191 (0.004) & 0.190 (0.004) & 0.190 (0.004) & 0.191 (0.004) \\
 & speckle noise & 200 & 0.188 (0.004) & \textbf{0.170} (0.005) & 0.191 (0.004) & 0.189 (0.004) & 0.188 (0.004) & 0.191 (0.004) \\
 & speckle noise & 300 & 0.186 (0.004) & \textbf{0.168} (0.005) & 0.191 (0.004) & 0.188 (0.004) & 0.187 (0.004) & 0.191 (0.004) \\
\midrule
\multirow{21}{*}{Brier} & brightness & 100 & 0.149 (0.001) & \textbf{0.145} (0.001) & 0.149 (0.001) & 0.149 (0.001) & 0.149 (0.001) & 0.149 (0.001) \\
 & brightness & 200 & 0.149 (0.001) & \textbf{0.145} (0.001) & 0.149 (0.001) & 0.149 (0.001) & 0.149 (0.001) & 0.149 (0.001) \\
 & brightness & 300 & 0.148 (0.001) & \textbf{0.145} (0.001) & 0.149 (0.001) & 0.148 (0.001) & 0.148 (0.001) & 0.149 (0.001) \\
 & elastic transform & 100 & 0.230 (0.001) & \textbf{0.226} (0.001) & 0.231 (0.001) & 0.230 (0.001) & 0.230 (0.001) & 0.231 (0.001) \\
 & elastic transform & 200 & 0.230 (0.001) & \textbf{0.225} (0.001) & 0.231 (0.001) & 0.230 (0.001) & 0.230 (0.001) & 0.231 (0.001) \\
 & elastic transform & 300 & 0.230 (0.001) & \textbf{0.225} (0.001) & 0.231 (0.001) & 0.230 (0.001) & 0.230 (0.001) & 0.231 (0.001) \\
 & gaussian blur & 100 & 0.166 (0.001) & \textbf{0.164} (0.001) & 0.166 (0.001) & 0.165 (0.001) & 0.166 (0.001) & 0.166 (0.001) \\
 & gaussian blur & 200 & 0.165 (0.001) & \textbf{0.164} (0.001) & 0.166 (0.001) & 0.165 (0.001) & 0.165 (0.001) & 0.166 (0.001) \\
 & gaussian blur & 300 & 0.165 (0.001) & \textbf{0.164} (0.001) & 0.166 (0.001) & 0.165 (0.001) & 0.165 (0.001) & 0.166 (0.001) \\
 & impulse noise & 100 & 0.614 (0.008) & \textbf{0.597} (0.008) & 0.616 (0.008) & 0.615 (0.008) & 0.615 (0.008) & 0.616 (0.008) \\
 & impulse noise & 200 & 0.612 (0.008) & \textbf{0.595} (0.008) & 0.616 (0.008) & 0.614 (0.008) & 0.613 (0.008) & 0.616 (0.008) \\
 & impulse noise & 300 & 0.611 (0.008) & \textbf{0.593} (0.008) & 0.616 (0.008) & 0.612 (0.008) & 0.612 (0.008) & 0.616 (0.008) \\
 & motion blur & 100 & 0.259 (0.001) & \textbf{0.255} (0.001) & 0.259 (0.001) & 0.259 (0.001) & 0.259 (0.001) & 0.259 (0.001) \\
 & motion blur & 200 & 0.259 (0.001) & \textbf{0.254} (0.001) & 0.259 (0.001) & 0.259 (0.001) & 0.259 (0.001) & 0.259 (0.001) \\
 & motion blur & 300 & 0.258 (0.001) & \textbf{0.254} (0.001) & 0.259 (0.001) & 0.259 (0.001) & 0.259 (0.001) & 0.259 (0.001) \\
 & shot noise & 100 & 0.509 (0.009) & \textbf{0.496} (0.009) & 0.510 (0.010) & 0.509 (0.009) & 0.509 (0.009) & 0.510 (0.010) \\
 & shot noise & 200 & 0.507 (0.009) & \textbf{0.495} (0.009) & 0.510 (0.010) & 0.508 (0.009) & 0.508 (0.010) & 0.510 (0.010) \\
 & shot noise & 300 & 0.506 (0.009) & \textbf{0.493} (0.009) & 0.510 (0.010) & 0.507 (0.009) & 0.507 (0.009) & 0.510 (0.010) \\
 & speckle noise & 100 & 0.490 (0.008) & \textbf{0.478} (0.008) & 0.491 (0.008) & 0.490 (0.008) & 0.490 (0.008) & 0.491 (0.008) \\
 & speckle noise & 200 & 0.489 (0.008) & \textbf{0.476} (0.008) & 0.491 (0.008) & 0.489 (0.008) & 0.489 (0.008) & 0.491 (0.008) \\
 & speckle noise & 300 & 0.488 (0.008) & \textbf{0.475} (0.008) & 0.491 (0.008) & 0.489 (0.008) & 0.489 (0.008) & 0.491 (0.008) \\
\bottomrule
\end{tabular}
\end{scriptsize}
\caption{Extension of Table \ref{tab:OOD} for corrupted CIFAR10 including more corruptions, several values of $s$ and standard errors over 5 seeds (in parentheses). Smallest values in a row are marked in bold.}
\label{tab:OOD_extended}
\end{table}

\newpage

\subsection{Corrupted MNIST}
\label{subsec:corrupted_mnist}

In addition to corrupted CIFAR10 we also study a corrupted version of MNIST. We use the 15 corruption types from \cite{Mu2019mnistc}. These include:
\begin{itemize}
    \item Noise: shot noise and impulse noise simulate random corruptions during the imaging process.
    \item Blur: glass blur (viewing through frosted glass via local pixel shuffling) and motion blur (blurring along a random line).
    \item Affine transformations: shear, translate, scale, and rotate.
    \item Occlusions and overlays: stripe (inverted vertical stripe), spatter (random splotches), zigzag and dotted line (superimposed patterns with exponentially controlled brightness).
    \item Other: brightness (increased luminance) and canny edges (edge detection filter), fog (simulation via diamond square algorithm).
\end{itemize}
The model was trained on MNIST. We used the same architecture as for FashionMNIST. For details on the training procedure, see Section \ref{subsec:architectures_and_training}. For all subspace methods we fixed the subspace dimension to be $s=100$. Tables \ref{tab:MNIST_c_approximation_quality} and \ref{tab:MNIST_c_UQ_eval} report the result on approximation quality (compared to the full LA) and performance in uncertainty quantification. The approximation quality results show a clear preference for the lowrank methods, especially for $P_{\mathrm{lowrank-KFAC}}$. For the evaluation of uncertainty quantification the results are more mixed but in the majority of the corruptions we obtain the best scores for lowrank methods. We suspect the mixed behavior for this dataset to be linked to a non-optimal performance of the full LA itself, so that a larger deviation from the latter can lead to a better score.

\begin{table*}[h]
\begin{scriptsize}
\centering
\begin{tabular}{llccccc}
\toprule
metric & corruption & lowrank-Diagonal & lowrank-KFAC & subset-Diagonal & subset-Magnitude & subset-SWAG \\
\midrule
KL & brightness & 0.868 (0.006) & \textbf{0.445} (0.020) & 1.759 (0.083) & 1.500 (0.063) & 1.587 (0.069) \\
 & canny edges & 1.081 (0.033) & \textbf{0.481} (0.006) & 3.183 (0.112) & 2.338 (0.070) & 2.488 (0.072) \\
 & dotted line & 0.765 (0.015) & \textbf{0.320} (0.008) & 2.073 (0.046) & 1.645 (0.037) & 1.687 (0.037) \\
 & fog & 0.305 (0.008) & \textbf{0.212} (0.005) & 0.414 (0.015) & 0.386 (0.012) & 0.393 (0.015) \\
 & glass blur & 0.673 (0.005) & \textbf{0.356} (0.008) & 1.219 (0.026) & 1.030 (0.019) & 1.080 (0.023) \\
 & impulse noise & 1.224 (0.029) & \textbf{0.647} (0.011) & 3.067 (0.124) & 2.378 (0.064) & 2.572 (0.075) \\
 & motion blur & 0.636 (0.011) & \textbf{0.335} (0.008) & 1.210 (0.030) & 1.065 (0.024) & 1.069 (0.024) \\
 & rotate & 0.693 (0.013) & \textbf{0.283} (0.006) & 1.752 (0.032) & 1.453 (0.025) & 1.449 (0.024) \\
 & scale & 0.916 (0.011) & \textbf{0.373} (0.010) & 1.923 (0.056) & 1.616 (0.038) & 1.627 (0.042) \\
 & shear & 0.641 (0.010) & \textbf{0.270} (0.007) & 1.607 (0.032) & 1.342 (0.033) & 1.352 (0.031) \\
 & shot noise & 0.796 (0.009) & \textbf{0.360} (0.008) & 1.956 (0.050) & 1.544 (0.039) & 1.611 (0.039) \\
 & spatter & 0.636 (0.008) & \textbf{0.264} (0.007) & 1.666 (0.048) & 1.337 (0.037) & 1.358 (0.035) \\
 & stripe & 1.405 (0.039) & \textbf{0.833} (0.038) & 3.254 (0.449) & 3.442 (0.169) & 3.586 (0.141) \\
 & translate & 1.133 (0.029) & \textbf{0.536} (0.011) & 2.980 (0.025) & 2.526 (0.026) & 2.482 (0.024) \\
 & zigzag & 0.963 (0.022) & \textbf{0.444} (0.010) & 2.664 (0.051) & 2.209 (0.043) & 2.237 (0.042) \\
\midrule
Log-Trace & brightness & 10.681 (0.063) & \textbf{11.273} (0.075) & -inf & 8.876 (0.043) & 8.557 (0.018) \\
 & canny edges & 11.602 (0.082) & \textbf{12.208} (0.058) & -inf & 9.730 (0.049) & 9.488 (0.054) \\
 & dotted line & 11.572 (0.048) & \textbf{12.173} (0.022) & -inf & 9.420 (0.048) & 9.346 (0.032) \\
 & fog & 8.786 (0.033) & \textbf{9.371} (0.035) & -inf & 7.221 (0.040) & 7.015 (0.079) \\
 & glass blur & 10.286 (0.054) & \textbf{10.913} (0.041) & -inf & 8.600 (0.059) & 8.299 (0.023) \\
 & impulse noise & 11.337 (0.075) & \textbf{11.874} (0.059) & -inf & 9.607 (0.049) & 9.336 (0.033) \\
 & motion blur & 10.433 (0.054) & \textbf{10.997} (0.032) & -inf & 8.344 (0.053) & 8.304 (0.047) \\
 & rotate & 11.258 (0.045) & \textbf{11.909} (0.015) & -inf & 9.091 (0.048) & 9.075 (0.038) \\
 & scale & 10.783 (0.061) & \textbf{11.593} (0.028) & -inf & 8.833 (0.080) & 8.818 (0.033) \\
 & shear & 11.319 (0.039) & \textbf{11.922} (0.016) & -inf & 9.147 (0.041) & 9.114 (0.026) \\
 & shot noise & 11.247 (0.050) & \textbf{11.835} (0.030) & -inf & 9.259 (0.049) & 9.099 (0.032) \\
 & spatter & 11.466 (0.045) & \textbf{12.069} (0.020) & -inf & 9.277 (0.052) & 9.221 (0.033) \\
 & stripe & 11.719 (0.059) & \textbf{12.154} (0.039) & -inf & 9.751 (0.046) & 9.599 (0.022) \\
 & translate & 11.348 (0.049) & \textbf{11.972} (0.018) & -inf & 9.288 (0.062) & 9.305 (0.041) \\
 & zigzag & 11.621 (0.051) & \textbf{12.190} (0.025) & -inf & 9.486 (0.042) & 9.435 (0.029) \\
\bottomrule
\end{tabular}
\end{scriptsize}
\caption{Quality of approximation of the full Laplace approximation via the subspace models studied in this work on corrupted versions of MNIST for subspace dimension $s=100$. We use KL-divergence \ref{eq:KL} and logarithm of the trace \ref{eq:TraceOrderMain} to measure approximation quality. The standard error, across 5 seeds, is reported in parentheses. Smallest (KL) and highest (Log-Trace) values in a row are marked in bold. For $P_{\mathrm{subset-Diagonal}}$ the uncertainties are identical or almost identitical to zero for the considered $s$, which yields a Log-Trace of $-\infty$.}
\label{tab:MNIST_c_approximation_quality}
\end{table*}

\begin{table*}[p]
\begin{scriptsize}
\begin{tabular}{llcccccc}
\toprule
metric & corruption & lowrank-D. & lowrank-KFAC & subset-D. & subset-M. & subset-SWAG & $\hat{\theta}$ \\
\midrule
\multirow{15}{*}{NLL} & brightness & 0.106 (0.018) & 0.145 (0.023) & 0.100 (0.019) & 0.102 (0.019) & \textbf{0.099} (0.019) & 0.100 (0.019) \\
 & canny edges & 1.744 (0.156) & \textbf{1.334} (0.104) & 2.767 (0.303) & 2.324 (0.242) & 2.339 (0.237) & 2.784 (0.301) \\
 & dotted line & \textbf{0.066} (0.002) & 0.086 (0.002) & 0.076 (0.002) & 0.071 (0.002) & 0.069 (0.002) & 0.075 (0.002) \\
 & fog & 0.315 (0.017) & 0.350 (0.017) & 0.293 (0.017) & 0.300 (0.017) & 0.298 (0.017) & \textbf{0.291} (0.018) \\
 & glass blur & 0.233 (0.003) & 0.256 (0.002) & 0.235 (0.003) & 0.234 (0.003) & \textbf{0.232} (0.003) & 0.235 (0.003) \\
 & impulse noise & \textbf{0.327} (0.038) & 0.336 (0.032) & 0.407 (0.055) & 0.369 (0.047) & 0.363 (0.045) & 0.407 (0.055) \\
 & motion blur & 0.194 (0.025) & 0.218 (0.023) & 0.192 (0.028) & 0.192 (0.027) & \textbf{0.191} (0.027) & 0.192 (0.028) \\
 & rotate & 0.227 (0.003) & \textbf{0.220} (0.001) & 0.306 (0.004) & 0.281 (0.003) & 0.281 (0.002) & 0.308 (0.004) \\
 & scale & 0.133 (0.006) & 0.169 (0.006) & 0.136 (0.006) & 0.135 (0.007) & \textbf{0.131} (0.006) & 0.136 (0.006) \\
 & shear & \textbf{0.067} (0.001) & 0.082 (0.002) & 0.077 (0.002) & 0.073 (0.002) & 0.074 (0.002) & 0.078 (0.002) \\
 & shot noise & \textbf{0.118} (0.006) & 0.131 (0.006) & 0.138 (0.007) & 0.130 (0.007) & 0.128 (0.007) & 0.138 (0.007) \\
 & spatter & \textbf{0.068} (0.003) & 0.081 (0.002) & 0.083 (0.004) & 0.077 (0.003) & 0.077 (0.004) & 0.083 (0.004) \\
 & stripe & 0.832 (0.244) & \textbf{0.750} (0.188) & 1.460 (0.505) & 1.186 (0.391) & 1.250 (0.428) & 1.485 (0.504) \\
 & translate & 0.940 (0.034) & \textbf{0.799} (0.030) & 1.413 (0.074) & 1.298 (0.060) & 1.277 (0.065) & 1.456 (0.077) \\
 & zigzag & 0.491 (0.027) & \textbf{0.422} (0.020) & 0.770 (0.041) & 0.665 (0.039) & 0.654 (0.033) & 0.753 (0.044) \\
\midrule
\multirow{15}{*}{ECE} & brightness & 0.025 (0.005) & 0.062 (0.009) & 0.008 (0.002) & 0.010 (0.001) & 0.009 (0.001) & \textbf{0.008} (0.001) \\
 & canny edges & 0.235 (0.019) & \textbf{0.164} (0.018) & 0.292 (0.021) & 0.280 (0.020) & 0.271 (0.019) & 0.293 (0.020) \\
 & dotted line & \textbf{0.006} (0.001) & 0.031 (0.001) & 0.010 (0.000) & 0.007 (0.000) & 0.008 (0.000) & 0.010 (0.000) \\
 & fog & 0.141 (0.010) & 0.169 (0.010) & 0.120 (0.009) & 0.125 (0.009) & 0.124 (0.009) & \textbf{0.118} (0.009) \\
 & glass blur & 0.013 (0.001) & 0.049 (0.002) & 0.014 (0.002) & \textbf{0.009} (0.000) & 0.010 (0.001) & 0.014 (0.002) \\
 & impulse noise & \textbf{0.017} (0.006) & 0.034 (0.003) & 0.058 (0.009) & 0.047 (0.008) & 0.047 (0.008) & 0.058 (0.009) \\
 & motion blur & 0.021 (0.003) & 0.056 (0.002) & 0.010 (0.002) & 0.010 (0.001) & \textbf{0.009} (0.001) & 0.009 (0.002) \\
 & rotate & 0.023 (0.001) & \textbf{0.015} (0.001) & 0.044 (0.001) & 0.040 (0.001) & 0.040 (0.000) & 0.044 (0.001) \\
 & scale & 0.010 (0.001) & 0.058 (0.002) & 0.012 (0.001) & 0.009 (0.001) & \textbf{0.008} (0.001) & 0.012 (0.001) \\
 & shear & \textbf{0.005} (0.000) & 0.025 (0.001) & 0.010 (0.000) & 0.008 (0.000) & 0.009 (0.000) & 0.010 (0.000) \\
 & shot noise & \textbf{0.004} (0.001) & 0.029 (0.001) & 0.017 (0.001) & 0.014 (0.001) & 0.014 (0.001) & 0.017 (0.001) \\
 & spatter & \textbf{0.004} (0.000) & 0.022 (0.001) & 0.011 (0.000) & 0.009 (0.000) & 0.009 (0.001) & 0.011 (0.000) \\
 & stripe & 0.107 (0.056) & \textbf{0.091} (0.051) & 0.197 (0.069) & 0.164 (0.062) & 0.163 (0.065) & 0.180 (0.064) \\
 & translate & 0.123 (0.006) & \textbf{0.057} (0.007) & 0.180 (0.009) & 0.173 (0.008) & 0.170 (0.009) & 0.184 (0.010) \\
 & zigzag & 0.067 (0.005) & \textbf{0.026} (0.003) & 0.105 (0.007) & 0.097 (0.007) & 0.094 (0.006) & 0.103 (0.007) \\
\midrule
\multirow{15}{*}{Brier} & brightness & 0.049 (0.009) & 0.061 (0.011) & 0.050 (0.010) & 0.051 (0.010) & \textbf{0.049} (0.010) & 0.050 (0.010) \\
 & canny edges & 0.592 (0.036) & \textbf{0.545} (0.033) & 0.632 (0.039) & 0.624 (0.038) & 0.605 (0.035) & 0.633 (0.038) \\
 & dotted line & 0.032 (0.001) & 0.037 (0.001) & 0.034 (0.001) & 0.033 (0.001) & \textbf{0.032} (0.001) & 0.033 (0.001) \\
 & fog & 0.128 (0.008) & 0.140 (0.009) & 0.121 (0.008) & 0.123 (0.008) & 0.122 (0.008) & \textbf{0.120} (0.008) \\
 & glass blur & 0.115 (0.001) & 0.121 (0.001) & 0.116 (0.002) & 0.117 (0.001) & \textbf{0.115} (0.001) & 0.116 (0.002) \\
 & impulse noise & \textbf{0.160} (0.018) & 0.162 (0.016) & 0.173 (0.019) & 0.169 (0.019) & 0.162 (0.017) & 0.170 (0.020) \\
 & motion blur & 0.092 (0.012) & 0.097 (0.011) & 0.090 (0.013) & 0.091 (0.012) & 0.090 (0.012) & \textbf{0.090} (0.013) \\
 & rotate & 0.109 (0.001) & \textbf{0.107} (0.001) & 0.117 (0.001) & 0.115 (0.001) & 0.115 (0.001) & 0.117 (0.001) \\
 & scale & 0.065 (0.003) & 0.073 (0.003) & 0.064 (0.003) & 0.065 (0.004) & \textbf{0.063} (0.003) & 0.064 (0.003) \\
 & shear & \textbf{0.032} (0.001) & 0.035 (0.001) & 0.033 (0.001) & 0.033 (0.001) & 0.034 (0.001) & 0.033 (0.001) \\
 & shot noise & \textbf{0.055} (0.002) & 0.057 (0.002) & 0.056 (0.002) & 0.057 (0.002) & 0.055 (0.003) & 0.056 (0.002) \\
 & spatter & \textbf{0.034} (0.001) & 0.036 (0.001) & 0.035 (0.001) & 0.035 (0.001) & 0.035 (0.002) & 0.035 (0.001) \\
 & stripe & 0.364 (0.102) & \textbf{0.358} (0.100) & 0.473 (0.128) & 0.403 (0.121) & 0.398 (0.124) & 0.415 (0.124) \\
 & translate & 0.394 (0.013) & \textbf{0.370} (0.013) & 0.428 (0.017) & 0.426 (0.016) & 0.420 (0.017) & 0.432 (0.018) \\
 & zigzag & 0.217 (0.013) & \textbf{0.205} (0.011) & 0.244 (0.013) & 0.234 (0.014) & 0.230 (0.013) & 0.237 (0.015) \\
\bottomrule
\end{tabular}
\end{scriptsize}
\caption{Performance of uncertainty quantification on corrupted MNIST for all subspace methods studied in this work (with $s=100)$ and the MAP solution \eqref{eq:MAP} ($\hat{\theta}$). The standard error (in parantheses) is reported for 5 seeds.}
\label{tab:MNIST_c_UQ_eval}
\end{table*}

\newpage

\section{BEHAVIOR FOR SMALLER AND LARGER $s$}
\label{sec:behavior_for_smaller_and_larger_s}

\paragraph{Very Small $s$ Values. }Due to issues in computational complexity other methods in the literature, such as \citep{Izmailov2019}, use very small values of $s$. For completeness, we therefore report the performance of the lowrank and subset methdos in the very small $s$ regime for the Red Wine dataset. Figures \ref{fig:redwine_small_s_kl} and \ref{fig:redwine_small_s_logtrace} show the KL-divergence \ref{eq:KL} and logarithm of the trace criterion \eqref{eq:TraceOrderMain} for $s$ between 1 and 20. We observe no fundamentaly different behavior in this regime. For these small $s$ the approximation quality is insufficient as revealed by the KL-divergence. This is in line with the observation from \cite{Izmailov2019} that a temperature is needed in this regime to obtain reasonable uncertainties. In this work, we do not consider such modifications of the posterior.

\paragraph{Large Values of $s$. } In Figure \ref{fig:plots_with_feasible_optimal_solution} it appears that $P_{\mathrm{lowrank-Diagonal}}$ yields a KL-divergence that does not improve at all with increasing $s$ for the Red Wine dataset. However, this is only due to the fact that we only consider $s$ below 1000. Once we consider very large values of $s$ the Kl divergence starts to drop finally around $s=10,000$ as shown in Figure \ref{fig:redwine_large_s_kl}. Note, that Figure \ref{fig:redwine_large_s_kl} only shows subset methods as the low rank methods, introduced in this work, can only be used for $s$ below the rank of the Jacobian (which is around $700$ for Red Wine).

\begin{figure}[h]
    \centering
    \begin{subfigure}{0.32\linewidth}
    \includegraphics[width=\textwidth]{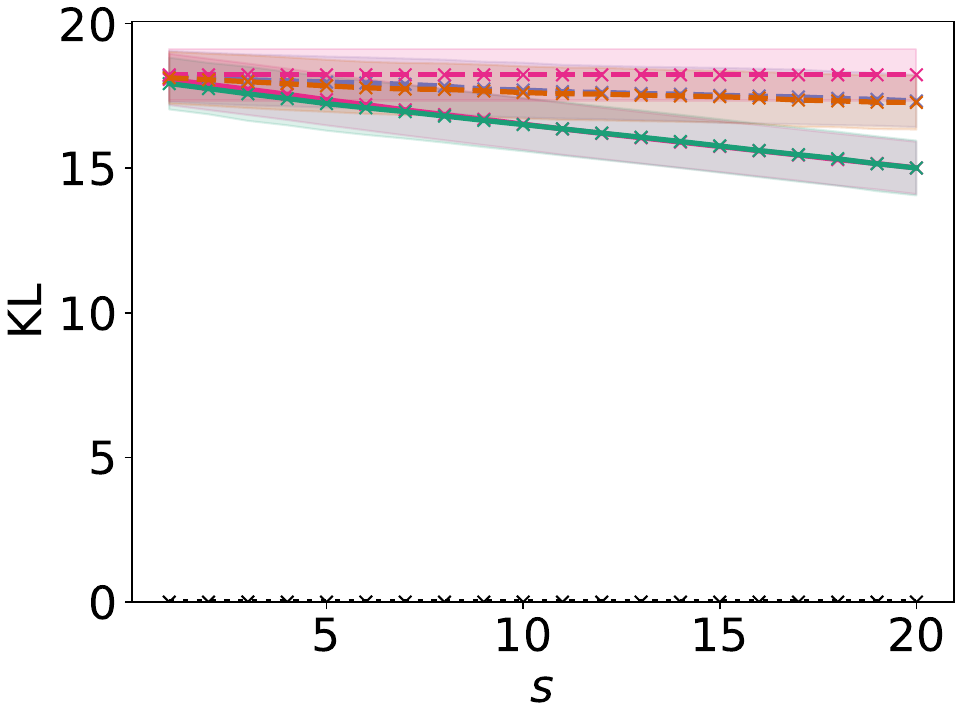}
    \caption{Red Wine (small $s$) - KL divergence}
    \label{fig:redwine_small_s_kl}
    \end{subfigure}%
    \begin{subfigure}{0.32\linewidth}
    \includegraphics[width=\textwidth]{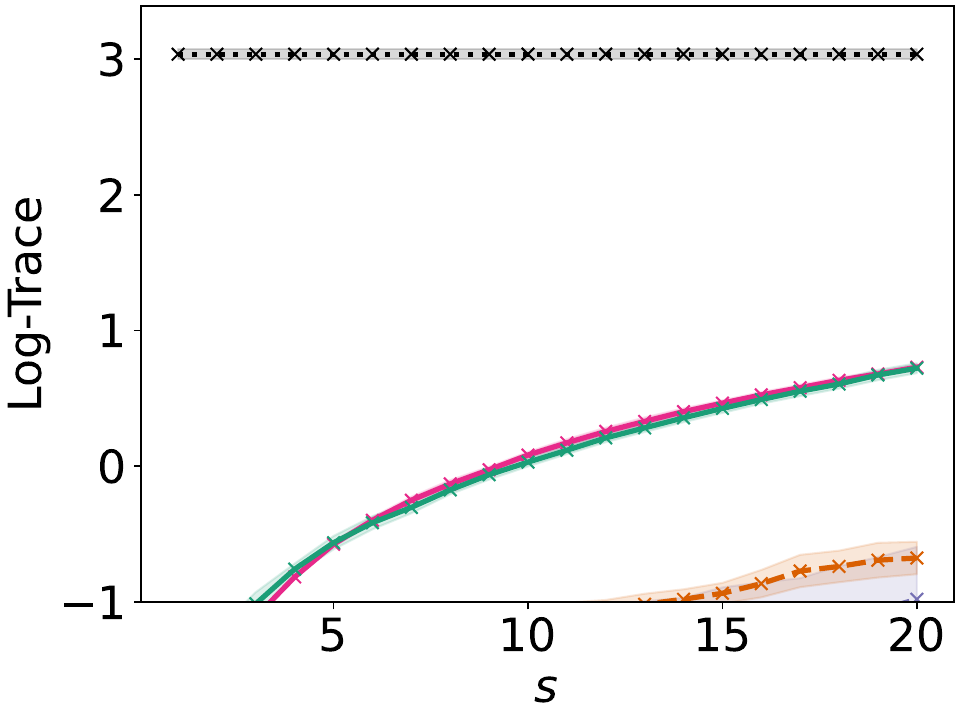}
    \caption{Red Wine (small $s$) - Log-Trace}
    \label{fig:redwine_small_s_logtrace}
    \end{subfigure} \\
    \begin{subfigure}{0.32\linewidth}
    \includegraphics[width=\textwidth]{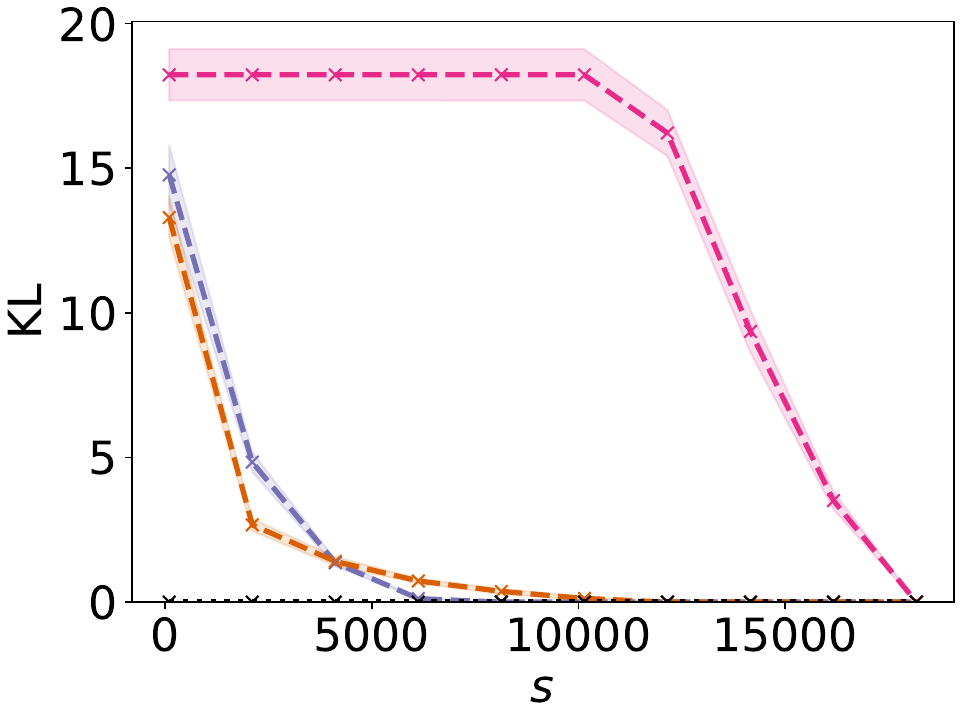}
    \caption{Red Wine (large $s$) - KL-divergence}
    \label{fig:redwine_large_s_kl}
    \end{subfigure}%
    \begin{subfigure}{0.32\linewidth}
    \includegraphics[width=\textwidth]{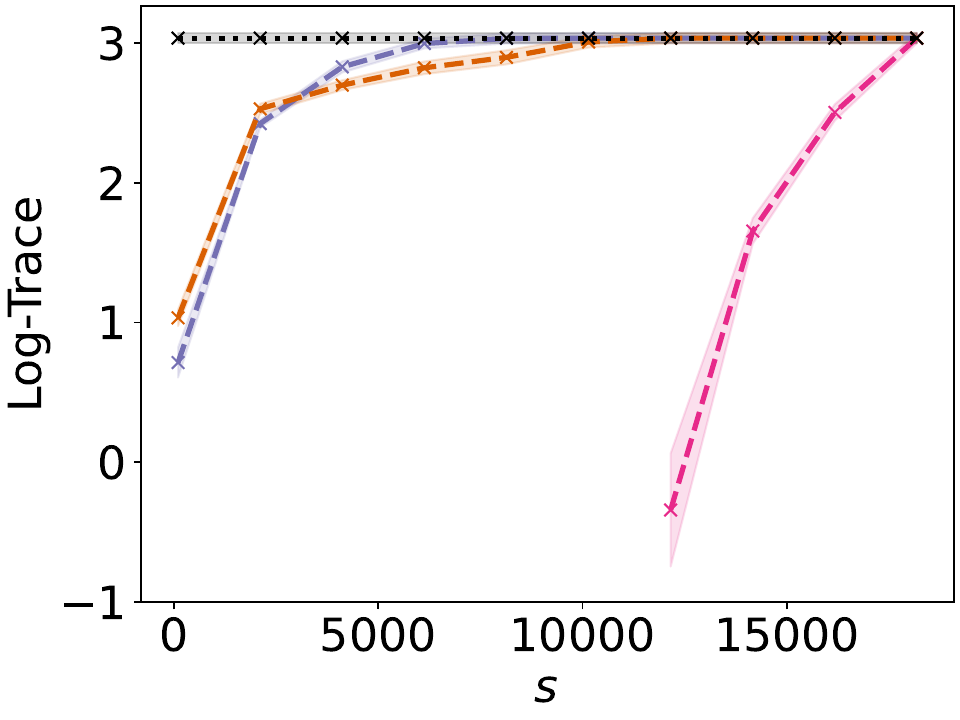}
    \caption{Red Wine (large $s$) - Log-Trace}
    \label{fig:redwine_large_s_logtrace}
    \end{subfigure}
    \caption{KL-divergence \eqref{eq:KL} (left column)) and logarithm of the trace \ref{eq:TraceOrderMain} (right column) for the Red Wine dataset and very small (top row) and very large (bottom row) values of $s$. The colour coding is identical to the one in Figure \ref{fig:plots_with_feasible_optimal_solution}. The bottom row does not show low rank methods as explained in Section \ref{sec:behavior_for_smaller_and_larger_s}.}
\end{figure}

\section{DEAD PARAMETERS}
    \label{sec:AppendixDeadParam}

\begin{table}[h!]
\centering
 \caption{Relative number of parameters $p$ that are insensitive to the input data for the regression datasets studied in this work. The standard deviation is taken over five seeds.}
 \label{tab:DeadParam}
 \begin{tabular}{c c c c c c c c} 
 \toprule
 dataset & ENB & Red Wine & California & Naval \\
 \midrule
  dead $p$  & $92 \pm 1\% $ & $89 \pm 2\%$ & $60\pm 2\%$ & $34\pm 4\%$ \\
 \bottomrule
 \end{tabular}
\end{table}

Note that our low rank methods allow for a wider class of subspace models as the subset methods, because they allow for linear combinations of the parameters instead of a simple selection.
 In fact, all subset solutions could be found by the low rank approximations, however, the opposite is not true. One reason why subset methods could nevertheless outperform low rank methods is that most of the parameters are irrelevant for a certain problem, i.e. have a gradient of zero w.r.t. the input. Indeed, Table \ref{tab:DeadParam} confirms this hypotheses, because the number of insensitive parameters positively correlates with an improved performance of the subset methods compared to the low rank methods. ENB, with most dead parameters, is the only experiment in which the subspace models are superior to the low rank subspace models according to the trace criterion \eqref{eq:TraceOrder}, cf. Figure \ref{fig:plots_with_feasible_optimal_solution}. For California or Naval Propulsion on the other hand, with much fewer dead parameters, low rank approximations clearly outperform subset approximations (cf. Figure \ref{fig:plots_with_feasible_optimal_solution}).

This effect is visualized in Figure \ref{fig:dead_parameters}. The top displays a heatmap that highlights the sensitivity of parameters (the gradient w.r.t. the input) for a subset of the test dataset of ENB (left) and California (right). Light colours denote high sensitivity and dark colours low sensitivity. Below the average gradient of all data points w.r.t. a certain parameter is shown. For ENB both plots indicate that only a few parameters are highly responsive for most data points. If a subset method can capture these parameters, it shall perform well. In contrast, for California the sensitivity is more spread and hence, a linear combination can be more appropriate.

 \begin{figure*}[!t]
    \centering
    \begin{subfigure}[b]{0.45\textwidth}
        \includegraphics[width=\textwidth]{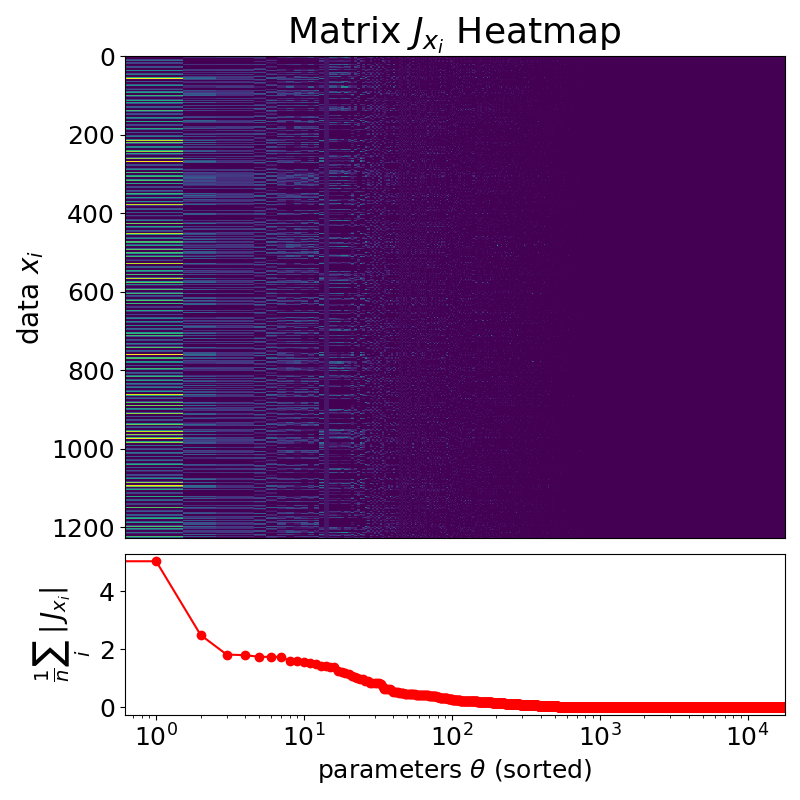}
        \caption{ENB}
    \end{subfigure}
    \hfill
    \begin{subfigure}[b]{0.45\textwidth}
        \includegraphics[width=\textwidth]{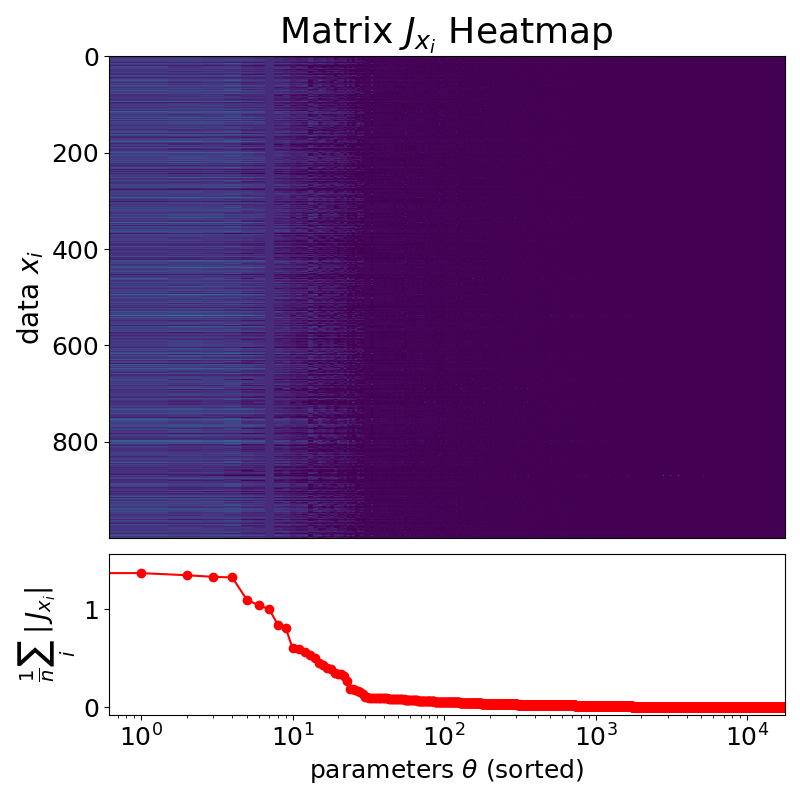}
        \caption{California}
    \end{subfigure}
    \caption{The top displays a heatmap which highlights the activity of the gradients corresponding to the parameter $\theta_i$ for a subset of the test data for the regression datasets ENB (left) and Carlifornia (right). The parameters are sorted according to their sensitivity. A dark bluish colour implies that the gradient w.r.t. the data point $x_i$ is negligible. The lower plot summarizes the magnitude of the sensitivity over all data points.}
    \label{fig:dead_parameters}
\end{figure*}

\section{FISHER INFORMATION, GENERALIZED GAUSS NEWTON AND HESSIAN}
\label{sec:FI_GGN_and_H}

For the computations in our experiments we use the Fisher information matrix $\mathcal{I}$ instead of the generalized Gauss-Newton matrix $H_{\mathrm{GGN}}$ which are identical objects for the cases considered in this work \cite{Heskes2000, Martens14, Dauphin2014}.
We summarize the relation between $\mathcal{I}$, $H_{\mathrm{GGN}}$ and the Hessian below in Section \ref{sec:FI}  -\ref{sec:FI_and_GGN}. For a more detailed analysis, we refer to the survey \cite{Martens14}.

As follows from the identities in \ref{sec:FI} - \ref{sec:FI_and_GGN} we have $\mathcal{I}=VV^\intercal$ with a $V\in \RR^{p\times NC}$ that can be computed via minibatches from $\mathcal{D}$ and expressions that involve first order derivatives of $f_\theta$. This allows us to compute for any matrix $P\in \RR^{p\times s}$ the expression 
\begin{align}
\label{eq:quadratic_Fisher_Form}
P^\intercal H_{\mathrm{GGN}}P =P^\intercal \mathcal{I} P  = (V^\intercal P)^\intercal V^\intercal P \in \RR^{s\times s}
\end{align}
in a scalable manner if $s$ is sufficiently small. Thus, while we often can't actually compute $H_{\mathrm{GGN}}$ or $\mathcal{I}$ in practice, we can usually compute quadratic forms such as \eqref{eq:quadratic_Fisher_Form}.


\subsection{Fisher Information}
\label{sec:FI}

\subsubsection{Multivariate Regression}

\[
    p(y_i | x_i, \theta) = \frac{1}{\sqrt{(2 \pi)^{C} \det(\Sigma)}}
    e^{- \frac{1}{2} \norm{y_i - f(x_i, \theta)}^2_{\Sigma^{-1}}}
\]
is a common choice to model multivariate regression problems. For simplicity we assume that the covariance matrix $\Sigma \in \mathbb{R}^{C\times C}$ is independent of the parameter $\theta$. The explicit form of the information matrix for a single input $x_i$ is 
\eqn{
    \mathcal{I}_{kl}(x_i) & = \frac{1}{2}
    \E{y \sim p(y | x_i, \theta)}{\partial_{\theta_k}\partial_{\theta_l} \norm{y_i - f_i}^2_{\Sigma^{-1}}} \\
    & = \sum_{c_1,c_2=1}^C \partial_{\theta_k} f_i^{c_1} \left(\Sigma^{-1}\right)^{c_1 c_2} \partial_{\theta_l} f^{c_2}_i \\
    & = \left(\left(\nabla_{\theta} f_i\right)^{\intercal} \Sigma^{-1} \nabla_{\theta} f_i \right)_{kl}\,. 
}{eq:IMultiRegr}
The abbreviation $f_i=f_{\theta}(x_i)$ is used for readability. 
For $\Sigma = \sigma^2 \mathbb{1}$ (as in this work), we obtain 
\eqnn{
    \mathcal{I}_{kl}(x_i) = \sigma^{-2}\left(\left(\nabla_{\theta} f_i\right)^{\intercal} \nabla_{\theta} f_i \right)_{kl} \,.
}
For the Fisher information matrix of the joint distribution we arrive at
\eqnn{
    \mathcal{I}_{kl} & = \frac{1}{2}
    \E{(x, y) \sim p(y,  x | \theta)}{\partial_{\theta_k}\partial_{\theta_l} \norm{y_i - f_i}^2_{\sigma^{-2}\mathbb{1}}} \\
    & \simeq \frac{\sigma^{-2}}{N} \sum_{i=1}^N \left(\left(\nabla_{\theta} f_i\right)^{\intercal} \nabla_{\theta} f_i \right)_{kl}\,,
}
where in the last line $q(x)$ is approximated by $\hat{q}(x)$.

\subsubsection{Softmax Classifier}

For classification we consider the categorical distribution $y| x, \theta \sim \mathrm{Cat}(y | \phi(f_{\theta}(x))$ with probability vector
\[
    \phi^c_i = \phi^c(f_{\theta}(x_i)) 
    = \frac{e^{f_\theta^c(x_i)}}{\sum_{\tilde{c}= 1}^C e^{f^{\tilde{c}}_\theta(x_i)}}
    = \frac{e^{f^c_i}}{\sum_{\tilde{c}= 1}^C e^{f^{\tilde{c}}_i}}.
\]
The general form of the Fisher information matrix is given by 
\eqnn{ 
    \mathcal{I}_{kl} 
    & = \E{(x, y) \sim p(x, y | \theta)}{\partial_{\theta_k} \ln p(x, y | \theta) \partial_{\theta_l} \ln p(x, y | \theta)}  \\
    & = \E{y \sim p(y | x, \theta), x \sim q(x)}{\partial_{\theta_k} \ln p(y|x, \theta) \partial_{\theta_l} \ln p(y|x, \theta)}  \\
    & \simeq \frac{1}{N} \sum_{i=1}^N \sum_{c=1}^C \phi^c_i \partial_{\theta_k} \ln \phi^c_i \partial_{\theta_l} \ln \phi^c_i  \\
    & = \frac{4}{N} \sum_{i=1}^N \sum_{c=1}^C  \partial_{\theta_k} \sqrt{\phi^c_i} \partial_{\theta_l} \sqrt{\phi^c_i} \,,
}
where in the third line the empirical distribution $\hat{q}(x)$ is used to compute the expected value of the random variable $x$.

\subsection{Relation Between Hessian and Fisher Information Matrix}
\label{sec:FI_and_H}
Given the averaged log-likelihood $\frac{1}{N} \sum_{i} \ln p(y_i | f_{\theta}(x_i))$ of the data its Hessian w.r.t $\theta$ can be written as
\eqn{
    H & = - \frac{1}{N} \sum_{i=1}^N \nabla^2_{\theta}\ln p(y_i | f_{\theta}(x_i)) \\
    & = \E{(x,y) \sim \hat{q}(x, y)}{H_{\text{-}\ln p(y | x, \theta)}} \\
    & = \frac{1}{N} \sum_{i=1}^N \E{y \sim \hat{q}(y| x_i)}{H_{\text{-}\ln p(y | x_i, \theta)}} \,,
}{eq:HessianSimilarFisher}
where we wrote $H_{\text{-}\ln p(y | x, \theta)} = - \nabla^2_{\theta}\ln p(y | f_{\theta}(x))$. 

The Fisher information matrix $\mathcal{I}$ of $p(x, y | \theta)$ w.r.t the parameter $\theta$ is
\eqnn{
    \mathcal{I} & = \E{(x,y) \sim p(x, y | \theta)}{\nabla_{\theta} \ln p(x, y | \theta)^\intercal \nabla_{\theta} \ln p(x, y | \theta)} \\
    & = \E{y \sim p(y | x, \theta), x \sim q(x)}{\nabla_{\theta} \ln p(y | x, \theta)^\intercal \nabla_{\theta} \ln p(y | x, \theta)} \\
    & = - \E{y \sim p(y | x, \theta), x \sim q(x)}{\nabla_{\theta}^2 \ln p(y | x, \theta)} \\
    & = \E{y \sim p(y | x, \theta), x \sim q(x)}{H_{\text{-}\ln p(y | x, \theta)}} \,.
}
Since $q(x)$ is not analytically known, we shall use the empirical distribution $\hat{q}(x)$ instead.
\eqn{
    \mathcal{I} = \frac{1}{N} \sum_{i=1}^N \E{y \sim p(y | x=x_i, \theta)}{H_{\text{-}\ln p(y | x, \theta)}} \,.
}{eq:FisherSimilarHessian}
The equations \eqref{eq:FisherSimilarHessian} and \eqref{eq:HessianSimilarFisher} are quite similar. The difference is the distribution under which the expectation is computed. However, note that \eqref{eq:HessianSimilarFisher} and \eqref{eq:FisherSimilarHessian} are different from the empirical Fisher information matrix
\eqnn{
    \mathcal{I_{\mathrm{empirical}}} & = \E{y \sim \hat{q}(x, y | \theta)}{\nabla_{\theta} \ln p(x, y | \theta)^\intercal \nabla_{\theta} \ln p(x, y | \theta)} \\
    & = \frac{1}{N} \sum_{i=1}^N \nabla_{\theta} \ln p(y_i | x_i,\theta)^\intercal \nabla_{\theta} \ln p(y_i | x_i, \theta).
}

\subsection{Relation Between Hessian and Generalized Gauss-Newton Matrix}
\label{sec:GGN_and_H}

The generalized Gauss-Newton matrix is often used as a substitute of the Hessian because it is positive semi-definite and easier to compute \cite{Martens14}. For generalized linear models both quantities coincide. Let us write the Jacobian w.r.t. the log-likelihood as $\nabla_{\theta} \ln p(y_i | f_{\theta}(x_i)) =  \nabla_{f_i} \ln p(y_i | f_i) \nabla_{\theta} f_{\theta}(x_i)=  \nabla_{f_i} \ln p(y_i | f_i)J_{f_i}$ and $H_{f_i^c} = \nabla_{\theta}^2 f_{\theta}(x_i)^c$ for $1\leq c\leq C$. Then the Hessian can be decomposed into
\eqn{
    H  = &\frac{-1}{N} \sum_{i=1}^N \Big(J_{f_i}^\intercal \nabla_{f_i}^2 \ln p(y_i | f_i) J_{f_i}   \\
     & + \sum_{c=1}^C H_{f_i^c} \partial_{f_i^c} \ln p(y_i | f_i) \Big)\\
     = &H_{\mathrm{GGN}} - \frac{1}{N} \sum_{i=1}^N \sum_{c=1}^C H_{f_i^c} \partial_{f_i^c} \ln p(y_i | f_i)
}{eq:HessinaSimilarGGN}
with the generalized Gauss-Newton matrix 
\eqn{
    H_{\mathrm{GGN}} & = - \frac{1}{N} \sum_{i=1}^N J_{f_i}^\intercal \nabla_{f_i}^2 \ln p(y_i | f_i) J_{f_i} \\
    & = \frac{1}{N} \sum_{i=1}^N J_{f_i}^\intercal H_{\text{-}\ln p(y_i | {f_i})} J_{f_i}\,.
}{eq:AppendixGGN}
A sufficient condition that the generalized Gauss-Newton matrix and the Hessian coincide is that the model is linear, because for linear models $H_{f_i^c} = 0$ for $1\leq c \leq C$. In the definition of the generalized Gauss-Newton matrix a choice about where the cut between the loss and the network function has to be made. This is to some degree arbitrary, however, \cite{Schraudolph2002} recommends to perform as much as possible of the computation in the loss such that $\ln p(y_i | f)$ is still convex to ensure positive semi-definiteness of $H_{\mathrm{GGN}}$.

\subsection{Relation Between Fisher Information Matrix and Generalized Gauss-Newton Matrix}
\label{sec:FI_and_GGN}

Rewriting $\nabla_{\theta} \ln p(y_i | f_{\theta}(x_i)) = \nabla_{f_i} \ln p(y_i | f_i)  J_{f_i}$ the Fisher information matrix is of the form 
\eqn{
    \mathcal{I} & = \E{y \sim p(y | x, \theta), x \sim q(x)}{\nabla_{\theta} \ln p(y | x, \theta)^\intercal \nabla_{\theta} \ln p(y | x, \theta)} \\
    & = \E{x}{J_f^\intercal \E{y}{\nabla_f \ln p(y | f)^\intercal \left( \nabla_f \ln p(y | f)\right)} J_f} \\
    & := \E{x}{J_f^\intercal \mathcal{I}_{\ln p(y|f)} J_f} \,,
}{eq:FisherSimilarGGN}
where we write shorthand $f=f_\theta(x)$ and
\eqnn{
    \mathcal{I}_{\ln p(y|f)} & = \E{y}{\nabla_f \ln p(y | f)^\intercal \nabla_f \ln p(y | f)} \\
    & = - \E{y}{\nabla_f^2 \ln p(y | f)} \\
    & = \E{y}{H_{\text{-}\ln p(y | f)}}
}
is the ``Fisher information matrix of the predictive distribution''. 

From these two identities it easily follows that if we substitute $q(x)$ by its empirical distribution $\hat{q}(x)$, the generalized Gauss-Newton matrix \eqref{eq:AppendixGGN} is identical to the Fisher information matrix \eqref{eq:FisherSimilarGGN} if $H_{\text{-}\ln p(y|f)}$ is constant in $y$. This is the case for squared error loss and cross-entropy loss \cite{Heskes2000,Pascanu2013,Martens14}. Indeed, for squared error loss we have
\[
     H_{\text{-}\ln p(y|f)} = \nabla_f^2 \frac{1}{2} \norm{f - y}^2_{\Sigma^{-1}} = \Sigma^{-1}
\]
and for cross-entropy loss we obtain
\eqnn{
     H_{\ln p(y|f);c^\prime c^{\prime\prime}} & = \partial_{f^{c^\prime}} \partial_{f^{c^{\prime\prime}}} \sum_{c=1}^C y^c \ln \phi^c \\
    &  = \partial_{f^{c^\prime}} \partial_{f^{c^{\prime\prime}}} \sum_{c=1}^C y^c \ln \frac{e^{f^c}}{\sum_{{\tilde{c}}=1}^C e^{f^{\tilde{c}}}} \\
     & = \partial_{f^{c^\prime}} \partial_{f^{c^{\prime\prime}}} \left(\sum_{c=1}^C y^c f^c - \sum_c y^c \ln \sum_{{\tilde{c}}=1}^C e^{f^{\tilde{c}}}  \right) \\
    & = \partial_{f^{c^\prime}} \partial_{f^{c^{\prime\prime}}} \left(\sum_{c=1}^C y^c f^c - \ln \sum_{{\tilde{c}}=1}^C e^{f^{\tilde{c}}}  \right) \\
    & = - \partial_{f^{c^\prime}} \partial_{f^{c^{\prime\prime}}} \ln \sum_{{\tilde{c}}=1}^C e^{f^{\tilde{c}}}
    = - \partial_{f^{c^{\prime\prime}}} \phi^{c^\prime} \\
    & = - \delta_{c^{\prime}c^{\prime\prime}} \phi^{c^{\prime}} + \phi^{c^{\prime}} \phi^{c^{\prime\prime}} \,,
}
which are both constant in $y$.

\end{document}